\documentclass[10pt, twocolumn]{article}
\usepackage{arxiv}
\usepackage{geometry}
\usepackage{hyperref}

\usepackage{authblk}

\usepackage[dvipsnames]{xcolor}
\usepackage{natbib}
\usepackage{amsmath,amsthm,amssymb}
\newcommand\numberthis{\addtocounter{equation}{1}\tag{\theequation}}
\usepackage{dsfont}
\usepackage{stmaryrd}
\usepackage{graphicx}
\usepackage{float}
\usepackage{subfigure}
\usepackage{dblfloatfix}
\usepackage[noabbrev]{cleveref}
\usepackage{pifont}

\usepackage{booktabs}
\usepackage{eurosym}
\usepackage{enumitem}
\usepackage{microtype}

\usepackage{algorithm,algorithmic}
\renewcommand{\algorithmicrequire}{\textbf{Input:}}
\renewcommand{\algorithmicensure}{\textbf{Output:}}
\renewcommand{\algorithmiccomment}[1]{//#1}

\newtheorem{theorem}{Theorem}[section]
\newtheorem{lemma}[theorem]{Lemma}

\theoremstyle{theorem}

\newtheoremstyle{TheoremNum}
    {\topsep}{\topsep}              
    {\itshape}                      
    {}                              
    {\bfseries}                     
    {.}                             
    { }                             
    {\thmname{#1}\thmnote{ \bfseries #3}}
\theoremstyle{TheoremNum}
\newtheorem{customTheorem}{Theorem}

\theoremstyle{definition}
\newtheorem{definition}[theorem]{Definition}

\theoremstyle{definition}

\theoremstyle{definition}

\theoremstyle{definition}
\newtheorem{example}[theorem]{Example}

\theoremstyle{definition}
\newtheorem{model}[theorem]{Model}

\definecolor{norange}{RGB}{230,120,20}
\definecolor{ngreen} {RGB}{98,158,31}

\definecolor{blindorange}{RGB}{215,131,37}
\definecolor{blindblue} {RGB}{81,172,226}
\definecolor{blindpurple} {RGB}{193,114,177}
\definecolor{blindgreen} {RGB}{33,145,106}

\usepackage[hang,flushmargin]{footmisc}

\usepackage{wrapfig}

\def\gs{\vspace{-0.5em}}

\renewcommand\epsilon{\varepsilon}
\renewcommand\phi{\varphi}

\def\Z{\mathcal Z}
\def\F{\mathcal F}

\definecolor{mydarkblue}{rgb}{0,0.08,0.45}
\hypersetup{
    colorlinks=true,
    linkcolor=mydarkblue,
    citecolor=mydarkblue,
    filecolor=mydarkblue,
    urlcolor=mydarkblue,
}

\usepackage{times}

\title{Adaptive Conformal Predictions for Time Series}
\author[,1,2,3]{Margaux Zaffran\thanks{Corresponding author: \texttt{margaux.zaffran@inria.fr}}}
\author[3]{Aymeric Dieuleveut}
\author[1,4]{Olivier Féron}
\author[1]{Yannig Goude}
\author[2,5]{Julie Josse}

\affil[1]{Eletrcicité De France R\&D, Palaiseau, France}
\affil[2]{INRIA Sophia-Antipolis, Montpellier, France}
\affil[3]{CMAP, Ecole Polytechnique, IP Paris, Palaiseau, France}
\affil[4]{FiME, Palaiseau, France}
\affil[5]{IDESP, Montpellier, France}

\begin{document}

\maketitle

\begin{abstract}

Uncertainty quantification of predictive models is crucial in decision-making problems. Conformal prediction is a general and theoretically sound answer. However, it requires exchangeable data, excluding time series. While recent works tackled this issue, we argue that Adaptive Conformal Inference \citep[ACI,][]{gibbs_adaptive_2021}, developed for distribution-shift time series, is a good procedure for time series with general dependency. We theoretically analyse the impact of the learning rate on its efficiency in the exchangeable and auto-regressive case. We propose a parameter-free method, AgACI, that adaptively builds upon ACI based on online expert aggregation. We lead extensive fair simulations against competing methods that advocate for ACI's use in time series. We conduct a real case study: electricity price forecasting. The proposed aggregation algorithm provides efficient prediction intervals for day-ahead forecasting. All the code and data to reproduce the experiments is made available.
\end{abstract}

\section{Introduction}

The increasing use of renewable intermittent energies leads to more dependent and volatile energy markets. Therefore, an accurate electricity price forecasting is required to stabilize energy production planning, gathering loads of research works as evidenced by recent substantial reviews \citep{weron_electricity_2014,lago_forecasting_2018-1,lago_forecasting_2021}. Furthermore, probabilistic forecasts are needed to develop risk-based strategies \citep{gaillard_additive_2016,maciejowska_probabilistic_2016,nowotarski_recent_2018,uniejewski_regularized_2021}. On the one hand, the lack of uncertainty quantification of predictive models is a major barrier to the adoption of powerful machine learning methods. On the other hand, probabilistic forecasts are only valid asymptotically or upon strong assumptions on the data.

Conformal prediction \citep[CP,][]{vovk_machine-learning_1999, vovk_algorithmic_2005,papadopoulos_inductive_2002} is a promising framework to overcome both issues. It is a general procedure to build predictive intervals for any (black box) predictive model, such as neural networks, which are \textit{valid} (i.e. achieve nominal marginal coverage) in finite sample and without any distributional assumptions except that the data are exchangeable. 

Thereby, CP has received increasing attention lately, favored by the development of \textit{split conformal prediction} \citep[SCP,][reformulated from \textit{inductive} CP, \citeauthor{papadopoulos_inductive_2002}, \citeyear{papadopoulos_inductive_2002}]{lei_distribution-free_2018}. 
More formally, suppose we have $n$ training samples $\left( x_i, y_i \right) \in \mathds{R}^d \times \mathds{R}$, $i \in \llbracket 1,n \rrbracket$, realizations of random variables $(X_1, Y_1) \dots, (X_n,Y_n)$, and that we aim at predicting a new observation $y_{n+1}$ at $x_{n+1}$. Given a \textit{miscoverage rate} $\alpha \in [0,1]$ fixed by the user (typically 0.1 or 0.05) the aim is to build a predictive interval $\mathcal{C}_{\alpha}$ such that:
\begin{equation}
\mathds{P}\left\{Y_{n+1} \in \mathcal{C}_{\alpha}\left(X_{n+1}\right)\right\} \geq 1-\alpha,
\label{eq:cp_prop}
\end{equation}
with $\mathcal{C}_{\alpha}$ as small as possible, in order to be informative. For the sequel, we call a \textit{valid interval} an interval satisfying \cref{eq:cp_prop} and an \textit{efficient interval} when it is as small as possible \citep{vovk_algorithmic_2005, shafer_tutorial_2008}. 

To achieve this, SCP first splits the $n$ points of the training set in two sets $\rm{Tr}, \rm{Cal}\subset \llbracket 1,n \rrbracket$, to create a \textit{proper training set}, $\rm{Tr}$, and a \textit{calibration set}, $\rm{Cal}$. On the proper training\- set a regression model $\hat \mu$ (chosen by the user) is fitted, and then used to predict on the calibration set. A \textit{conformity score} is applied to assess the conformity between the calibration's response values and the predicted values, giving $S_{\rm{Cal}}=  \{(s_i)_{i\in \rm{Cal}}\}$. In regression, usually the absolute value of the residuals is used, i.e.~$s_i = |\hat\mu(x_i) - y_i|$. Finally, a corrected\footnote{The correction $\alpha \to \hat\alpha$ is needed because of the inflation of quantiles in finite sample (see Lemma 2 in \citet{romano_conformalized_2019} or Section 2 in \citet{lei_distribution-free_2018}).} $(1-\hat\alpha)$-th quantile  of these scores $\widehat{Q}_{1-\hat \alpha}(S_{\rm{Cal}})$ is computed to define the size of the interval, which, in its simplest form, is centered on the predicted value: $\mathcal{C}_{\alpha}\left(x_{n+1}\right) = \widehat{C}_{\hat \alpha }(x_{n+1}):=[ \hat \mu(x_{n+1})\pm \widehat{Q}_{1-\hat \alpha}(S_{\rm{Cal}})]$. These steps are detailed in \Cref{app:scp}. More details on CP, including beyond regression, are given in \citet{vovk_algorithmic_2005,angelopoulos-gentle}. 

The cornerstone of SCP \textit{validity} results is the exchangeability assumption of the data~\citep[see][and \Cref{app:guarantees_scp}]{lei_distribution-free_2018}. 
However, this assumption is not met in time series forecasting problems. 
Despite the lack of theoretical guarantees, several works have applied CP to time series. \citet{dashevskiy_2008, dashevskiy_2011} apply original (\textit{inductive}) CP \citep{papadopoulos_inductive_2002} to both simulated (using Auto-Regressive Moving Average (ARMA) processes) and real network traffic data and obtain \textit{valid} intervals.
\citet{pmlr-v128-wisniewski20a, kath_conformal_2021} apply SCP respectively to financial data (e.g. markets makers' net positions) and to electricity price forecasting on various markets. In order to account for the temporal aspect, they consider an online version of SCP. In both studies, the \textit{validity} varied greatly depending on the markets and the underlying regression model, suggesting that further developments of CP and theoretical guarantees for time series are needed. 

To this end, \citet{chernozhukov_exact_2018} extend the CP theory to ergodic cases in order to include dependent data. \citet{pmlr-v139-xu21h} improve on that theory and propose a new algorithm, Ensemble Prediction Interval (EnbPI), adapted to time series by adding a sequential aspect.

Another case that breaks the exchangeability assumption is \textit{distribution shift}, which allows for example to deal with cases where the test data is shifted with respect to the training data.  \citet{tibshirani_conformal_2019} consider covariate shift while \citet{cauchois_robust_2020} tackle a joint distributional shift setting (that is, of $(X,Y)$). In both studies, a single shift in the distribution is considered, a major limitation for applying these methods to time series. In an adversarial setting, \citet{gibbs_adaptive_2021} propose Adaptive Conformal Inference (ACI), accounting for an undefined number of shifts on the joint distribution. It is based on refitting the predictive model, as well as updating online the quantile level used by a recursive scheme depending on an hyper-parameter $\gamma$ (a learning rate). Furthermore, they prove an asymptotic \textit{validity} result for any data distribution. 

\begin{figure}[!b]
\gs\gs
    \centering
    \includegraphics[scale=0.19]{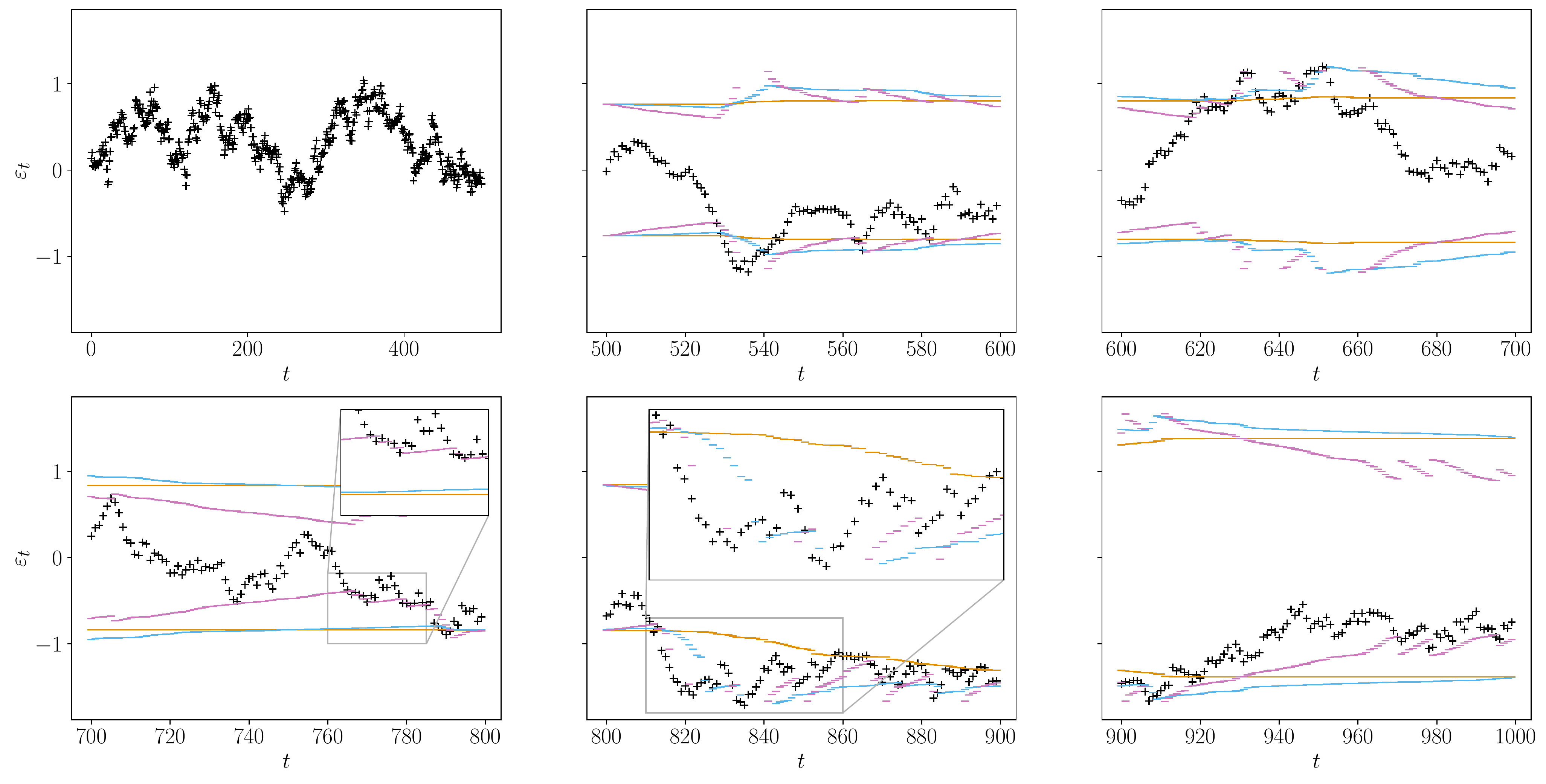}
    \gs \gs \gs
    \caption{ACI on one simulated path $\varepsilon_t$, $t=1, \dots, 1000$, from an AR(1) process (in black). The first 500 values form the initial calibration set (top left subplot), and predicted interval bounds are computed on the last 500 points (5 last subplots, 100 on each) for  $\gamma =0$ (\textcolor{blindorange}{orange}),  $\gamma = 0.01$ (\textcolor{blindblue}{blue}) and $\gamma = 0.05$ (\textcolor{blindpurple}{purple}).}
    \label{fig:acp_example}
\end{figure}
We argue in this work that the design and guarantees of ACI can be beneficial for dependent data without distribution shifts. We illustrate this on  a toy example in \Cref{fig:acp_example}, assuming that the fitted regression model produces AR(1) residuals. The two ACI versions ($\gamma\!\!=$\textcolor{blindblue}{$0.01$} and \textcolor{blindpurple}{$0.05$} intervals) adapt better to the data than classical online SCP (\textcolor{blindorange}{$\gamma=0$}). 

\textbf{Contributions.} We propose to analyse ACI \citep{gibbs_adaptive_2021} in the context of time series with general de\-pendency and make the following contributions:
\begin{itemize}[topsep=0pt,noitemsep,leftmargin=*,wide]
    \item Relying on an asymptotic analysis of ACI's behaviour for simple time series distribution, we prove that ACI deteriorates \textit{efficiency} in an exchangeable case while improving it in an AR setting with a well-chosen $\gamma$ (\Cref{sec:theory}).
    \item We introduce AgACI, a parameter-free method using online expert aggregation, to avoid choosing $\gamma$, achieving good performances in terms of $\textit{validity}$ and $\textit{efficiency}$~(\Cref{sec:aci_adaptive}). 
    \item We compare ACI to EnbPI and online SCP on extensive synthetic experiments  
    and we propose an easy-to-interpret visualisation combining \textit{validity} and \textit{efficiency} (\Cref{sec:benchmark}).
    \item We forecast and give predictive intervals on French electricity prices, an area where accurate predictions, but also controlled predictive intervals, are required (\Cref{sec:comp_real}).
\end{itemize}
To allow for better benchmarking of existing and new methods, we provide (re-)implementations in Python of all the described methods and a complete pipeline of analysis on \href{https://github.com/mzaffran/AdaptiveConformalPredictionsTimeSeries}{GitHub}.

\gs

\section{Setting: ACI for time series}
\label{sec:acp}
\gs
In this section, we introduce ACI and our framework. We consider $T_0$ observations $\left( x_1, y_1 \right), \dots, \left( x_{T_0}, y_{T_0} \right)$ in $\mathds{R}^d \times \mathds{R}$. The aim is to predict the response values and give predictive intervals for $T_1$ subsequent observations $x_{T_0+1},\dots,x_{T_0+T_1}$ sequentially: at any prediction step $t \in \llbracket T_0+1, T_0+T_1 \rrbracket$, $y_{t-T_0},\dots,y_{t-1}$ have been revealed. Thereby, the data $\left(\left(x_{t-T_0}, y_{t-T_0}\right), \dots, \left(x_{t-1}, y_{t-1}\right)\right)$ are used for the construction of the predicted interval. 

\textbf{Adaptive Conformal Inference.} Proposed by \citet{gibbs_adaptive_2021}, ACI is designed to adapt CP to temporal distribution shifts. The idea of ACI is twofold. First, one considers an online procedure with a random split\footnote{\Cref{fig:scheme_both}(a) with \textcolor{blindorange}{training} and \textcolor{blindblue}{calibration} part shuffled randomly.}, i.e., $\rm{Tr}_t$ and  $\rm{Cal}_t$ are random subsets of the last $T_0$ points. Second, to improve adaptation when the data is highly shifted, an \textit{effective miscoverage level} $\alpha_t$,  updated recursively, is used instead of the target level $\alpha$. Set $\alpha_1 = \alpha$, and for $t\geq 1$ 
\begin{align}
\begin{cases}
\widehat{C}_{\alpha_t}\left(x_{t}\right) &=  [\hat \mu(x_{t}) \pm \widehat{Q}_{1- \alpha_t}(S_{\rm{Cal}_t})] \\ 
\alpha_{t+1}&=\alpha_{t}+\gamma\left(\alpha-\mathds{1}\{ y_{t} \notin \widehat{C}_{\alpha_t}\left(x_{t}\right) \}\right),
\label{eq:update_scheme}
\end{cases}
\end{align}
for  $\gamma \geq 0$\footnote{ACI actually wraps around \textit{any} CP procedure, here the definition is given using mean regression SCP.}. If ACI does not cover at time~$t$, then ${\alpha_{t+1} \leq \alpha_t}$, and the size of the predictive interval increases; conversely when it covers. Nothing prevents $\alpha_t \leq 0$ or $\alpha_t \geq 1$. While the later is rare (as $\alpha$ is small) and produces by convention $\widehat{C}_{\alpha_t}(\cdot) = \{\hat \mu(\cdot)\}$ (i.e.~$\widehat{Q}_{1-\alpha_t}= 0$) , the former can happen frequently for some $\gamma$, giving $\widehat{C}_{\alpha_t} \equiv \mathds{R}$ ($\widehat{Q}_{1-\alpha_t}=+\infty$). 
 
\textbf{How to deal with infinite intervals.} A specificity of ACI's algorithm is thus to often produce infinite intervals. Defining the \textit{average} length of an interval is then impossible. In order to assess the \textit{efficiency} in the following, we consider two solutions: 
(i) imputing the length of infinite intervals by (twice) the overall maximum of the residuals, or $Q(1)$ if  the residual's quantile function is known and bounded\footnote{This happens in practice when the response and prediction are bounded, e.g., thanks to physical/real constraints as for the spot prices presented in \Cref{sec:data_real}, that are bounded by market rules.};  
(ii) focusing on the median instead. 

\textbf{ACI on time series with general dependency.} As highlighted by \citet{pmlr-v128-wisniewski20a,kath_conformal_2021}, the first step to adapt a method for dependent time series is to work online which is the case for ACI. Moreover, the update of the quantile level according to the previous error implies that ACI could cope with a fitted model that has not correctly caught the temporal evolution, such as a trend, a seasonality pattern or a dependence on the past. Therefore, ACI is a perfect candidate for CP for time series with general dependency.
To account for the temporal structure, we change the random split to a sequential split.\footnote{As in \Cref{fig:scheme_both}(a). This is also consistent with OSSCP (Sec.~\ref{sec:osscp}).}

To gain understanding on ACI in the context of dependent temporal data, we analyse a situation where a fitted regression model $\hat{\mu}$ produces AR(1) residuals, thus $y_t-\hat \mu(x_t) = \varepsilon_t$, where $\varepsilon_t$ is an AR(1) process: $\varepsilon_{t+1} = 0.99\varepsilon_t + \xi_{t+1}$, with $\xi_t \sim \mathcal{N}(0,0.01)$. We plot this toy example in \Cref{fig:acp_example}, for $T_0 = T_1 = 500$. 
Three versions of ACI are compared: \textcolor{blindorange}{$\gamma = 0$}, the quantile level is not updated but the calibration set $\rm{Cal}_t$ is; \textcolor{blindblue}{$\gamma = 0.01$} and \textcolor{blindpurple}{$\gamma = 0.05$}. To obtain an insightful visualisation\footnote{We suggest focusing the visualisation on the scores to analyse the behaviour of CP methods, as they are at the core of the \textit{validity} proof. A detailed discussion on this is given in App.~\ref{app:how_to_visu}}, we represent the interval $[\pm \widehat{Q}_{1-\alpha_t}(S_{\rm{Cal}_t})]$ instead of $\widehat{C}_{\alpha_t}(x_t)$. When no intervals are displayed, ACI is predicting~$\mathds{R}$. Here and in the sequel, we use $\alpha=0.1$.

In this toy example, the coverage rate among many observations is \textit{valid} for $\gamma \in \{0.01,0.05\}$ (90\% and 92\% of points included) but not for $\gamma = 0$ (72.6\%). Moreover, \Cref{fig:acp_example} shows that the type of errors depends on $\gamma$. For $\gamma = 0$, ACI excludes consecutive observations (e.g. for $t \in [810,860]$, zoomed-in plot). For $\gamma \in \{0.01,0.05\}$, ACI manages to adapt to these observations, and the higher the $\gamma$, the less the adaptation is delayed. Furthermore, when the residuals are small and far from both interval bounds, ACI quickly reduces the interval's length and produces more \textit{efficient} intervals. Consequently, ACI may also not cover on points for which the residuals have a relatively small values compared to the calibration's values (e.g. for $t \in [760,785]$).

\section{Impact of $\gamma$ on ACI efficiency}
\label{sec:theory}
The choice of the parameter $\gamma$ strongly impacts the behaviour of ACI: while the method always satisfies the \textit{asymptotic validity} property, i.e.~$\frac{1}{T}\sum_{t=1}^{T} \mathds{1}\{ y_{t} \notin \widehat{C}_{\alpha_t}\left(x_{t}\right)\} \overset{a.s.}{\underset{T\to \infty}{\longrightarrow}} \alpha$ 
\citep[Proposition 4.1 in][]{gibbs_adaptive_2021}, this property does not give any insight on the length of resulting intervals. Besides, this guarantee directly stems from the fact that  $\frac{1}{T}\sum_{t=1}^{T} \mathds{1}\{ y_{t} \notin \widehat{C}_{\alpha_t}\left(x_{t}\right) \} -\alpha \le 2/(\gamma T)$. This tends to suggest the use of larger $\gamma$ values, that unfortunately generate frequent infinite intervals. Here, we thus analyse the impact of $\gamma$ on ACI's \textit{efficiency} in simple yet insightful cases: in \Cref{subsec:theory-exch}, focusing on the exchangeable case, then in \Cref{sec:theory_ar}, with a simple AR process on the residuals. 

\textbf{Approach.} Our focus is on the impact of the key parameter~$\gamma$. Analysing simple theoretical distributions allows to build intuition on the behaviour of the algorithm for more complex data structure. In order to derive theoretical results, we thus make  supplementary modelling assumptions on the residuals, and do not consider the impact of the calibration set: we introduce $Q$ the quantile function of the scores and assume, for all $\hat \alpha$ and $t$, $\widehat Q_{1-\hat \alpha}(S_{\rm{Cal}_t})=Q(1-\hat \alpha)$. This corresponds to considering the limit as $\#\rm{Cal}\to \infty$. This allows to focus on the impact of recursive updates in \eqref{eq:update_scheme} and  describe their behaviour by relying on Markov Chain theory.

\subsection{Exchangeable case}
\label{subsec:theory-exch}

ACI is usually applied in an adversarial context. If the scores are actually exchangeable, ACI's \textit{validity} would not improve upon SCP (known to be quasi-exactly \textit{valid}), thus assessing ACI's impact on \textit{efficiency} is necessary. Define $L(\alpha_t) = 2Q(1-\alpha_t)$ the length of the interval predicted by the adaptive algorithm at time $t$, and ${L_0 = 2Q(1-\alpha)}$ the length of the interval predicted by the non-adaptive algorithm (or equivalently, $\gamma = 0$).

\begin{theorem}
Assume that: (i) $\alpha \in \mathds{Q}$; 
    (ii) the scores are exchangeable with quantile function $Q$; (iii) the quantile function is perfectly estimated at each time (as defined above); (iv) the quantile function $Q$ is bounded and $\mathcal{C}^4([0,1])$.
Then, for all $\gamma>0$,  $\left(\alpha_t\right)_{t > 0}$ forms a Markov Chain, that admits a stationary distribution $\pi_\gamma$, and
\begin{equation*}
    \frac{1}{T}\sum\limits_{t=1}^T L(\alpha_t) \overset{a.s.}{\underset{T \rightarrow +\infty}{\longrightarrow}} \mathds{E}_{\pi_\gamma}[L] \overset{\text{not.}}{=} \mathds{E}_{\tilde \alpha \sim  \pi_\gamma}[L(\tilde \alpha)].
\end{equation*}
Moreover, as $\gamma \to 0$, 
\begin{equation*}
      \mathds{E}_{\pi_\gamma}[L] = L_0 + Q''(1-\alpha)\frac{\gamma}{2}\alpha(1-\alpha) + O(\gamma^{3/2}).
\end{equation*}
\label{thm:iid}
\end{theorem}
\gs\gs\gs\gs
\textbf{Interpretation of assumptions.} Assumption (i) is weak since a practitioner will always select $\alpha \in \mathds{Q}$ while  assumption (ii) describes the classical exchangeable setting. The main assumptions are (iii) and (iv): (iii) can be interpreted as considering an infinite calibration set while (iv) is necessary\footnote{$\forall \gamma\!\!>\!0$, $\mathds{P}_{\pi_\gamma}(\tilde \alpha \le 0)\!>\!0$: we need $ |Q(1)|\!<\!\infty$ to define $\mathds{E}_{\pi_\gamma}[L]$.} in order to define $\mathds{E}_{\pi_\gamma}[L]$: here, we extend $Q(1-\hat\alpha)$ by $Q(1)$ for $\hat \alpha<0$. Finally, the regularity assumption on $Q$ is purely technical. 

\textbf{Interpretation of the result.} For standard distributions, $Q''(1-\alpha) > 0$,\footnote{as $Q'(x) = \frac{1}{f(Q(x))}$ with $f$ the scores' probability density function, $Q'(x)$ increases locally around $x$ if and only if $f$ decreases locally around $Q(x)$ ($Q$ is increasing). Thus, $Q''(x) > 0$ if and only if $f$ decreases locally around $Q(x)$. Thereby, for $x = 1-\alpha$ high (usually the case), $Q''(1-\alpha) > 0$ for standard distributions.} and \Cref{thm:iid} implies that ACI on exchangeable scores \textit{degrades} the \textit{efficiency} linearly with $\gamma$ compared to CP. This is an important takeaway from the analysis, that underlines that such adaptive algorithms may actually hinder the performance if the data does not have any temporal dependency, and a small $\gamma$ is preferable. For example, if the residuals are standard gaussians, for $\alpha=0.01$, setting $\gamma = 0.03$ (resp. $\gamma = 0.05$) will increase the length by 1.59\% (resp. by 3.38\%) with respect to $\gamma = 0$.

\subsection{AR(1) case}
\label{sec:theory_ar}
We now consider the case of (highly) correlated residuals, which happens in many practical time series applications.

\begin{theorem}
\label{thm:ar}
Assume that: (i) $\alpha \in \mathds{Q}$; (ii) the residuals follow an AR(1) process (i.e., $\varepsilon_{t+1} = \varphi\varepsilon_{t}+\xi_{t+1}$ with $(\xi_t)_t$ i.i.d. random variables admitting a continuous density with respect to Lebesgue measure, of support $\mathcal{S}$) clipped at a large value $R$, and $[-R,R] \subset \mathcal{S}$;
(iii) the quantile function $Q$ of the stationary distribution of $(\varepsilon_t)_t$ is known; (iv) $Q$ is bounded by $R$. 
Then $(\alpha_t,\varepsilon_{t-1})$ is a homogeneous Markov Chain in $\mathds R^2$ that admits a unique stationary distribution $\pi_{\gamma,\varphi}$. Moreover, $$\frac{1}{T}\sum\limits_{t=1}^T L(\alpha_t) \overset{a.s.}{\underset{T \rightarrow +\infty}{\longrightarrow}} \mathds{E}_{\pi_{\gamma,\varphi}}[L].$$ 
\end{theorem}

We numerically estimate $\gamma^*_\varphi = \arg\!\min_\gamma \mathds{E}_{\pi_{\gamma,\varphi}}[L]$ in \Cref{fig:aci_ar_numerical}. To do so, AR(1) processes of length $T = 10^6$ are simulated for various $\varphi$ and asymptotic variance 1. 
ACI is applied on each of them, with 100 different $\gamma \in [0,0.2]$. \Cref{fig:aci_ar_numerical} (left) represents the average length depending on $\gamma$ for each $\varphi$, and (right)  the values of $\gamma$ minimizing this average length for each $\varphi$ (for 25 repetitions of the experiment). The average length is computed after imputing all the infinite intervals' length by the maximum of the process, as explained in \Cref{sec:acp}. A similar study using instead the median length is provided after the proofs in \Cref{app:aci_theory}.

\textbf{Interpretation.} We make the following observations:
\begin{enumerate}[noitemsep,topsep=0pt,wide]
\item  For high $\varphi$, ACI indeed improves for a strictly positive $\gamma$ upon $\gamma=0$. This proves that ACI can be used to produce smaller intervals for time series CP. The function $\gamma\mapsto \mathds E_{\pi_{\gamma,\varphi}}[L]$ decreases until $\gamma^*_\varphi$, then increases again, as expected because very large $\gamma$ cause the algorithm to be less stable and produce numerous infinite intervals.
\item In \Cref{fig:aci_ar_numerical} (left), zoomed-in plot, the black line represents asymptotic result of \Cref{thm:iid}. We retrieve here that the expected length is minimal for $\gamma=0$ and grows linearly with $\gamma$ around 0. This behaviour is very similar for $\varphi = 0.6$.
\item For any $\gamma$, the function $\varphi\mapsto \mathds E_{\pi_{\gamma,\varphi}}[L]$ is decreasing (\Cref{fig:aci_ar_numerical}, left). Indeed, stronger correlation between residuals (i.e., a higher $\varphi$), allows to build smaller intervals. This confirms that ACI's impact strengthens with the strength of the temporal dependence.
\item Surprisingly, the function $\varphi\mapsto\gamma^*_\varphi$, that corresponds to the optimal learning rate for a given signal, \textit{is non-monotonic}, (\Cref{fig:aci_ar_numerical}, right). As $\gamma=0$ is optimal for $\varphi=0$, the function first increases. However, the optimal learning rate then diminishes as $\varphi$ increases. 
This sheds light on the complex intrinsic tradeoffs of the method: for small values  of $\varphi$, using $\gamma>0$ simply degrades the \textit{efficiency}; for ``moderate'' values of $\varphi$ using a larger $\gamma$ is necessary to quickly benefit from the short-term dependency between residuals; finally, for larger values  of $\varphi$, the process exhibits a longer memory, thus it is crucial to find a smaller learning rate that produces more stable intervals, even if it means that the algorithm won't adapt as quickly.
\end{enumerate}

\begin{figure}[!t]
    \centering
    \begin{minipage}[b]{0.235\textwidth}
        \includegraphics[scale=0.3]{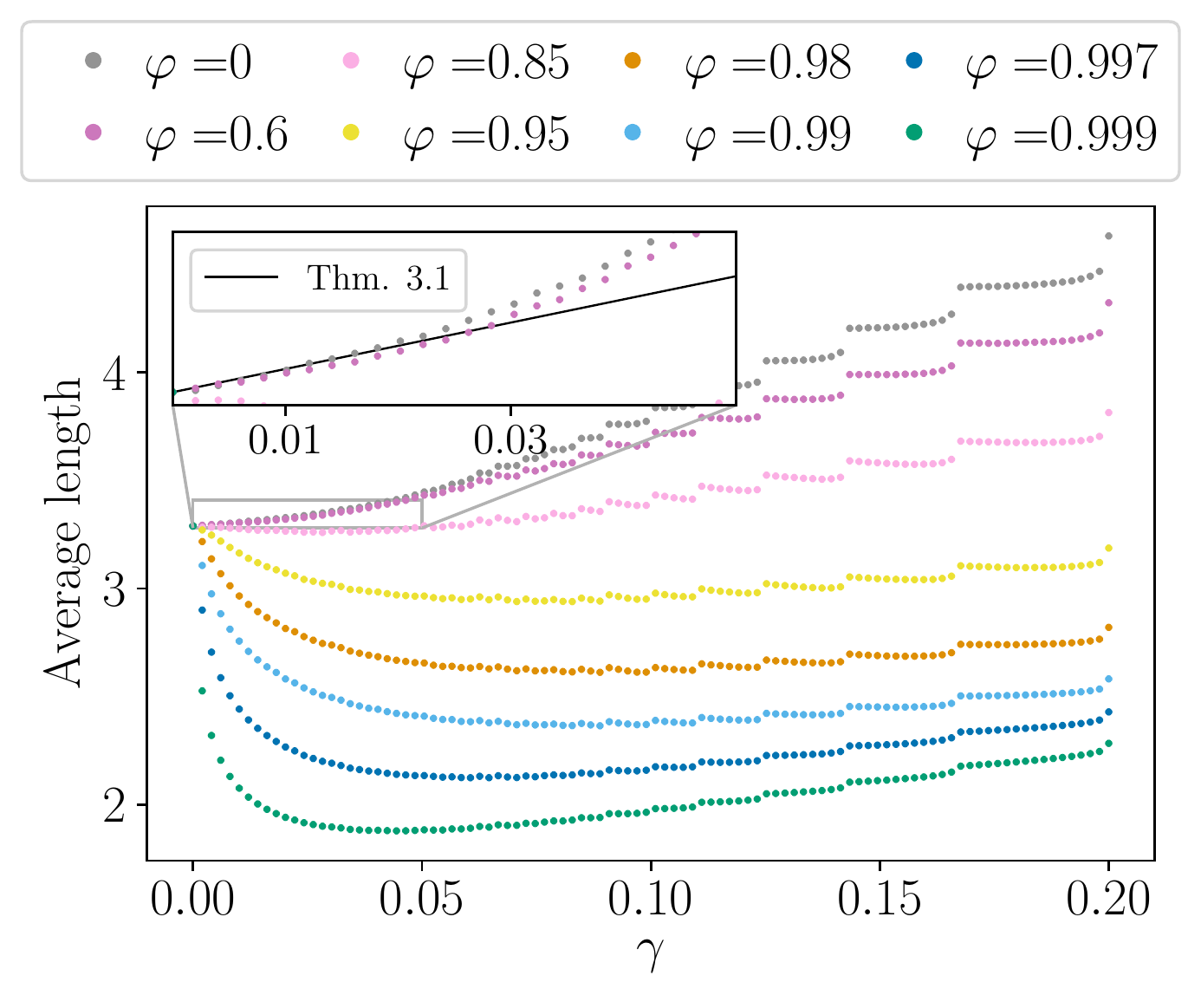}
    \end{minipage}
    \hfill
    \begin{minipage}[b]{0.235\textwidth}
        \includegraphics[scale=0.3]{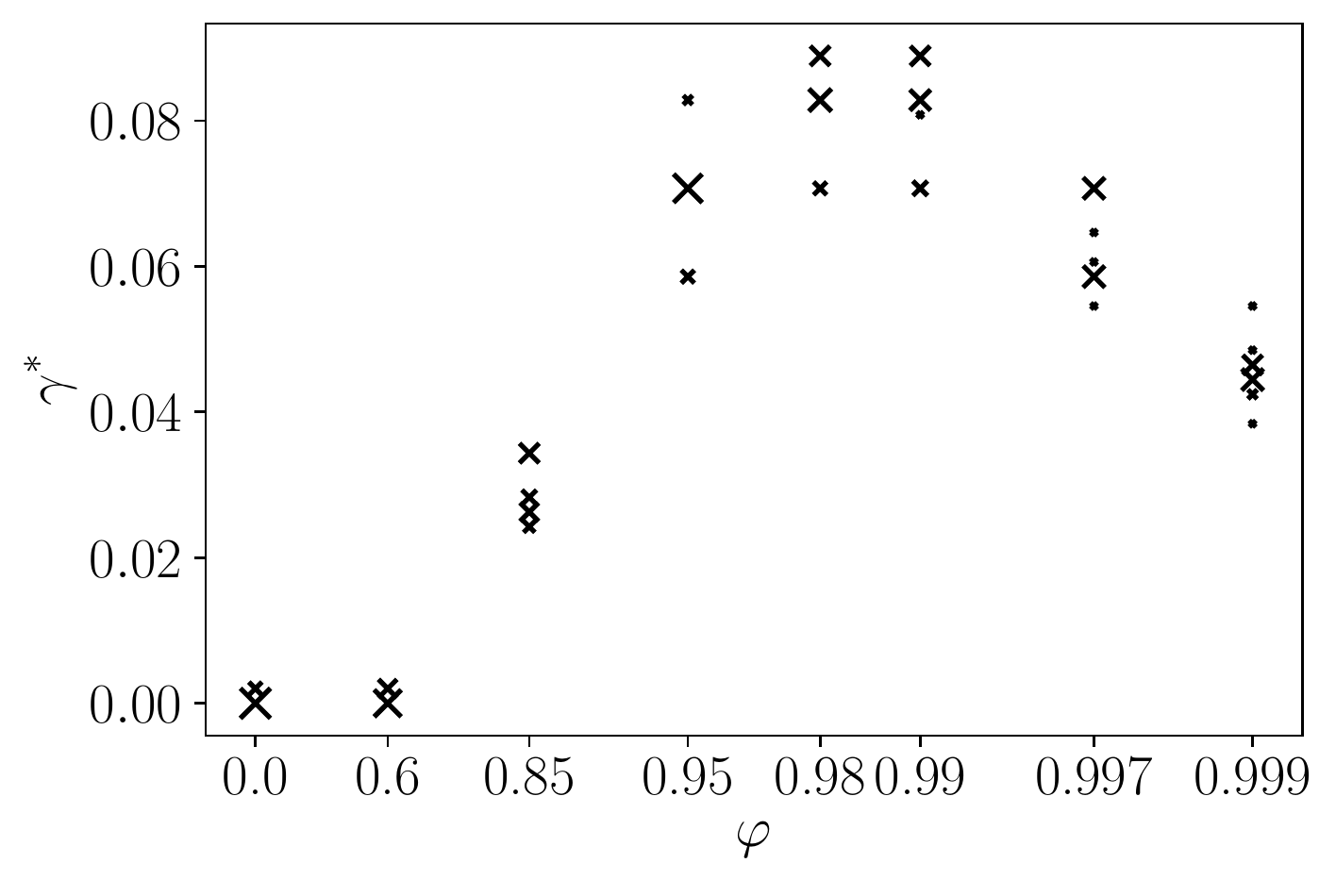}
    \end{minipage}
    \gs\gs
    \caption{Left: evolution of the mean length depending on $\gamma$ for various~$\varphi$. Right: $\gamma^*$ minimizing the average length for each $\varphi$ (each cross has a size proportional to the number of runs for which $\gamma^*$ was the minimizer).}
    \label{fig:aci_ar_numerical}
    \gs\gs
\end{figure}

Overall, these results highlight the importance of the choice of $\gamma$, as not choosing $\gamma^*$ can lead to significantly larger intervals. In addition, they provide insights on the corresponding dynamics. Yet the choice of $\gamma$ in more complex practical settings remains difficult: this calls for adaptive strategies.

\section{Adaptive strategies based on ACI}
\label{sec:aci_adaptive}
To prevent the critical choice of $\gamma$ an ideal solution is an adaptive strategy with a time dependent $\gamma$. We propose two  strategies based on running ACI for $K \in \mathds{N}$ values $\{(\gamma_k)_{k\le K}\}$ of $\gamma$, chosen by the user. Overall, this does not increase the computational cost because $\rm{Tr}_t$ and $\rm{Cal}_t$ are shared between all ACI; thus the only additional cost is the computation of the $K$ different quantiles. For any $x_t$, denote $\widehat{C}_{\alpha_{t,k}}(x_t)$ the interval at time $t$ built by ACI using $\gamma_k$. 

\textbf{Naive strategy.} A simple strategy is to use at each step the $\gamma$ that achieved in the past the best \textit{efficiency} while ensuring \textit{validity}. For stability purposes, consider a warm-up period $T_w \le T_1-1 $. For each $t\geq T_0+T_w$, we select $k^*_{t+1} \in \; {\text{argmin}_{k \in \mathcal{A}_t}} \left\{ {t}^{-1}\sum_{s = 1}^{t} \text{length}(\widehat{C}_{\alpha_{s,k}}(x_s)) \right\}$ with $\mathcal{A}_t = \{ k \in \llbracket 1,K \rrbracket \mid {t}^{-1}\sum_{s = 1}^{t} \mathds{1}_{y_s \in \widehat{C}_{\alpha_{s,k}}(x_s)} \geq 1 - \alpha \}$ 
or $k^*_{t+1} \in \; {\text{argmin}}_{k \in \llbracket 1,K \rrbracket} \{ | 1 - \alpha - {t}^{-1} \sum_{s = 1}^{t} \mathds{1}_{y_s \in \widehat{C}_{\alpha_{s,k}}(x_s)} | \}$ if $\mathcal{A}_t = \emptyset$. For the first $T_w$ steps, an arbitrary strategy is applied (in simulations, $\gamma = 0$  for $t \le T_w = 50$). 

\textbf{Online Expert Aggregation on ACI (AgACI).} Instead of picking one $\gamma$ in the grid, we introduce an adaptive aggregation of \textit{experts}~\citep{cesa2006prediction}, with expert $k$ being ACI with parameter $\gamma_k$. This strategy is detailed in \Cref{alg:agg_aci}. At each step $t$, it performs two independent aggregations of the $K$-ACI intervals $\widehat{C}_{\alpha_{t,k}}(\cdot) \overset{\text{not.}}{=} [\hat{b}^{(\ell)}_{t,k}(\cdot), \hat{b}^{(u)}_{t,k}(\cdot)]$,
one for each bound, and outputs $\widetilde{C}_{t}(\cdot) \overset{\text{not.}}{=} [\tilde{b}^{(\ell)}_{t}(\cdot), \tilde{b}^{(u)}_{t}(\cdot)]$. Aggregation computes an optimal weighted mean of the experts (Line 11), where the weights $\omega^{(\ell)}_{t,k}$, $\omega^{(u)}_{t,k}$ assigned to expert $k$ depend on all experts performances (suffered \textit{losses}) at time steps $1, \cdots, t$ (Line 9). We use the pinball loss $\rho_\beta$, as it is frequent in quantile regression, where the pinball parameter $\beta$ is chosen to $\alpha/2$ (resp. $1-\alpha/2$) for the lower (resp. upper) bound. These {losses} are plugged in the \textit{aggregation rule} $\Phi$. Finally, the aggregation rule can include the computation of the gradients of the loss
(\textit{gradient trick}). As aggregation rules require bounded experts, a thresholding step is added (Line 6).

We chose $\Phi$ to be the Bernstein Online Aggregation \citep[BOA,][]{wintenberger2017optimal}, that was successfully applied for financial data \citep{berrisch2021crps, remlinger2021expert}. We rely on R package OPERA \citep{opera}, which allows the user to easily select among many aggregation rules such as EWA \citep{vovk1990aggregating}, ML-Poly \citep{gaillard2014second} or FTRL \citep{ftrl, hazan2019introduction}, etc., that give similar results in our experiments. We use the gradient trick in the simulations. In the sequel, AgACI refers to AgACI using BOA and gradient~trick.

\setlength{\textfloatsep}{10pt}
\begin{algorithm}[!h]
\caption{Online Expert Aggregation on ACI (AgACI)}
\label{alg:agg_aci}
\begin{algorithmic}[1] 
\REQUIRE Miscoverage rate $\alpha$, grid $\{\gamma_k, k \in \llbracket 1,K \rrbracket \}$, aggregation rule $\Phi$, threshold values $M^{(\ell)},M^{(u)}$.
\STATE Let $\beta^{(\ell)} = \alpha/2$ and $\beta^{(u)} = 1-\alpha/2$
\FOR {$t \in \llbracket T_0 + 1 , T_0 + T_1 \rrbracket$} 
\STATE Set $\widetilde{C}_{t}(x_t) = [\tilde{b}^{(\ell)}_{t}(x_t),\tilde{b}^{(u)}_{t}(x_t)]$
\FOR {$k \in \llbracket 1,K \rrbracket$} 
\STATE Compute $\hat{b}^{(\cdot)}_{t,k}(x_t)$ using ACI with $\gamma_k$.
\STATE \textbf{if} $\hat{b}^{(\cdot)}_{t,k}(x_t) \notin \mathds{R}$ \textbf{then} set $\hat{b}^{(\cdot)}_{t,k}(x_t) = M^{(\cdot)}$
\ENDFOR
\FOR {$k \in \llbracket 1,K \rrbracket$}
\STATE $\begin{aligned}
    \omega^{(\cdot)}_{t,k} = \Phi\left(\left\{ \right.\right.& \rho_{\beta^{(\cdot)}}(y_s,\hat{b}^{(\cdot)}_{s,l}(x_s)), s \in \llbracket T_0+1, t \rrbracket,  \\
    &\left.\left. l \in \llbracket 1,K \rrbracket \right\}\right)
\end{aligned}$
\ENDFOR
\STATE Define $ \tilde{b}^{(\cdot)}_{t+1}(x) = \frac{\sum_{k=1}^K \omega^{(\cdot)}_{t,k} \hat{b}^{(\cdot)}_{t,k}(x)}{\sum_{l=1}^K \omega^{(\cdot)}_{t,l}}$ for any $x \in \mathds{R}^d$
\ENDFOR
\end{algorithmic}
\end{algorithm}

\section{Numerical evaluation on synthetic data sets}
\label{sec:benchmark}
In this section we conduct synthetic experiments on a wide range of data sets presented in \Cref{sec:data_syn}. The goal of this section is twofold. First, in  \Cref{sec:acp_gamma_emp}, comparing our proposed adaptive strategies to ACI with a wide range of $\gamma$ values. Second, in \Cref{sec:comp_syn}, comparing performances of AgACI and ACI to that of competitors -- namely EnbPI and online sequential SCP, described in \Cref{sec:baselines}.
\subsection{Data generation process and settings}
\label{sec:data_syn}
We generate data according to:
\begin{equation}
\begin{aligned}
Y_{t} = & 10 \sin \left(\pi X_{t,1} X_{t,2}\right)+20\left(X_{t,3}-0.5\right)^{2} \\
& + 10 X_{t,4} + 5 X_{t,5} + 0 X_{t,6}+\varepsilon_{t},
\end{aligned}
\label{eq:friedman}
\end{equation}
where the $X_{t}$ are multivariate uniformly distributed on $[0,1]$, and $X_{t,6}$ represents an uninformative variable. The noise $\varepsilon_{t}$ is generated from an ARMA(1,1) process of parameters $\varphi$ and $\theta$, i.e.~$ \varepsilon_{t+1} =  \varphi \varepsilon_{t} + \xi_{t+1} + \theta \xi_t $, with $\xi_t$ a white noise called the \textit{innovation} (see \Cref{app:arma} for details).  When the noise is i.i.d., one retrieves the simulations from \citet{Friedman}. The temporal dependence is present only in the noise in order to control its strength and its impact on the algorithms' performance.

Given the non-linear structure of the data generating process, we use  a random forest (RF) as predictive model, with the same hyper-parameters through all the experiments (specified in \Cref{app:rf}). 

\begin{figure*}[!b]
\gs\gs
\centering
\includegraphics[scale=0.28]{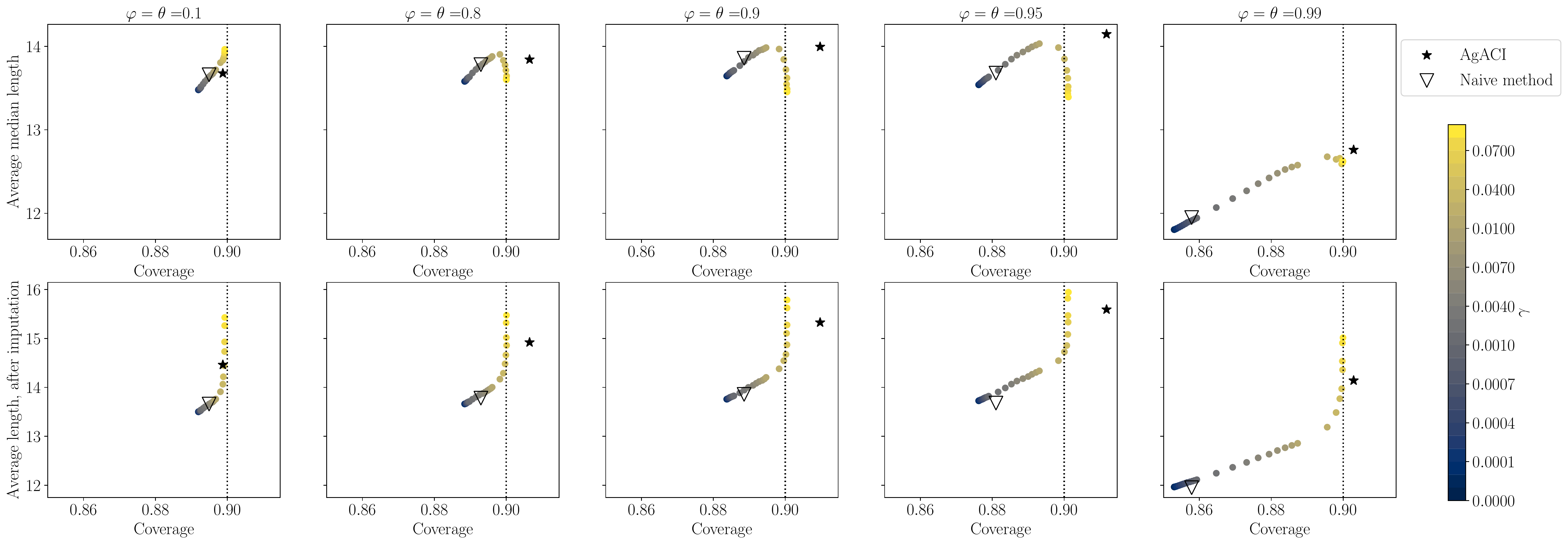}
\gs\gs
\caption{ACI performance with various $\theta$, $\varphi$ and $\gamma$ on data simulated according to \cref{eq:friedman} with a Gaussian ARMA(1,1) noise of asymptotic variance 10 (see \Cref{app:arma}). Top row: average median length w.r.t. the coverage. Bottom row: average length after imputation w.r.t. the coverage. Stars correspond to the proposed online expert aggregation strategy, AgACI, and empty triangles to the naive choice.}
\label{fig:acp_gamma}
\end{figure*}

To assess the impact of the temporal structure, we vary $\varphi$ and $\theta$ in $\{0.1,0.8,0.9,0.95,0.99\}$. To focus on the impact of the dependence structure, the value of the {innovation}'s variance is selected so that the asymptotic variance of $\epsilon_t$ is independent of $\varphi,\theta$: here we choose $\lim_{t\to\infty} \text{Var}(\varepsilon_t)=10$. For each set of parameters, we generate $n=500$ samples $(\varepsilon_t)_{t\in\llbracket1,T_0+T_1\rrbracket}$ with $T_0=200$. In the sequel we display the results on an ARMA(1,1) which are representative of all the results obtained. For the sake of simplicity, we consider $\varphi = \theta$. Complementary results (i) for an asymptotic variance of~1 (corresponding to a higher \textit{signal to noise} ratio), (ii) for AR(1) and MA(1) models are available in \Cref{app:syn_expe}.

\textbf{Joint visualisation of validity \& efficiency.} In order to simultaneously assess \textit{validity} and \textit{efficiency}, in \Cref{fig:acp_gamma,fig:arma_fixed_10,fig:comp_spot}, we represent on the same graph the empirical coverage and average median length (used for \textit{efficiency} as imputing the infinite bounds by the maximum of the whole sequence is not always feasible in practice). In those three figures, the vertical dotted line represents the target miscoverage rate, $\alpha=0.1$. Consequently, a method is \textit{valid} when it lies at the right of this line, and the lower the better. 
 
\subsection{Impact of $\gamma$, performance of AgACI}
\label{sec:acp_gamma_emp}
\Cref{fig:acp_gamma} illustrates the behaviour of ACI (with multiple values of $\gamma$), the naive strategy (empty triangles) and AgACI (black stars) for increasing (from left to right) values of $\varphi$, $\theta$, with $T_1 = 200$. In particular, the top row shows the joint \textit{validity} \& \textit{efficiency} and, for this figure only, we add in the bottom row the same graph using the average length after imputation (see details in \Cref{app:syn_expe}) to assess \textit{efficiency} in another way.

When $\gamma$ is small, one observes an undercoverage, which increases when the temporal dependency of $\varepsilon$ increases. Increasing $\gamma$ enables ACI to increase the interval's size faster when we do not cover, and thus to improve \textit{validity}, which is  achieved for high values of $\gamma$; however this also increases the frequency of uninformative (infinite) intervals, as deduced from the bottom row of \Cref{fig:acp_gamma}, where the average length after imputation grows with $\gamma$. Remark that these results do not contradict the validity result recalled at the beginning of \Cref{sec:theory}, which is only asymptotic while we predict on 200 points.
For $\varphi, \theta$ small, we observe that similarly to \Cref{thm:iid}, the \textit{efficiency} does not improve with $\gamma$.
For moderate values of $\varphi,\theta \in \{0.8, 0.9, 0.95\}$, we observe that the average median length is decreasing with $\gamma$ for $\gamma \ge 0.01$.
This effect is observable on average but not present in all the 500 experiments. 
One possible explanation is that the shrinking effect of ACI on the predicted interval enables to significantly reduce the predicted interval when $\gamma$ is large, and this effect is, on average, more important than the number of large intervals. 

Moreover, the naive strategy is clearly not \textit{valid}: indeed it results in greedily choosing a $\gamma$ that achieved good results in the past, and is consequently slightly more likely to fail to cover in future steps. Thereby, we do not consider it anymore. Finally, AgACI achieves \textit{valid} coverage without increasing the median length with respect to each expert, and even improves the coverage. Overall, it appears to be a good candidate as a parameter-free method.

\subsection{Description of baseline methods}
\label{sec:baselines}
We consider as baseline \textit{online sequential split conformal prediction} (OSSCP), a generalisation of SCP. The other competitor is EnbPI \citep{pmlr-v139-xu21h}, specifically designed for time series. Pseudo-codes and details are given in \Cref{app:comp_methods}. Offline SCP (for which $\rm{Tr}_t \equiv \rm{Tr}_0$ and $\rm{Cal}_t \equiv \rm{Cal}_0$)
is not considered as a competitor because it is unfair to compare an \textit{offline} algorithm to one that uses more recent data points. This corresponds to comparing a prediction at horizon $T_{\text{large}}$ to one at horizon $T_{\text{small}}$. This is a limitation of the comparison in \citet{pmlr-v139-xu21h}. 

\textbf{OSSCP.}
\label{sec:osscp}
We consider an online version of SCP by refitting the underlying regression model and recalibrating using the newest points. Moreover, to appropriately account for the temporal structure of the data, we use a \textit{sequential split} as in \citet{pmlr-v128-wisniewski20a}: at any $t$, the time indices in $\rm{Tr}_t$ are smaller than those of $\rm{Cal}_t$. Not randomizing aims at excluding future observations from $\rm{Tr}_t$, which may lead to an under-estimation of the errors on $\rm{Cal}_t$, thus eventually to smaller intervals with under-coverage. We compare both splitting strategies on simulations in \Cref{app:randomized_sequential}. OSSCP procedure is schematised in \Cref{fig:scheme_both}(a).

\begin{figure}
\centering
\includegraphics[scale=0.285]{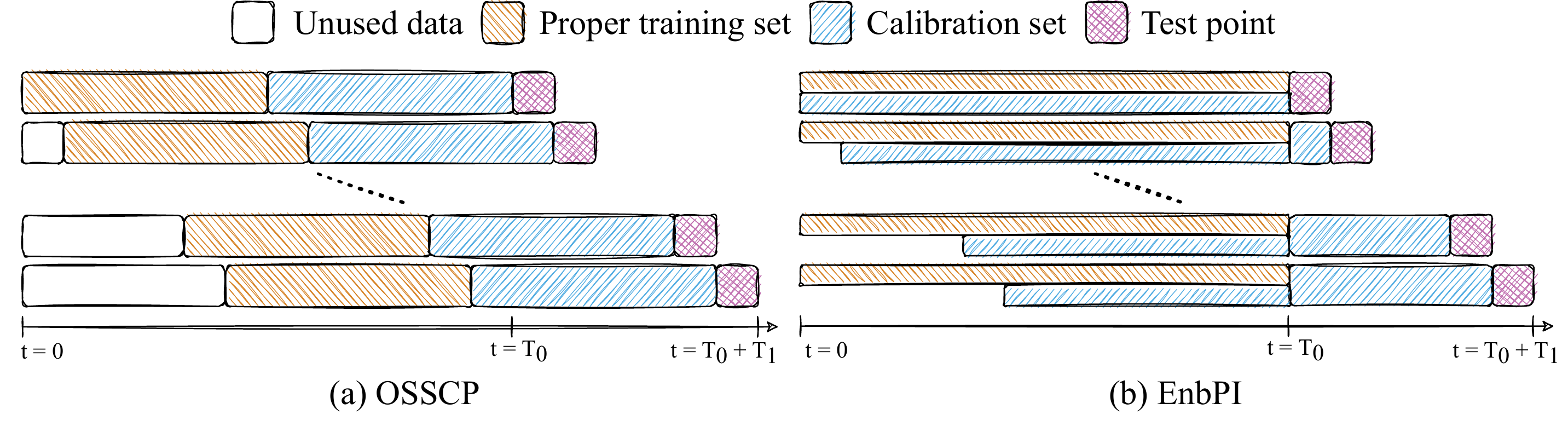}
\gs\gs\gs
\caption{Scheme of the two baselines: OSSCP and EnbPI. In (a), \textcolor{blindorange}{$\rm{Tr}$} and \textcolor{blindblue}{$\rm{Cal}$} have equal size, but it can be changed.}
\label{fig:scheme_both}
\end{figure}

\textbf{Original EnbPI.}
EnbPI, Ensemble Prediction Interval \citep{pmlr-v139-xu21h}, works by updating the list of \textit{conformity scores} with the most recent ones so that the intervals adapt to latest performances, without refitting the underlying regression model. Thereby, the predicted intervals can adapt to seasonality and trend. In EnbPI, $B$ bootstrap samples of the training set are generated and the regression algorithm is fitted on each bootstrap sample producing $B$ predictors. Finally, the predictors are aggregated in two ways: first, for each training point of index $t\le T_0$, EnbPI aggregates only the subset of predictors trained on bootstrap sample \textit{excluding} $(x_t,y_t)$. This way, EnbPI constructs a set of hold-out calibration scores. Second, for test points of index ${t > T_0}$ EnbPI aggregates all the $B$ predictors. A sketch of EnbPI is presented in \Cref{fig:scheme_both}(b). Note that in \citet{pmlr-v139-xu21h} they use a classical bootstrap procedure, not dedicated to time series.

They show empirically that it leads to \textit{valid} coverage on real world time series, such as hourly wind power production and solar irradiation, while offline SCP fails to attain \textit{valid} coverage. 

\textbf{EnbPI V2.} \citet{pmlr-v139-xu21h} used the mean aggregation during the training phase and the $(1-\alpha)$-th quantile of the predictors for the prediction. We consider using the mean aggregation all along the procedure as mixing both aggregations may hurt the performance of the algorithm (as shown in the following simulations). Note that simultaneously to our work, authors released an updated version on ArXiv \citep{xu2021conformal}, incorporating a similar change. 

\subsection{Experimental results: impact of $\varphi, \theta$}
\label{sec:comp_syn}
\Cref{fig:arma_fixed_10} presents the results for data generated as in \Cref{sec:data_syn}, for various $(\varphi,\theta)$. Each sample contains 300 observations, with $T_0 = 200$ and $T_1 = 100$. We compare AgACI (with $K=30$ experts),  ACI (with $\gamma \in \{0.01,0.05\}$), OSSCP, EnbPI and EnbPI V2 (with mean aggregation). To assess the impact and interest of an online procedure, we also add offline SCP. Finally, to ensure the robustness of our conclusions each experiment is repeated $n=500$ times, and \Cref{fig:arma_fixed_10} includes the standard errors  (given by $\frac{\hat{\sigma}_n}{\sqrt{n}}$, where $\hat{\sigma}_n$ is the empirical standard deviation). 

Each color is associated to a set $(\varphi,\theta)$, each marker to an algorithm. To improve readability, we often link markers of the same method. There are thus two ways of analysing \Cref{fig:arma_fixed_10}: for a given method, the lines highlight the evolution of its performance with $(\varphi,\theta)$; for a given data distribution, the set of markers of its color allow to compare the methods. \Cref{fig:arma_fixed_10}, and results on AR(1) in \Cref{app:syn_expe_baselines_10}, highlight that in an AR(1) or ARMA(1,1) process: 
    \begin{itemize}[topsep=0pt,itemsep=1pt,leftmargin=*,noitemsep]
        \item Refitting the method (OSSCP vs Offline SCP) brings a significant improvement, that increases with higher dependence (higher values for $\varphi$ and $\theta$).
        \item All methods produce smaller intervals for $\varphi = \theta = 0.99$.
        \item EnbPI looses coverage while producing shorter intervals when the dependence increases. The performance of EnbPI depends significantly on the type and strength of dependence. 
        \item EnbPI V2 is closer to the target coverage than original EnbPI. 
        \item OSSCP looses \textit{validity} \& coverage as $\varphi$ and $\theta$ increase.
        \item While ACI with $\gamma = 0.01$ also struggles for high values of $\varphi$ and $\theta$ such as 0.99, we observe that it still attains \textit{valid} coverage with a well chosen $\gamma$. Most importantly, ACI performances are robust to the increase of the dependence strength: except for the \textcolor{blindgreen}{$\varphi = \theta = 0.99$}, its markers are really close to each other.
        \item AgACI always nearly attains \textit{validity} (coverage is over $89.8\%$ for all $\varphi$), and achieves the best \textit{efficiency} performance among \textit{valid} methods.
    \end{itemize}
Note that  ACI's \textit{valid} coverage with some $\gamma$ comes at the price of predicting more infinite intervals. A more detailed analysis on this phenomenon is conducted in \Cref{app:acp_inf}. This can also be observed in graphs obtained with the average length after imputation, which are similar to \Cref{fig:arma_fixed_10} and \Cref{app:syn_expe_baselines_10}. In these graphs, the \textit{validity} remains unchanged as expected, but the \textit{efficiency} is more degraded for ACI with $\gamma = 0.05$ and for AgACI, since they produce more often uninformative intervals, as observed in \Cref{fig:acp_gamma}. 

\begin{figure}
\centerline{\includegraphics[scale=0.28]{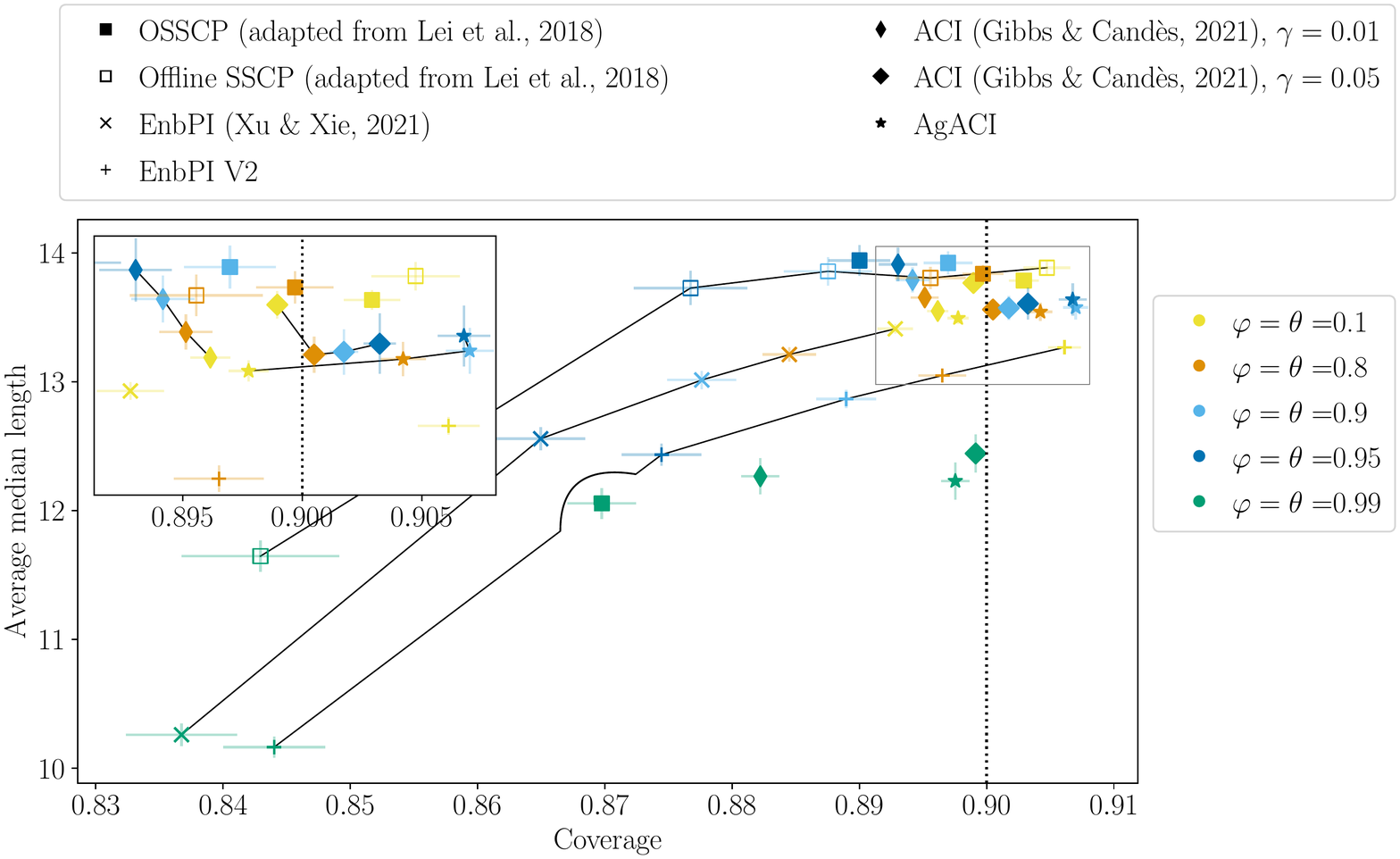}}
\gs\gs
\caption{Performance of various CP methods on data simulated according to \cref{eq:friedman} with a Gaussian ARMA(1,1) noise of asymptotic variance 10 (see \Cref{app:arma}). Results aggregated from 500 independent runs. Empirical standard errors displayed.}
\label{fig:arma_fixed_10}
\end{figure}

\textbf{Summary.} We highlight the following takeaways:
\begin{enumerate}[topsep=0pt,itemsep=1pt,leftmargin=*,noitemsep]
    \item The temporal dependence impacts the \textit{validity}.
    \item Online is significantly better than offline.
    \item \textbf{OSSCP.} Achieves \textit{valid} coverage for $\varphi$ and $\theta$ smaller than 0.9, but is not robust to the increasing dependence.
    \item \textbf{EnbPI.} Its \textit{validity}  strongly depends on the data distribution (it is \textit{valid} on a MA(1) noise, not in AR(1) and ARMA(1,1) noise). When the method is \textit{valid}, it produces the smallest intervals. EnbPI V2 method should be preferred.
    \item \textbf{ACI.} Achieves \textit{valid} coverage for every simulation settings with a well chosen $\gamma$, or for dependence such that $\varphi < 0.95$. 
    It is robust to the strength of the dependence.
   \item \textbf{AgACI.} Achieves \textit{valid} coverage for every simulation settings, with good \textit{efficiency}.
\end{enumerate}

\section{Forecasting French electricity spot prices}
\label{sec:comp_real}
In this last section, the task of forecasting French electricity spot prices with predictive intervals is considered in order to assess the methods on a real time series, and most importantly to show the relevance of ACI and AgACI in practice for time series without distribution shifts.

\subsection{Presentation of the price data}
\label{sec:data_real}

The data set contains the French electricity spot prices, set by an auction market, from 2016 to 2019. Each day $D$ before 12 AM, any producer (resp. supplier) submit their orders for the 24 hours of day $D+1$. An order consists of an electricity volume in MWh  offered for sale (resp. required to be purchased) and a price in \euro/MWh, at which they accept to sell (resp. buy) this volume. At 12 AM, the algorithm ``Euphemia'' \citep{euphemia} fixes the 24 hourly prices of day $D+1$ according to these offers and additional constraints. Thereby, it is an hourly data set, containing $(3 \times 365 + 366) \times 24 = 35 064$ observations. Our aim is to predict at day $D$ (before 12 AM) the 24 prices of day $D+1$. Given the prices' construction, we consider the following explanatory variables: day-ahead forecast consumption, day-of-the-week, 24 prices of the day $D-1$ and 24 prices of the day $D-7$. An extract of the considered data set is presented in \Cref{app:price_data}. 

\begin{figure}[!h]
\centerline{\includegraphics[width=\linewidth]{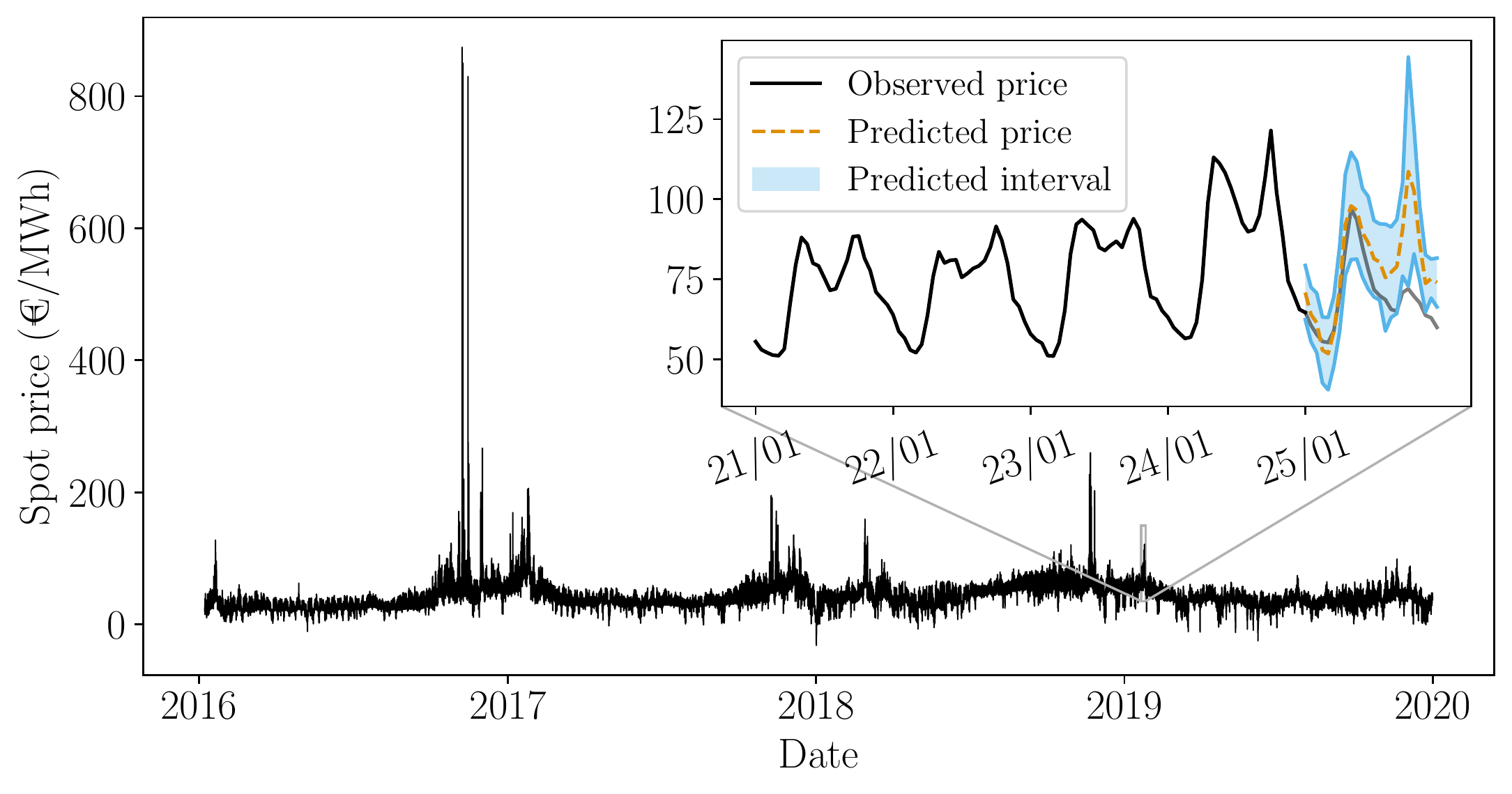}}
\gs\gs\gs
\caption{French electricity spot prices, from 2016 to 2019. Predicted intervals on the 25th of January 2019, using AgACI.}
\label{fig:spot}
\end{figure}
These prices exhibits medium to high peaks, as illustrated in \Cref{fig:spot} where the French prices had reached 800 \euro/MWh in fall 2016, compared to an average price of approximately 40 \euro/MWh in 2019. These extreme events are mainly due to the non-storability of electricity and the inelasticity of the demand: when the demand is high compared to the available production, production units with expensive production costs must be called, leading to a huge market price.

\subsection{Price prediction with predictive intervals in 2019}

Since the 24 hours have very distinct patterns, we fit one model per hour, using again RF. We predict for the year 2019, using a sliding window of 3 years, as described in \Cref{fig:scheme_both}(a), using one year and 6 months as proper training set and the most recent year and a half for calibration. The results are represented in \Cref{fig:comp_spot}.

\begin{figure}
\centerline{\includegraphics[scale=0.28]{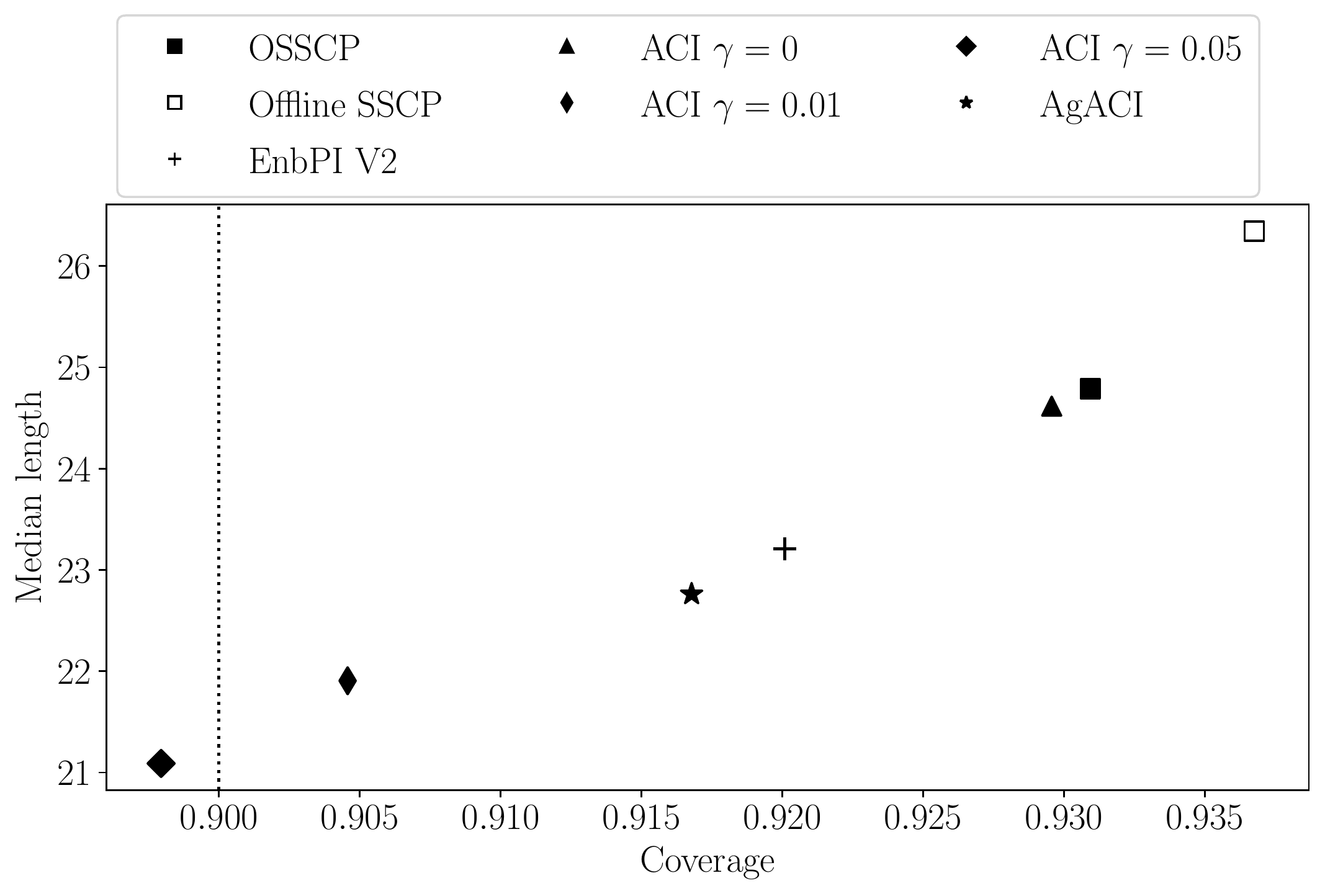}}
\caption{Performance of different CP methods on hourly spot electricity prices in France, trained from 2016 to 2018 and forecasted on 2019. Median length with respect to empirical coverage.}
\label{fig:comp_spot}
\end{figure}

\textbf{OSSCP} over-covers but to a lesser extent than the offline version. This can be explained by a low presence of peaks during the test period. Indeed, by updating the whole procedure, the high peaks are ``forgotten" which leads to small intervals while it is not the case for the offline version which leads to too large intervals. Thereby, online versions can help to improve $\textit{efficiency}$, in addition to $\textit{validity}$. \textbf{EnbPI} attains a \textit{valid} coverage by over-covering. The under-coverage observed in the simulation study is not systematic, as in \citet{pmlr-v139-xu21h}. \textbf{ACI} gives the smallest intervals with a correct coverage, for $\gamma = 0.01$ and $\gamma = 0.05$. The update of the quantile level  enables to shrink the intervals. While the simulation in \Cref{sec:comp_syn} study outlines that ACI improves \textit{validity}, this application illustrates that it can provide \textit{efficient} interval. \textbf{AgACI} is more \textit{efficient} than $\gamma = 0$ while maintaining \textit{validity}. Yet it slightly over-covers, and is slightly less \textit{efficient} than ACI with a well chosen~$\gamma$.

An illustration of the predicted intervals is given in the inset graphic of \Cref{fig:spot}, for AgACI, to highlight the practical relevance of such an approach on the spot prices.

\section{Conclusion}
This article shows why and how ACI can be used for interval prediction in the context of time series with general dependencies. We prove that ACI  deteriorates \textit{efficiency} compared to CP in the exchangeable case and analyse the dependency on $\gamma$ in the AR case with the support of numerical simulations. We propose an algorithm, AgACI, based on online expert aggregation, that wraps around ACI to avoid the choice of $\gamma$. We conduct extensive experiments on synthetic time series for various strengths and structures of time dependence, demonstrating ACI's robustness and better performances than baselines, with well chosen $\gamma$ or using AgACI. Finally we perform a detailed application study on the high-stakes electricity price forecasting problem in the energy transition era. Future work includes theoretical study of the proposed aggregation algorithm, including whether it preserves the asymptotic \textit{validity} observed experimentally or to quantify its \textit{efficiency} with respect to the performances of each {expert}. 

\section*{Acknowledgements}

We thank Maximilien Germain, Pablo Jiménez and Constantin Philippenko for interesting discussions. The work of A. Dieuleveut is partially supported by ANR-19-CHIA-0002-01/chaire SCAI.

\bibliography{ref}
\bibliographystyle{apalike}

\newpage
\appendix
\onecolumn

\part*{Appendices}
The appendices are organized as follows. First, \Cref{app:scp} provides details about the Split Conformal Prediction procedure. Second,  \Cref{app:aci_theory} proves the results of \Cref{sec:theory} and conducts the numerical analysis of \Cref{sec:theory_ar} in the case where the \textit{efficiency} is computed using the median length. Then, \Cref{app:experimental_details} contains details on the experimental setup (hyper-parameters, settings, pseudo-codes of competing algorithms). Finally, \Cref{app:syn_expe,app:price} contain complementary numerical results, respectively on synthetic data sets and on the French electricity spot prices data set.

\section{Details on Split Conformal Prediction}
\label{app:scp}

In this section, we introduce and review the simplest theoretical properties of Split Conformal Prediction  (SCP). More specifically, we present the whole algorithm, the theoretical guarantees and discuss the visualisation challenges arising when visualising a CP procedure. 

\subsection{Split Conformal Prediction Algorithm}
\label{app:scp_alg}

\begin{algorithm}
\caption{Split Conformal Algorithm, with absolute value residuals scores}
\label{alg:SCP}
\begin{algorithmic}[1] 
\REQUIRE Regression algorithm $\mathcal{A}$, significance level $\alpha$, examples $z_{1}, \ldots, z_{T}$ with $z_t = \left(x_t,y_t\right)$. 
\ENSURE Prediction interval $\mathcal{\hat{C}}_{\alpha}(x) \text { for any } x \in \mathds{R}^d$. 
\STATE Randomly split $\{1, \ldots, T\}$ into two disjoint sets $\rm{Tr}$ and $\rm{Cal}$. 
\STATE Fit a mean regression function: $\hat{\mu}(\cdot) \leftarrow \mathcal{A}\left(\left\{z_t, t \in \rm{Tr}\right\}\right)$ 
\FOR {$j \in \rm{Cal} $} 
\STATE Set $s_j = |y_j-\hat{\mu}(x_j)|$, the \textit{conformity scores}
\ENDFOR
\STATE Set $S_{\rm{Cal}} = \{s_j, j\in \rm{Cal}\}$
\STATE Compute $\widehat Q_{1-\alpha^{\rm{SCP}}}\left(S_{\rm{Cal}}\right)$, the $1-\alpha^{\rm{SCP}}$-th empirical quantile of $S_{\rm{Cal}}$, with $1-\alpha^{\rm{SCP}} := (1-\alpha)\left(1+1/\left|\rm{Cal}\right|\right)$.
\label{ln:cp_quantile}
\STATE Set  $\mathcal{\hat{C}}_{\alpha}(x) = \left[\hat{\mu}(x)\pm \widehat Q_{1-\alpha^{\rm{SCP}}}\left(S_{\rm{Cal}}\right)\right]$, for any $x \in \mathds{R}^d$. 
\end{algorithmic}
\end{algorithm}
 
\subsection{Theoretical guarantees of Split Conformal Prediction}
\label{app:guarantees_scp}

Conformal prediction relies on the assumption that the data is exchangeable. 
\begin{definition}[Exchangeability]
$\left( X_t, Y_t \right) _{t = 1}^T$ are exchangeable if for any permutation $\sigma$ of $\llbracket 1,T \rrbracket$ we have:
\begin{equation*}
\mathcal{L}\left( \left(X_1,Y_1\right), \dots, \left(X_T,Y_T\right) \right) = \mathcal{L}\left( \left(X_{\sigma(1)},Y_{\sigma(1)}\right), \dots, \left(X_{\sigma(T)},Y_{\sigma(T)}\right) \right),
\end{equation*} 
where $\mathcal{L}$ designates the joint distribution.
\end{definition}

\citet{lei_distribution-free_2018} proves the following \Cref{thm:icp_guarantees} about SCP quasi-exact \textit{validity}.
\begin{theorem}
Suppose $\left( X_t, Y_t \right) _{t = 1}^{T+1}$ are exchangeable, and we apply \cref{alg:SCP} on $\left( X_t, Y_t \right) _{t = 1}^T$ to predict an interval on $X_{T+1}$, $\hat{\mathcal{C}}_{\alpha}\left(X_{T+1}\right)$. Then we have:\begin{equation*}
\mathds{P}\left\{Y_{T+1} \in \hat{\mathcal{C}}_{\alpha}\left(X_{T+1}\right)\right\} \geq 1-\alpha.
\end{equation*} 
If, in addition, the scores $S_{\rm{Cal}}$ have a continuous joint distribution, we also have an upper bound:
\begin{equation*}
\mathds{P}\left\{Y_{T+1} \in \hat{\mathcal{C}}_{\alpha}\left(X_{T+1}\right)\right\} \leq 1-\alpha +\frac{2}{T+2}.
\end{equation*} 
\label{thm:icp_guarantees}
\end{theorem}

\subsection{Examples of dependent scores when data noise is exchangeable}
\label{app:scp_examples}
In this subsection, we provide two examples that highlight the importance of adapting CP to time series. In these examples, the scores are non exchangeable while the true noise of the data is exchangeable. 

\begin{example}[Endogenous and not perfectly estimated] 
\label{ex:endo}
Assume ${X_t = Y_{t-1} \in \mathds{R}}$ and that
\begin{equation*}
    Y_t = aY_{t-1} +\varepsilon_t,
\end{equation*}
where $\varepsilon_t$ is a white noise. This corresponds to an order-1 Auto-Regressive (i.e. AR(1)). 

Assume that the fitted model is $\hat{f}_t(x) = \hat{a}x$, with $\hat{a} \neq a$. Then, for any $t$, we have that:
\begin{align*}
    \hat{\varepsilon}_t & = Y_t - \hat{Y}_t = \left(a-\hat{a}\right)Y_{t-1} + \varepsilon_t \\
    \hat{\varepsilon}_t & = a \hat{\varepsilon}_{t-1} + \xi_t
\end{align*}
with $\xi_t = \varepsilon_t - \hat{a} \varepsilon_{t-1}$. 

The residual process $\left(\hat{\varepsilon}_t\right)_{t \geq 0}$ is an ARMA(1,1) (Auto-Regressive Moving-Average, see section \ref{app:arma}) of parameters $\varphi = a$ and $\theta = -\hat{a}$.

Thus, we have generated dependent residuals (ARMA residuals) even though the underlying model only had white noise.
\qed
\end{example}
\begin{example}[Exogenous and misspecified]
Assume $X_t \in \mathds{R}^2$ and that:
\begin{equation*}
    Y_t = a X_{1,t} + b X_{2,t} + \varepsilon_t,
\end{equation*}
with $\varepsilon_t \underset{\text{i.i.d.}}{\sim} \mathcal{N}(0,1)$, $X_{2,t+1} = \varphi X_{2,t} + \xi_t, \; \; \xi_t \underset{\text{i.i.d.}}{\sim} \mathcal{N}(0,1)$ and $X_{1,t}$ can be any random variable. 

Assume that we misspecify the model such that the fitted model is $\hat{f}_t(x) = ax_{1}$ for any $t \geq 0$. Then, for any $t \geq 0$, we have that
\begin{equation*}
\hat{\varepsilon}_t = Y_t - \hat{Y}_t = b X_{2,t} + \varepsilon_t.
\end{equation*}
Thus, we have generated dependent residuals (Auto-Regressive residuals) even if the underlying model only had i.i.d. Gaussian noise. 
\qed
\label{ex:exo}
\end{example}

\subsection{How should we visualise CP predicted intervals?}
\label{app:how_to_visu}

We propose to have a closer look at how are constructed the prediction of this method. In this aim, we introduce \cref{mod:1}.

\begin{model}
\begin{equation*}
\begin{aligned}
    x_t & = \cos{\left(\frac{2 \pi}{180} t\right)} + \sin{\left(\frac{2 \pi}{180} t\right)} + \frac{t}{100} \\
    \varepsilon_{t+1} & = 0.99 \varepsilon_t + \xi_{t+1}, \; \; \; \xi_t \sim \mathcal{N}(0,0.01) & . \\
    Y_t & = f_t(x_t) + \varepsilon_t = x_t + \varepsilon_t
\end{aligned}
\end{equation*}
\label{mod:1}
\end{model}

In this \cref{mod:1}, the explanatory variables are deterministic. A generation from this model is represented in \Cref{fig:visu_data}. The first subplot, \Cref{fig:visu_data_reg}, represents $x_t$ across time. The second subplot, \Cref{fig:visu_data_noise}, represents the noise $\varepsilon_t$ across time. Finally, the last subplot, \Cref{fig:visu_data_entire}, represents the whole process $Y_t$ across time.

\begin{figure}[!h]
     \centering
     \begin{subfigure}[$f_t(x_t)$]{
         \includegraphics[width=0.25\textwidth]{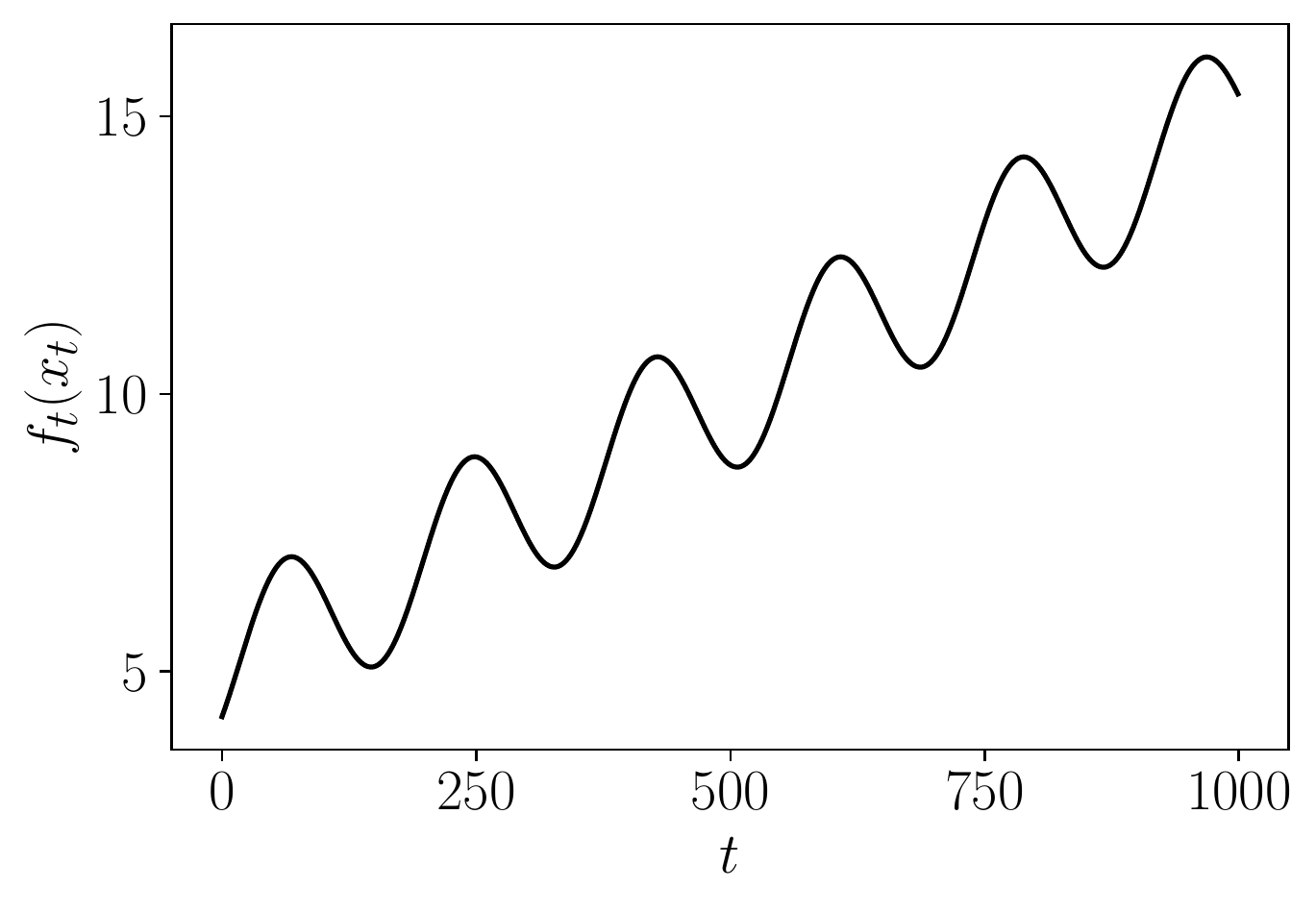}
         \label{fig:visu_data_reg}
     }
     \end{subfigure}
     \begin{subfigure}[$\varepsilon_t$]{
         \includegraphics[width=0.26\textwidth]{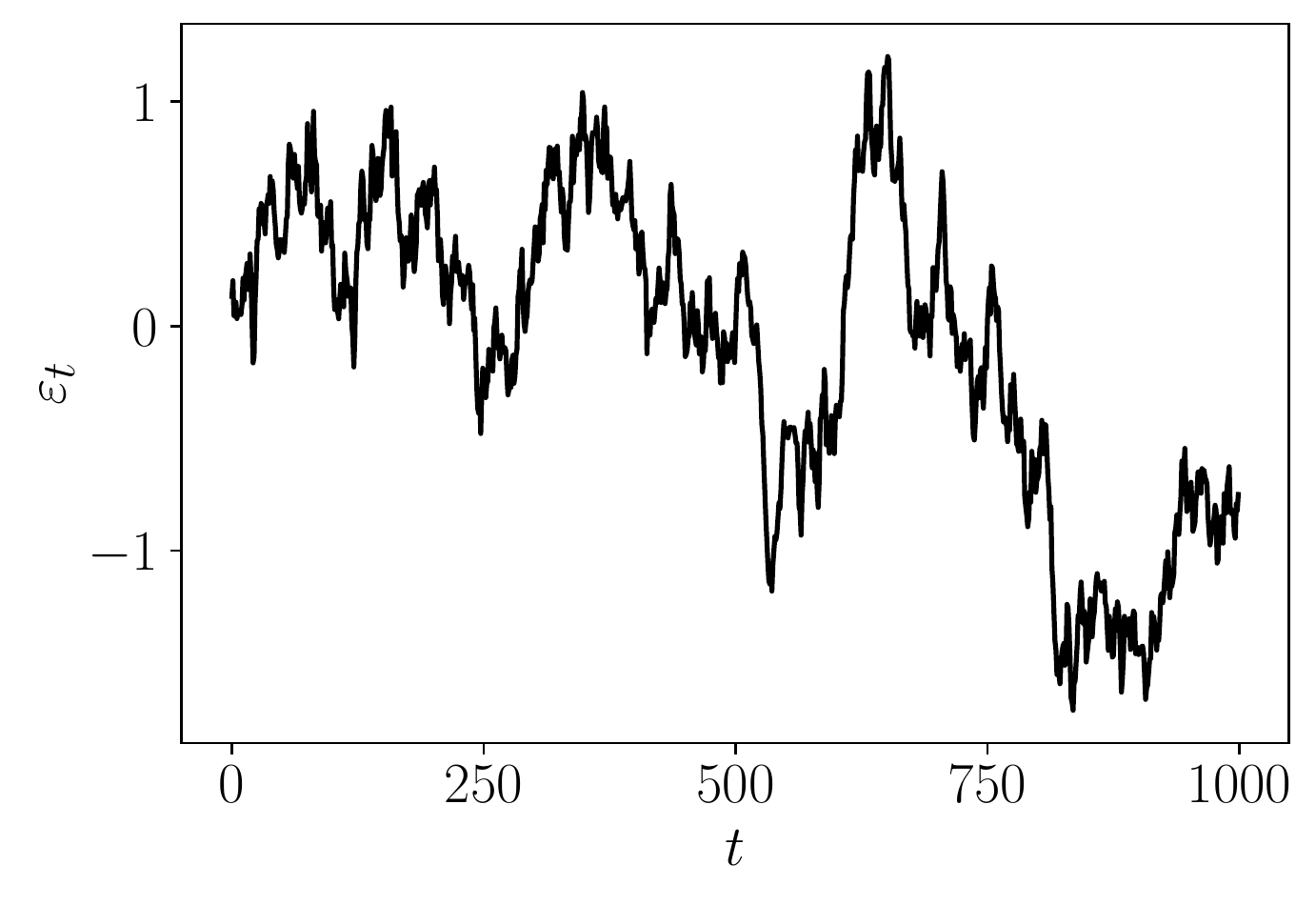}
         \label{fig:visu_data_noise}
     }
     \end{subfigure}
     \begin{subfigure}[$y_t$]{
         \centering
         \includegraphics[width=0.25\textwidth]{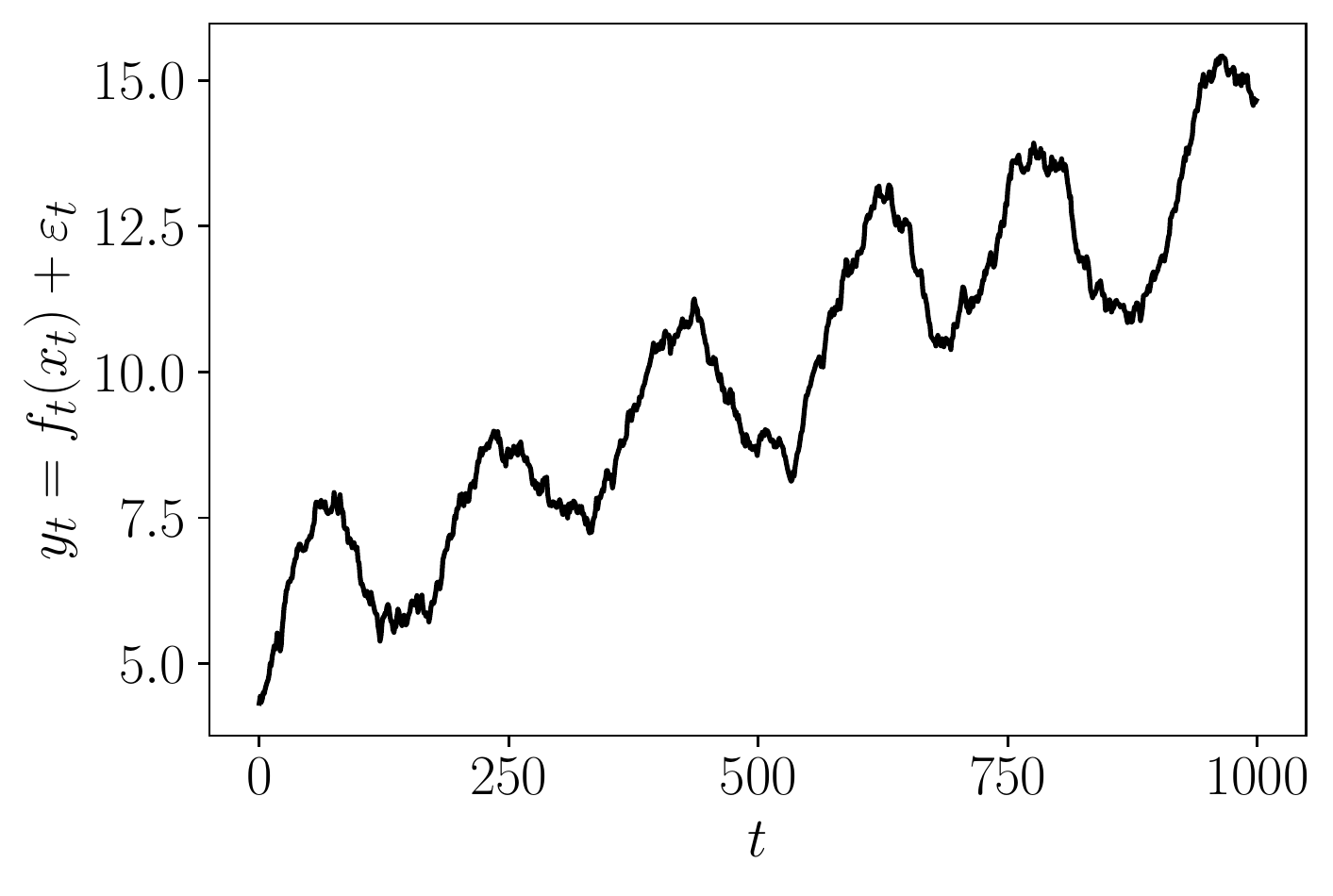}
         \label{fig:visu_data_entire}
        } 
     \end{subfigure}
        \caption{Representation of data simulated according to \cref{mod:1}.}
        \label{fig:visu_data}
\end{figure}

The aim is to predict intervals of coverage 0.9 for values of $Y_t$, at $t > 500$, that is to say $T_0 = 500$ here. For simplicity, we assume $\hat{f_t} = f_t$ at each time step $t$ and we do not represent the points used to obtain this perfect regression model. There are two ways of visualizing the predictions, that are represented in each row of \Cref{fig:visu_oscp}. If the focus of the analysis is on a specific application with the aim of analysing the whole prediction, it is relevant to represent the response $y_t$ itself and the associated intervals. This is represented in the first row of \Cref{fig:visu_oscp}. Nevertheless, to better understand a CP method, it is relevant to represent the scores and the corresponding intervals, rescaled. This is represented in the second row of \Cref{fig:visu_oscp} (even if the residuals are displayed and not their absolute value, i.e. the scores). 

\begin{figure*}[!h]
\centerline{\includegraphics[scale=0.25]{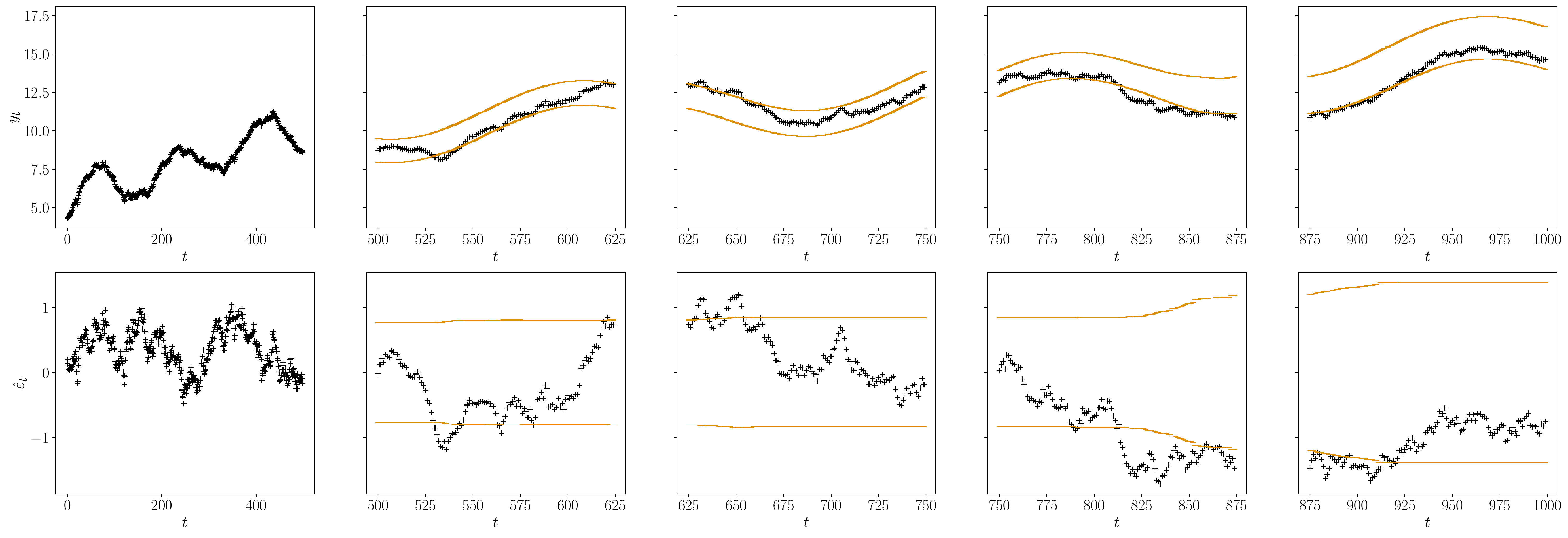}}
\caption{Visualisation of OSSCP on simulated data, from model \cref{mod:1}. 1000 data points are generated. The 500 first ones form the initial calibration set, displayed on the first subplot of each row. The 500 last ones are the ones the algorithm tries to predict. They are displayed on the 4 last subplots of each row: the second displays response from $t = 501$ to $t = 625$, the third from $t = 626$ to $t = 750$ and so on. Observed values are in black, predicted intervals bounds are displayed in orange}
\label{fig:visu_oscp}
\end{figure*}

To better understand the difference between the two visualizations, let's look specifically at some observations. In the first line of the \Cref{fig:visu_oscp}, we can see that the intervals widen for $t \in [801;900]$, while struggling to include the observations. Nevertheless, it is difficult to understand the underlying phenomenon on such a plot. Indeed, the points seem very similar to those for $t \in [660;720]$. What considerably influences the CP are the scores and not the observed values. Thus, in the second line, at times $t \in [801;900]$, we observe more clearly that the values go out of the previous range of values, being around 1.5 in absolute value. This explains why the intervals widen: the calibration set contains more and more high values, which increases the value of the quantile and, therefore, the length of the interval.
To conclude, to analyse and assess the performances of CP procedures, we recommend representing  the intervals around the \textit{conformity scores} (or the residuals, depending on the score function) rather than the observed values. This is because the scores are what truly determine the conformal behaviour. 

\section{Proof of the results presented in \Cref{sec:theory} and additional numerical experiments}
\label{app:aci_theory}
\subsection{Proof of \Cref{thm:iid}}
\label{subsec:proofiid}

We recall here \Cref{thm:iid}.

\begin{customTheorem}[\ref{thm:iid}]
Assume that: (i) $\alpha \in \mathds{Q}$; 
    (ii) the scores are exchangeable with quantile function $Q$; (iii) the quantile function is perfectly estimated at each time (as defined above); (iv) the quantile function $Q$ is bounded and $\mathcal{C}^4([0,1])$.
Then, for all $\gamma>0$,  $\left(\alpha_t\right)_{t > 0}$ forms a Markov Chain, that admits a stationary distribution $\pi_\gamma$, and
\begin{equation*}
    \frac{1}{T}\sum\limits_{t=1}^T L(\alpha_t) \overset{a.s.}{\underset{T \rightarrow +\infty}{\longrightarrow}} \mathds{E}_{\pi_\gamma}[L] \overset{\text{not.}}{=} \mathds{E}_{\tilde \alpha \sim  \pi_\gamma}[L(\tilde \alpha)].
\end{equation*}
Moreover, as $\gamma \to 0$, 
\begin{equation*}
      \mathds{E}_{\pi_\gamma}[L] = L_0 + Q''(1-\alpha)\frac{\gamma}{2}\alpha(1-\alpha) + O(\gamma^{3/2}).
\end{equation*}
\end{customTheorem}

To prove \Cref{thm:iid}, we rely on the following lemmas, that will be proved after the theorem. We denote $B_\beta$ a Bernoulli random variable of parameter $\beta$ and $P(x)$ designates the projection of $x$ onto $[0,1]$. Finally, for $\gamma > 0$, define the following Markov Chain:
\begin{equation}
    \alpha_{t+1} = \alpha_t + \gamma\left(\alpha - B_{P(\alpha_t)}\right) \text{for }t > 0, \\
    \label{eq:def_MC}
\end{equation}

We introduce $(p,q) \in \mathds{N}\times\mathds{N}^*$, $p < q$, s.t.~$\alpha=\frac{p}{q}$, and $\mathcal A=  \left\{ \alpha + \gamma\frac{\text{gcd}(q-p,p)}{q} \mathds{Z} \right\} \cap\ \  ]\gamma(\alpha-1),1+\gamma\alpha[.$

\begin{lemma}[Finite state space] 
Assume that $\alpha \in \mathds{Q}$. Then, for any $\gamma > 0$, the Markov Chain defined by $\alpha_1\in \mathcal A$ and ${\alpha_{t+1} = \alpha_t + \gamma\left(\alpha - B_{P(\alpha_t)}\right)}$, for $t > 0$  has a finite state space $\mathcal A$. 
\label{lem:iid_finite}
\end{lemma}

\begin{lemma}[Irreducibility] 
Assume that $\alpha \in \mathds{Q}$. Then, for any $\gamma > 0$, the Markov Chain defined by \Cref{eq:def_MC}, for $t > 0$ and $\alpha_1\in \mathcal A$, is irreducible.
\label{lem:iid_irr}
\end{lemma}

Thereby we will prove that the chain admits a unique stationary distribution $\pi_\gamma$, we now compute the first four moments of the stationary distribution in \Cref{lem:iid_mom1,lem:iid_mom2,lem:iid_mom3,lem:iid_mom4}. The final proof relies on a Taylor expansion, that requires to control these four moments.

\begin{lemma}[Expectation]
Let $\gamma > 0$ and consider again the Markov Chain defined in \cref{eq:def_MC}.
We have: 
\begin{equation*}
    \mathds{E}_{\pi_\gamma}\left[(P(\tilde\alpha)-\alpha)\right] = 0.
\end{equation*}
\label{lem:iid_mom1}
\end{lemma}

\begin{lemma}[Second order moment]
Let $\gamma > 0$ and consider again the Markov Chain defined in \cref{eq:def_MC}.
As $\gamma \rightarrow 0$, we have: 
\begin{equation*}
    \mathds{E}_{\pi_\gamma}\left[(P(\tilde\alpha)-\alpha)^2\right] = \frac{\gamma}{ 2} \alpha(1-\alpha) + O(\gamma^2).
\end{equation*}
\label{lem:iid_mom2}
\end{lemma}

\begin{lemma}[Third order moment]
Let $\gamma > 0$ and consider again the Markov Chain defined in \cref{eq:def_MC}.
As $\gamma \rightarrow 0$, we have: 
\begin{equation*}
    \mathds{E}_{\pi_\gamma}\left[(P(\tilde\alpha)-\alpha)^3\right] = O(\gamma^{3/2}).
\end{equation*}
\label{lem:iid_mom3}
\end{lemma}

\begin{lemma}[Fourth order moment]
Let $\gamma > 0$ and consider again the Markov Chain defined in \cref{eq:def_MC}.
As $\gamma \rightarrow 0$, we have: 
\begin{equation*}
    \mathds{E}_{\pi_\gamma}\left[(P(\tilde\alpha)-\alpha)^4\right] = O(\gamma^{3/2}).
\end{equation*}
\label{lem:iid_mom4}
\end{lemma}

The proofs of these Lemmas are postponed to \Cref{subsec:proofLemma,app:proofmoments}. Here, we first give the proof of the main theorem.

\begin{proof}[Proof of \Cref{thm:iid}] Let $\gamma > 0$. For any $t > 0$ we have, for the recursion introduced in \Cref{eq:update_scheme}, that
\begin{equation*}
    \alpha_{t+1} := \alpha_t + \gamma\left(\alpha - \mathds{1}_{y_t \notin \widehat{C}_{\alpha_t}(x_t)} \right) = \alpha_t + \gamma\left(\alpha - \mathds{1}_{S_t > \hat{Q}_{1-P(\alpha_t)} } \right),
\end{equation*}
where $S_t$ is the conformity score at time $t$.
Noting that $\mathds{1}_{S_t > \hat{Q}_t\left(1-P(\alpha_t)\right) } \overset{d}{=} B_{\mathds{P}\left(S_t > \hat{Q}_t\left(1-P(\alpha_t)\right) \right)}$, we obtain:
\begin{align*}
    \alpha_{t+1} & \overset{d}{=} \alpha_t + \gamma\left(\alpha - B_{\mathds{P}\left(S_t > \hat{Q}_t\left(1-P(\alpha_t)\right) \right)} \right) \\
     & \overset{d}{=} \alpha_t + \gamma\left(\alpha - B_{\mathds{P}\left(S_t > Q\left(1-P(\alpha_t)\right) \right)} \right) \\
      & \overset{d}{=} \alpha_t + \gamma\left(\alpha - B_{P(\alpha_t)} \right),
\end{align*}
where the second line results from assumption (ii) and (iii), and the last equation from assumption (iii) only. Consequently, by induction, the chain defined by \Cref{eq:update_scheme} and 
\begin{equation}
    \alpha_{t+1} = \alpha_t + \gamma\left(\alpha - B_{P(\alpha_t)} \right),
     \label{eq:acp_iid}
\end{equation}
with $\alpha_1 = \alpha$, have the same distribution. 

Using assumption (i), \Cref{lem:iid_finite} ensures that the state space $\mathcal{A}$ of the Markov Chain defined in \cref{eq:acp_iid} is finite. Furthermore, \Cref{lem:iid_irr} also ensures that the chain is irreducible. Therefore, the chain is irreducible on a finite state space, thus it admits a unique stationary distribution, noted $\pi_\gamma$ and for any positive function $f$ such that $\int f \mathrm{d} \pi_{\gamma} <\infty$, we have~\citep[Theorem 17.1.7]{meyn2012markov}:
\begin{equation*}
\frac{1}{T} \sum\limits_{t=1}^T f(\alpha_t) \overset{a.s.}{\underset{T \rightarrow \infty}{\longrightarrow}} \int f \mathrm{d} \pi_{\gamma}.
\end{equation*}

Remark that $L(\beta) = 2Q(1-P(\beta))$ for any $\beta$. Therefore, combined with previous result we get the first result of \Cref{thm:iid}:
\begin{equation*}
\frac{1}{T} \sum\limits_{t=1}^T L(\alpha_t) \overset{a.s.}{\underset{T \rightarrow +\infty}{\longrightarrow}} \mathds{E}_{\tilde{\alpha}\sim \pi_\gamma}\left[L(\tilde{\alpha})\right].
\end{equation*}

We now need to characterize $\mathds{E}_{\tilde{\alpha}\sim \pi_\gamma}\left[L(\tilde{\alpha})\right] = 2\mathds{E}_{\tilde{\alpha}\sim \pi_\gamma}\left[Q(1-P(\tilde{\alpha}))\right]$ as $\gamma \rightarrow 0$. Assume that $Q \in \mathcal{C}^4([0,1])$. Using Taylor series expansion, for any $\tilde{\alpha} \in \mathcal{A}$, there exists $\beta(\tilde{\alpha}) \in [0,1]$:
\begin{equation}
\begin{aligned}
Q(1-P(\tilde{\alpha})) = & Q(1-\alpha) + Q'(1-\alpha)(\alpha-P(\tilde{\alpha})) + \frac{Q''(1-\alpha)}{2}(\alpha-P(\tilde{\alpha}))^2 \\
& + \frac{Q'''(1-\alpha)}{6}(\alpha-P(\tilde{\alpha}))^3 + \frac{Q''''(1-\beta(\tilde{\alpha}))}{24}(\alpha-P(\tilde{\alpha}))^4.
\end{aligned}
\label{eq:iid_tay_lag}
\end{equation}
To conclude, we take the expectation under $\pi_\gamma$ of \cref{eq:iid_tay_lag}, which gives:
\begin{equation}
\begin{aligned}
\mathds{E}_{\pi_\gamma}\left[Q(1-P(\tilde{\alpha}))\right] = & Q(1-\alpha) + Q'(1-\alpha)\mathds{E}_{\pi_\gamma}\left[(\alpha-P(\tilde{\alpha}))\right] + \frac{Q''(1-\alpha)}{2}\mathds{E}_{\pi_\gamma}\left[(\alpha-P(\tilde{\alpha}))^2\right] \\
& + \frac{Q'''(1-\alpha)}{6}\mathds{E}_{\pi_\gamma}\left[(\alpha-P(\tilde{\alpha}))^3\right] + \mathds{E}_{\pi_\gamma}\left[\frac{Q''''(1-\beta(\tilde{\alpha}))}{24}(\alpha-P(\tilde{\alpha}))^4\right].
\end{aligned}
\label{eq:iid_tay_lag_exp}
\end{equation}
Injecting results of \Cref{lem:iid_mom1,lem:iid_mom2,lem:iid_mom3} in \cref{eq:iid_tay_lag_exp}, we obtain:
\begin{equation}
\begin{aligned}
\mathds{E}_{\pi_\gamma}\left[Q(1-P(\tilde{\alpha}))\right] = & Q(1-\alpha) + \frac{Q''(1-\alpha)}{4}\gamma\alpha(1-\alpha) + O(\gamma^{3/2})  + \mathds{E}_{\pi_\gamma}\left[\frac{Q''''(1-\beta(\tilde{\alpha}))}{24}(\alpha-P(\tilde{\alpha}))^4\right].
\end{aligned}
\label{eq:iid_tay_lag_exp_mom4}
\end{equation}
Finally, we can control the last term since $Q \in \mathcal{C}^4([0,1])$ by assumption, thus there exists $M > 0$ such that for any $x \in [0,1]$, $|Q''''(1-x)|<M$. Hence, using \Cref{lem:iid_mom4} we obtain:
\begin{align*}
    \left| \mathds{E}_{\pi_\gamma}\left[Q''''(1-\beta(\tilde{\alpha}))(\alpha-P(\tilde{\alpha}))^4\right] \right| & \leq \mathds{E}_{\pi_\gamma}\left[\left| Q''''(1-\beta(\tilde{\alpha})) \right| (\alpha-P(\tilde{\alpha}))^4\right] \\
    & \leq M \mathds{E}_{\pi_\gamma}\left[(\alpha-P(\tilde{\alpha}))^4\right] \\
    & \leq M O(\gamma^{3/2}) \\
    \mathds{E}_{\pi_\gamma}\left[Q''''(1-\beta(\tilde{\alpha}))(\alpha-P(\tilde{\alpha}))^4\right] & = O(\gamma^{3/2}). \numberthis \label{eq:iid_thm_mom4}
\end{align*}
Finally, combining \cref{eq:iid_tay_lag_exp_mom4,,eq:iid_thm_mom4} to conclude the proof by obtaining:
\begin{equation}
\mathds{E}_{\pi_\gamma}\left[Q(1-P(\tilde{\alpha}))\right] = Q(1-\alpha) + \frac{Q''(1-\alpha)}{4}\gamma\alpha(1-\alpha) + O(\gamma^{3/2}).
\end{equation}

\end{proof}
This concludes the proof of \Cref{thm:iid}.

\textbf{Remark: is it possible to use only 3 moments?} The proof here relies on the control of the first four moments. It is not clear that the same result could be obtained using only a third order Taylor expansion, as we would then require a bound on $\mathds E[|P(\tilde\alpha)-\alpha|^3]$, which is \textit{not} guaranteed to be $O(\gamma^{3/2})$, contrary to $\mathds E[(P(\tilde\alpha)-\alpha)^3]$.

\subsection{Proof of~\Cref{lem:iid_irr,lem:iid_finite}} \label{subsec:proofLemma}

\begin{proof}[Proof of \Cref{lem:iid_finite}]
Let $\gamma > 0$ and denote $\alpha = \frac{p}{q}$ with $0 < p < q$ and $(p,q) \in \mathds{N}^2$. We denote $E$ the state space of the  Markov Chain defined by \cref{eq:acp_iid}, starting from $a\in \mathcal A$. We show that $E = \mathcal A$.

First, $(\alpha_t)$ is stritcly bounded by $\gamma (\alpha - 1)$ and $1 + \gamma \alpha$. Thus $E \subset ]\gamma(\alpha-1),\gamma\alpha[$. Secondly, for any starting point $\alpha_1\in \mathcal{A}$, we can observe that:
\begin{align*}
\{\alpha_t, {t\geq 1}\} & \overset{a.s.}{\subset} \alpha_1 + \{k \gamma(\alpha-1) + n\gamma\alpha, (k,n)\in\mathds{N}^2 \} \\
& \subset \alpha_1 + \{ k \gamma(\alpha-1) + n\gamma\alpha, (k,n)\in\mathds{Z}^2 \} \\
& = \alpha_1 + \{ k \gamma \frac{p-q}{q} + n\gamma\frac{p}{q}, (k,n)\in\mathds{Z}^2 \} \\
& = \alpha_1 + \frac{\gamma}{q} \{ (q-p)\mathds{Z}+ p\mathds{Z} \} \\
& = \alpha_1 + \frac{\gamma}{q} \text{gcd}(q-p,p)\mathds{Z}\\
& =  \alpha + \frac{\gamma}{q} \text{gcd}(q-p,p)\mathds{Z}
\end{align*} 
where $\text{gcd}(a,b)$ is the greatest common divisor of $a$ and $b$. We have used at the last line that  $\alpha_1\in \mathcal A$ writes as $\alpha +\frac{\gamma}{q} \text{gcd}(q-p,p) k$, for some $k\in \mathds Z$. Combining both results, we get that: 
\begin{equation*}
E \subset \left\{ \alpha + \frac{\gamma}{q} \text{gcd}(q-p,p)\mathds{Z} \right\} \cap  \ \ \ ]\gamma(\alpha-1),\gamma\alpha[.
\end{equation*} 
This shows that the state space is finite and a subset of $\mathcal{A}$. The reciprocal implication is proved in the following Lemma, together with irreducibility.
\end{proof}

\begin{proof}[Proof of \Cref{lem:iid_irr}]
Our objective is to show that there is a path of positive probability going from any point of the state space $\mathcal{A}$ to any point of the same state space $\mathcal{A}$. Note that the chain always has at most two options when on a state $x$: make a step $\gamma \alpha$, with probability $1- P(x)$, or a step $\gamma (\alpha-1)$, with probability $P(x)$.

Let $(x,y) \in \mathcal{A}^2$. Thereby, there exist $(k,n),(l,m) \in \mathds{N}^2$ such that:
\begin{align*}
x & = \alpha + k \gamma \alpha + n \gamma (\alpha-1) \\
y & = \alpha + l \gamma \alpha + m \gamma (\alpha-1). 
\end{align*}
Thus, starting from $x$, to attain $y$, the chain has to make the path $y-x = (l-k) \gamma \alpha + (m-n) \gamma (\alpha-1)$. 

Noting that for any $h \in \mathds{N}$ we have $\gamma \alpha (q-p) h + \gamma (\alpha - 1) h p = 0$, we can equivalently write that:
\begin{equation}
y - x = u \gamma \alpha + v \gamma (\alpha-1), \label{eq:irred}
\end{equation}
with $(u,v) \in \mathds{N}^2 \setminus \{(0,0)\}$.

Thus, for any $(x,y) \in \mathcal{A}^2$ there exists $(u,v) \in \mathds{N}^2 \setminus \{(0,0)\}$ such that $y - x = u \gamma \alpha + v \gamma (\alpha-1)$.

Let's show by induction on $u+v$ that for any $(u,v)\in \mathds N^2$, and  $(x,y) \in \mathcal{A}^2$ satisfying \Cref{eq:irred} there exists a path of strictly positive probability between $x$ and $y$. 

\textbf{Initialization.} Suppose first that $u+v = 1$. Then, there are two options: $u = 1$ and $v = 0$ or the reverse. Assume the former: \Cref{eq:irred} gives $y = x + \gamma \alpha$ and necessarily $x < 1$ since $y < 1 + \gamma \alpha$ because $y \in \mathcal{A}$. Thereby the step $\gamma \alpha$ has a probability $1 - P(x) > 0$ to occur. Thus the chain can attain $y$ starting from $x$, i.e., $\mathds P(\alpha_2 = y |\alpha_1= x) >0$. The second case works similarly, by observing that necessarily $x > 0$. 

\textbf{Heredity.} Let $m\in \mathds N$. We assume that   for any $(u,v)\in \mathds N^2$ such that  $u+v=m$, and  $(x,y) \in \mathcal{A}^2$ satisfying \Cref{eq:irred} there exists a path of strictly positive probability between $x$ and $y$, or formally there exists $t\in \mathds N$ such that $\mathds P(\alpha_t = y |\alpha_1= x) >0$. 

Suppose now that $u+v = m+1$ with $m \in \mathds{N}^*$.  If $v = 0$, then $y = x + u\gamma\alpha$ and similarly than for $v = 0$ and $u = 1$, the step $\gamma \alpha$ is probable. Let $z = x + \gamma \alpha$. We have:
\begin{itemize}[noitemsep,topsep=0pt,wide]
    \item $\mathds P(\alpha_{2} = z | \alpha_1 =x ) = 1- P(x)> 0$.
    \item By our induction hypothesis, $(y,z)$ satisfy Eq.~\ref{eq:irred} with $u+v =m$, thus there exists $t$ such that $\mathds P(\alpha_{t} = y | \alpha_2 =y ) >0$.
\end{itemize}
Overall, $\mathds P(\alpha_{t} = y | \alpha_1=x ) >0$.

If instead $u = 0$, then $y = x + v\gamma(\alpha-1)$ and as for $u = 0$ and $v = 1$, the step $\gamma \alpha$ is of strictly positive probability and we conclude similarly. 

Finally, if both $u$ and $v$ are non-null, then we can make the step $\gamma (\alpha-1)$ if $x > 0$ and the step $\gamma \alpha$ otherwise, before using our induction hypothesis.

This shows that we can build a path of strictly positive probability for any $(x,y) \in \mathcal{A}^2$, and thereby that the chain is irreducible.
\end{proof}

\subsection{Control of the first four moments: \Cref{lem:iid_mom1,lem:iid_mom2,lem:iid_mom3,lem:iid_mom4}}\label{app:proofmoments}
In the following Lemmas, to compute the first order moments of $\pi_\gamma$, we consider the chain 
$\alpha_{t+1} = \alpha_t + \gamma\left(\alpha - B_{P(\alpha_t)}\right)$ for $t > 0$, launched from the stationary distribution $ \alpha_1  \sim \pi_\gamma$. Thanks to the stationarity property, for all $t\geq 1$, $\alpha_t\sim \pi_\gamma$.
\begin{proof}[Proof of \Cref{lem:iid_mom1}]
Let $\gamma > 0$. To derive $\mathds{E}_{\pi_\gamma}\left[(P(\alpha_1)-\alpha)\right]$ we start by \cref{eq:acp_iid} with $t = 1$:
\begin{align*}
\alpha_2 & = \alpha_1+\gamma\left(\alpha-B_{P(\alpha_1)}\right)\\
\text{taking expectation} \qquad\qquad\qquad\quad\qquad \mathds{E}\left[\alpha_2\right] & = \mathds{E}\left[\alpha_1\right]+\gamma\left(\alpha-\mathds{E}\left[B_{P(\alpha_1)}\right]\right) \\
\text{using} \quad \mathds{E}\left[\alpha_1\right] =\mathds{E}\left[\alpha_2\right]  = \mathds{E}_{\pi_\gamma} [\alpha], \qquad \qquad
0 & = \gamma\left(\alpha-\mathds{E}_{\pi_\gamma}\left[B_{P(\alpha_1)}\right]\right) \\
\mathds{E}_{\pi_\gamma}\left[\mathds{E}\left[B_{P(\alpha_1)}|\alpha_1\right]\right] & = \alpha \\
\mathds{E}_{\pi_\gamma}\left[P(\alpha_1)\right] & = \alpha.
\end{align*}
\end{proof}

\begin{proof}[Proof of \Cref{lem:iid_mom2}]
Let $\gamma > 0$. To derive $\mathds{E}_{\pi_\gamma}\left[(P(\alpha_1)-\alpha)^2\right]$ we start by \cref{eq:acp_iid} with $t = 1$:
\begin{eqnarray*}
(\alpha_2 - \alpha)^2 &= & (\alpha_1 - \alpha)^2 + \gamma^2(\alpha-B_{P(\alpha_1)})^2 + 2\gamma(\alpha-B_{P(\alpha_1)})(\alpha_1 - \alpha) \\
\mathds{E}_{\pi_\gamma}\left[(\alpha_2 - \alpha)^2\right] &= & \mathds{E}_{\pi_\gamma}\left[(\alpha_1 - \alpha)^2\right]+ \gamma^2\mathds{E}_{\pi_\gamma}\left[(\alpha-B_{P(\alpha_1)})^2\right] + 2\gamma\mathds{E}_{\pi_\gamma}\left[(\alpha-B_{P(\alpha_1)})(\alpha_1 - \alpha))\right] \\
0 &= & \gamma^2\mathds{E}_{\pi_\gamma}\left[(\alpha-B_{P(\alpha_1)})^2\right] + 2\gamma\mathds{E}_{\pi_\gamma}\left[(\alpha-P(\alpha_1))(\alpha_1 - \alpha)\right] 
\end{eqnarray*}
Consequently,
\begin{align*}
2\gamma\mathds{E}_{\pi_\gamma}\left[(P(\alpha_1)-\alpha)(\alpha_1 - P(\alpha_1) + P(\alpha_1) - \alpha)\right]  = & \gamma^2\mathds{E}_{\pi_\gamma}\left[(\alpha-B_{P(\alpha_1)}+P(\alpha_1)-P(\alpha_1))^2\right] \\
2\gamma\mathds{E}_{\pi_\gamma}\left[(P(\alpha_1)-\alpha)^2\right] - 2\gamma\mathds{E}_{\pi_\gamma}\left[(\alpha-P(\alpha_1))(\alpha_1 - P(\alpha_1))\right] = & \gamma^2\mathds{E}_{\pi_\gamma}\left[(\alpha-B_{P(\alpha_1)}+P(\alpha_1)-P(\alpha_1))^2\right]  \\
(2-\gamma)\mathds{E}_{\pi_\gamma}\left[(P(\alpha_1)-\alpha)^2\right] = & \gamma\mathds{E}_{\pi_\gamma}\left[P(\alpha_1)(1-P(\alpha_1))\right]  \\
& + 2\mathds{E}_{\pi_\gamma}\left[(\alpha-P(\alpha_1))(\alpha_1 - P(\alpha_1))\right].  \numberthis \label{eq:trefle}
\end{align*}

We can compute $\mathds{E}_{\pi_\gamma}\left[P(\alpha_1)(1-P(\alpha_1))\right]$:
\begin{align*}
	\mathds{E}_{\pi_\gamma}\left[P(\alpha_1)(1-P(\alpha_1)) - \alpha(1-\alpha)\right] & = \mathds{E}_{\pi_\gamma}\left[(P(\alpha_1)-\alpha)(1-P(\alpha_1)) + \alpha(1-P(\alpha_1))- \alpha(1-\alpha)\right]  \\
	& = \mathds{E}_{\pi_\gamma}\left[(P(\alpha_1)-\alpha)(1-P(\alpha_1)) + \alpha(\alpha-P(\alpha_1))\right]  \\
	& = \mathds{E}_{\pi_\gamma}\left[(P(\alpha_1)-\alpha)(1-P(\alpha_1)-\alpha)\right]  \\
	& = \mathds{E}_{\pi_\gamma}\left[(P(\alpha_1)-\alpha)(\alpha-P(\alpha_1)+1-2\alpha)\right]  \\
	& = -\mathds{E}_{\pi_\gamma}\left[(P(\alpha_1)-\alpha)^2\right]+\mathds{E}_{\pi_\gamma}\left[(P(\alpha_1)-\alpha)(1-2\alpha)\right]  \\
	& = -\mathds{E}_{\pi_\gamma}\left[(P(\alpha_1)-\alpha)^2\right] \\
	\Rightarrow \mathds{E}_{\pi_\gamma}\left[P(\alpha_1)(1-P(\alpha_1))\right]& = \alpha(1-\alpha) - \mathds{E}_{\pi_\gamma}\left[(P(\alpha_1)-\alpha)^2\right] \numberthis \label{eq:var}
\end{align*}

Reinjecting \cref{eq:var} in \cref{eq:trefle}:
\begin{equation}
\mathds{E}_{\pi_\gamma}\left[(P(\alpha_1)-\alpha)^2\right] = \frac{\gamma}{2}\alpha(1-\alpha)  + \mathds{E}_{\pi_\gamma}\left[(\alpha-P(\alpha_1))(\alpha_1 - P(\alpha_1))\right]  \numberthis \label{eq:trefle_finale}
\end{equation}

We are now going to derive an upper and lower bound of $\mathds{E}_{\pi_\gamma}\left[(P(\alpha_1)-\alpha)^2\right]$.
Note that $\text{sign}(\alpha - P(\alpha_1)) = -\text{sign}(\alpha_1 - P(\alpha_1))$, thus $\mathds{E}_{\pi_\gamma}\left[(\alpha-P(\alpha_1))(\alpha_1 - P(\alpha_1))\right] \leq 0$. 
Hence we obtain the following upper bound:
\begin{equation}
\mathds{E}_{\pi_\gamma}\left[(P(\alpha_1)-\alpha)^2\right] \leq \frac{\gamma}{2}\alpha(1-\alpha).
\label{eq:mom2_maj}
\end{equation}
Furthermore, using again this observation, and additionally that $|\alpha - P(\alpha_1)| \leq 1$ and $|\alpha_1 - P(\alpha_1)| \leq \gamma$  and from \cref{eq:trefle_finale}, we can obtain:
\begin{align*}
\mathds{E}_{\pi_\gamma}\left[(P(\alpha_1)-\alpha)^2\right] \geq & \frac{\gamma}{2}\alpha(1-\alpha) -\gamma\mathds{P}_{\pi_\gamma}(\alpha_1 \notin ]0,1[) \\
\geq & \frac{\gamma}{2}\alpha(1-\alpha) -\gamma C^{-1}_\alpha \mathds{E}_{\pi_\gamma}\left[(P(\alpha_1)-\alpha)^2\right] \\
\mathds{E}_{\pi_\gamma}\left[(P(\alpha_1)-\alpha)^2\right] \geq & \frac{1}{1+\gamma C_\alpha^{-1}} \frac{\gamma}{2}\alpha(1-\alpha), \numberthis \label{eq:mom2_min}
\end{align*}
where the second inequality holds by observing that:
\begin{align*}
	\mathds{E}_{\pi_\gamma}\left[(P(\alpha_1) - \alpha)^2\right] & \geq (1-\alpha)^2\mathds{P}_{\pi_\gamma}(\alpha_1 > 1) +\alpha^2\mathds{P}_{\pi_\gamma}(\alpha_1 < 0) \\
	\mathds{E}_{\pi_\gamma}\left[(P(\alpha_1) - \alpha)^2\right] & \geq C_\alpha \mathds{P}_{\pi_\gamma}(\alpha_1 \not\in \left[0,1\right]) \\
	\Rightarrow  \mathds{P}_{\pi_\gamma}(\alpha_1 \not\in \left[0,1\right]) & \leq  C_\alpha^{-1}\mathds{E}_{\pi_\gamma}\left[(P(\alpha_1) - \alpha)^2\right] 
\end{align*} 
with $C_\alpha = \min(\alpha^2,(1-\alpha)^2)$.

Gathering \cref{eq:mom2_maj,,eq:mom2_min}, we obtain:
\begin{align*}
\frac{1}{(1+\gamma C_\alpha^{-1})} \frac{\gamma}{ 2} \alpha(1-\alpha) \leq \mathds{E}_{\pi_\gamma}\left[(P(\alpha_1)-\alpha)^2\right] & \leq \frac{\gamma}{ 2}\alpha(1-\alpha) \\
\left( \frac{1}{(1+\gamma C_\alpha^{-1})} - 1 \right) \frac{\gamma}{ 2} \alpha(1-\alpha) \leq \mathds{E}_{\pi_\gamma}\left[(P(\alpha_1)-\alpha)^2\right] - \frac{\gamma}{ 2} \alpha(1-\alpha) & \leq 0 \\
\left| \mathds{E}_{\pi_\gamma}\left[(P(\alpha_1)-\alpha)^2\right] - \frac{\gamma}{ 2} \alpha(1-\alpha) \right| & \leq \frac{ \gamma^2 C_\alpha^{-1} }{2(1+\gamma C_\alpha^{-1})} \alpha(1-\alpha) \\
\mathds{E}_{\pi_\gamma}\left[(P(\alpha_1)-\alpha)^2\right] - \frac{\gamma}{ 2} \alpha(1-\alpha) & = O(\gamma^2). \numberthis 
\label{eq:mom2}
\end{align*}
\end{proof}

\begin{proof}[Proof of \Cref{lem:iid_mom3}]

Let $\gamma > 0$. We start again by using \cref{eq:acp_iid} and removing the first terms as ${\mathds{E}_{\pi_\gamma}\left[(\alpha_2-\alpha)^3\right] = \mathds{E}_{\pi_\gamma}\left[(\alpha_1-\alpha)^3\right]}$. Then we will isolate $\mathds{E}_{\pi_\gamma}\left[(P(\alpha_1)-\alpha)^3\right]$ and finally we will dominate each term obtained.

\begin{align*}
0 =  &  3\gamma\mathds{E}_{\pi_\gamma}\left[(\alpha_1 - \alpha)^2(\alpha-B_{P(\alpha_1)})\right] + 3\gamma^2\mathds{E}_{\pi_\gamma}\left[(\alpha_1 - \alpha)(\alpha-B_{P(\alpha_1)})^2\right] + \gamma^3\mathds{E}_{\pi_\gamma}\left[(\alpha-B_{P(\alpha_1)})^3\right] \\
0 =  &  3\gamma\mathds{E}_{\pi_\gamma}\left[(\alpha_1-\alpha)^2(\alpha-P(\alpha_1))\right] + 3\gamma^2\mathds{E}_{\pi_\gamma}\left[(\alpha_1-\alpha)(\alpha-P(\alpha_1))^2)\right] \\
& + 6\gamma^2\mathds{E}_{\pi_\gamma}\left[(\alpha_1-\alpha)(\alpha-P(\alpha_1))(P(\alpha_1)-B_{P(\alpha_1)})\right] +
3\gamma^2\mathds{E}_{\pi_\gamma}\left[(\alpha_1-\alpha)(P(\alpha_1)-B_{P(\alpha_1)})^2\right] \\
& + \gamma^3\mathds{E}_{\pi_\gamma}\left[(\alpha-B_{P(\alpha_1)})^3\right] \\
3\gamma\mathds{E}_{\pi_\gamma}\left[(P(\alpha_1)-\alpha)^3\right] =  &  3\gamma\mathds{E}_{\pi_\gamma}\left[(\alpha_1-P(\alpha_1))^2(\alpha-P(\alpha_1))\right] + 6\gamma\mathds{E}_{\pi_\gamma}\left[(\alpha_1-P(\alpha_1))(P(\alpha_1)-\alpha)(\alpha-P(\alpha_1))\right] \\
& + 3\gamma^2\mathds{E}_{\pi_\gamma}\left[(\alpha_1-\alpha)(\alpha-P(\alpha_1))^2)\right] + 3\gamma^2\mathds{E}_{\pi_\gamma}\left[(\alpha_1-\alpha)P(\alpha_1)(1-P(\alpha_1))\right] \\
& + \gamma^3\mathds{E}_{\pi_\gamma}\left[(\alpha-B_{P(\alpha_1)})^3\right] \\
3\mathds{E}_{\pi_\gamma}\left[(P(\alpha_1)-\alpha)^3\right] =  &  3\mathds{E}_{\pi_\gamma}\left[(\alpha_1-P(\alpha_1))^2(\alpha-P(\alpha_1))\right] - 6\mathds{E}_{\pi_\gamma}\left[(\alpha_1-P(\alpha_1))(P(\alpha_1)-\alpha)^2\right] \\
& + 3\gamma\mathds{E}_{\pi_\gamma}\left[(\alpha_1-\alpha)(\alpha-P(\alpha_1))^2)\right] + 3\gamma\mathds{E}_{\pi_\gamma}\left[(\alpha_1-\alpha)P(\alpha_1)(1-P(\alpha_1))\right] \\
& + \gamma^2\mathds{E}_{\pi_\gamma}\left[(\alpha-B_{P(\alpha_1)})^3\right] \\
3 \left| \mathds{E}_{\pi_\gamma}\left[(P(\alpha_1)-\alpha)^3\right] \right| \leq  &  3 \left| \mathds{E}_{\pi_\gamma}\left[(\alpha_1-P(\alpha_1))^2(\alpha-P(\alpha_1))\right] \right| + 6 \left| \mathds{E}_{\pi_\gamma}\left[(\alpha_1-P(\alpha_1))(P(\alpha_1)-\alpha)^2\right] \right|\\
& + 3 \gamma \left|\mathds{E}_{\pi_\gamma}\left[(\alpha_1-\alpha)(\alpha-P(\alpha_1))^2)\right] \right| + 3\gamma\left|\mathds{E}_{\pi_\gamma}\left[(\alpha_1-\alpha)P(\alpha_1)(1-P(\alpha_1))\right] \right| \\
& + \gamma^2\left|\mathds{E}_{\pi_\gamma}\left[(\alpha-B_{P(\alpha_1)})^3\right] \right| \numberthis \label{eq:mom3_interm}.
\end{align*}

To conclude, we can bound each term of the right hand side of \cref{eq:mom3_interm}. In order of appearance we obtain:
\begin{align*}
\left| \mathds{E}_{\pi_\gamma}\left[(\alpha_1-P(\alpha_1))^2(\alpha-P(\alpha_1))\right]\right| & \leq  \mathds{E}_{\pi_\gamma}\left[(\alpha_1-P(\alpha_1))^2\left|\alpha-P(\alpha_1)\right|\right] \\
\left| \mathds{E}_{\pi_\gamma}\left[(\alpha_1-P(\alpha_1))^2(\alpha-P(\alpha_1))\right]\right| & \leq \gamma^2. \numberthis \label{eq:mom3_term1}
\end{align*}
\begin{align*}
\left| \mathds{E}_{\pi_\gamma}\left[(\alpha_1-P(\alpha_1))(P(\alpha_1)-\alpha)^2\right] \right| & \leq \mathds{E}_{\pi_\gamma}\left[\left|\alpha_1-P(\alpha_1)\right|(P(\alpha_1)-\alpha)^2\right]  \\
& \leq \gamma \mathds{E}_{\pi_\gamma}\left[(P(\alpha_1)-\alpha)^2\right]  \\
\left| \mathds{E}_{\pi_\gamma}\left[(\alpha_1-P(\alpha_1))(P(\alpha_1)-\alpha)^2\right] \right| & \leq \frac{\gamma^2}{2}\alpha(1-\alpha) + O(\gamma^3) , \numberthis \label{eq:mom3_term2}
\end{align*}
where the last equality is obtained by using \Cref{lem:iid_mom2}.
\begin{align*}
\gamma \left|\mathds{E}_{\pi_\gamma}\left[(\alpha_1-\alpha)(\alpha-P(\alpha_1))^2)\right] \right| & \leq \gamma \mathds{E}_{\pi_\gamma}\left[\left|\alpha_1-\alpha\right|(\alpha-P(\alpha_1))^2)\right] \\
& \leq \gamma D_{\gamma,\alpha} \mathds{E}_{\pi_\gamma}\left[(\alpha-P(\alpha_1))^2)\right] \\
\gamma \left|\mathds{E}_{\pi_\gamma}\left[(\alpha_1-\alpha)(\alpha-P(\alpha_1))^2)\right] \right| & \leq D_{\gamma,\alpha} \frac{\gamma^2}{2}\alpha(1-\alpha) + O(\gamma^3), \numberthis \label{eq:mom3_term3}
\end{align*}
again using \Cref{lem:iid_mom2}, and with $D_{\gamma,\alpha} = \max(1+\gamma\alpha, \gamma(1-\alpha)) - \alpha = O(1)$.

\begin{align*}
    \gamma \left|\mathds{E}_{\pi_\gamma}\left[(\alpha_1-\alpha)P(\alpha_1)(1-P(\alpha_1))\right] \right| & \leq \gamma \left|\mathds{E}_{\pi_\gamma}\left[(\alpha_1-P(\alpha_1))P(\alpha_1)(1-P(\alpha_1))\right] \right| + \gamma \left|\mathds{E}_{\pi_\gamma}\left[(P(\alpha_1)-\alpha)P(\alpha_1)(1-P(\alpha_1))\right] \right| \\
    & \leq \gamma \frac{1}{4} \mathds{E}_{\pi_\gamma}\left[\left|\alpha_1-P(\alpha_1)\right|\right] + \gamma  \frac{1}{4} \mathds{E}_{\pi_\gamma}\left[\left| P(\alpha_1)-\alpha \right|\right] \\
    & \leq \frac{\gamma^2}{4} + \frac{\gamma}{4} \sqrt{\mathds{E}_{\pi_\gamma}\left[(P(\alpha_1)-\alpha)^2\right]}\\
    & \leq \frac{\gamma^2}{4} + \frac{\gamma}{4} \sqrt{\frac{\gamma}{ 2} \alpha(1-\alpha) + O(\gamma^2)}\\
    \gamma \left|\mathds{E}_{\pi_\gamma}\left[(\alpha_1-\alpha)P(\alpha_1)(1-P(\alpha_1))\right] \right| & \leq O(\gamma^{3/2}), \numberthis \label{eq:mom3_term4}
\end{align*}
where the last inequality comes from \Cref{lem:iid_mom2} a third time.

\begin{equation}
    \gamma^2\left|\mathds{E}_{\pi_\gamma}\left[(\alpha-B_{P(\alpha_1)})^3\right] \right| \leq \gamma^2\max(\alpha^3,(1-\alpha)^3).
    \label{eq:mom3_term5}
\end{equation}

Gathering \crefrange{eq:mom3_term1}{eq:mom3_term5} together with \cref{eq:mom3_interm}, we obtain the following upper bound:
\begin{equation*}
3 \left| \mathds{E}_{\pi_\gamma}\left[(P(\alpha_1)-\alpha)^3\right] \right| \leq 3 \gamma^2 + 3\gamma^2\alpha(1-\alpha) + O(\gamma^3) + 3 D_{\gamma,\alpha} \frac{\gamma^2}{2}\alpha(1-\alpha) + O(\gamma^3) + O(\gamma^{3/2}) + \gamma^2 \max(\alpha^3,(1-\alpha)^3),
\end{equation*}
which leads to:
\begin{equation}
    \mathds{E}_{\pi_\gamma}\left[(P(\alpha_1)-\alpha)^3\right] = O(\gamma^{3/2}). \label{eq:mom3}
\end{equation}
\end{proof}

\begin{proof}[Proof of \Cref{lem:iid_mom4}]
Let $\gamma > 0$. For the fourth order moment, the proof works in the same way for the third order moment, \Cref{lem:iid_mom3}.

\begin{align*}
0 =  &  4\gamma\mathds{E}_{\pi_\gamma}\left[(\alpha_1 - \alpha)^3(\alpha-B_{P(\alpha_1)})\right] + 6\gamma^2\mathds{E}_{\pi_\gamma}\left[(\alpha_1 - \alpha)^2(\alpha-B_{P(\alpha_1)})^2\right] \\
& + 4\gamma^3\mathds{E}_{\pi_\gamma}\left[(\alpha_1 - \alpha)(\alpha-B_{P(\alpha_1)})^3\right] + \gamma^4\mathds{E}_{\pi_\gamma}\left[(\alpha-B_{P(\alpha_1)})^4\right] \\
0 =  &  4\gamma\mathds{E}_{\pi_\gamma}\left[(\alpha_1 - P(\alpha_1) + P(\alpha_1) - \alpha)^3(\alpha-P(\alpha_1))\right] \\
& + 6\gamma^2\mathds{E}_{\pi_\gamma}\left[(\alpha_1 - \alpha)^2(\alpha-P(\alpha_1)+P(\alpha_1)-B_{P(\alpha_1)})^2\right] \\
& + 4\gamma^3\mathds{E}_{\pi_\gamma}\left[(\alpha_1 - \alpha)(\alpha-B_{P(\alpha_1)})^3\right] + \gamma^4\mathds{E}_{\pi_\gamma}\left[(\alpha-B_{P(\alpha_1)})^4\right] \\
4 \gamma \mathds{E}_{\pi_\gamma}\left[(P(\alpha_1)-\alpha)^4\right] =  &  4\gamma\mathds{E}_{\pi_\gamma}\left[(\alpha_1 - P(\alpha_1))^3(\alpha-P(\alpha_1))\right] + 12\gamma\mathds{E}_{\pi_\gamma}\left[(\alpha_1 - P(\alpha_1))^2(P(\alpha_1)-\alpha)(\alpha-P(\alpha_1))\right] \\
& + 12\gamma\mathds{E}_{\pi_\gamma}\left[(\alpha_1 - P(\alpha_1))(P(\alpha_1)-\alpha)^2(\alpha-P(\alpha_1))\right] \\
& + 6\gamma^2\mathds{E}_{\pi_\gamma}\left[(\alpha_1 - \alpha)^2(\alpha-P(\alpha_1))^2\right] + 0 + 6\gamma^2\mathds{E}_{\pi_\gamma}\left[(\alpha_1 - \alpha)^2(P(\alpha_1)-B_{P(\alpha_1)})^2\right] \\
& + 4\gamma^3\mathds{E}_{\pi_\gamma}\left[(\alpha_1 - \alpha)(\alpha-B_{P(\alpha_1)})^3\right] + \gamma^4\mathds{E}_{\pi_\gamma}\left[(\alpha-B_{P(\alpha_1)})^4\right] \\
4 \mathds{E}_{\pi_\gamma}\left[(P(\alpha_1)-\alpha)^4\right] =  &  4\mathds{E}_{\pi_\gamma}\left[(\alpha_1 - P(\alpha_1))^3(\alpha-P(\alpha_1))\right] - 12\mathds{E}_{\pi_\gamma}\left[(\alpha_1 - P(\alpha_1))^2(P(\alpha_1)-\alpha)^2\right] \\
& - 12\mathds{E}_{\pi_\gamma}\left[(\alpha_1 - P(\alpha_1))(P(\alpha_1)-\alpha)^3\right] \\
& + 6\gamma\mathds{E}_{\pi_\gamma}\left[(\alpha_1 - \alpha)^2(\alpha-P(\alpha_1))^2\right] + 6\gamma\mathds{E}_{\pi_\gamma}\left[(\alpha_1 - \alpha)^2P(\alpha_1)(1-P(\alpha_1))\right] \\
& + 4\gamma^2\mathds{E}_{\pi_\gamma}\left[(\alpha_1 - \alpha)(\alpha-B_{P(\alpha_1)})^3\right] + \gamma^3\mathds{E}_{\pi_\gamma}\left[(\alpha-B_{P(\alpha_1)})^4\right] \\
4 \left| \mathds{E}_{\pi_\gamma}\left[(P(\alpha_1)-\alpha)^4\right] \right| \leq  &  4 \left| \mathds{E}_{\pi_\gamma}\left[(\alpha_1 - P(\alpha_1))^3(\alpha-P(\alpha_1))\right] \right| + 12 \left| \mathds{E}_{\pi_\gamma}\left[(\alpha_1 - P(\alpha_1))^2(P(\alpha_1)-\alpha)^2\right] \right| \\
& + 12 \left| \mathds{E}_{\pi_\gamma}\left[(\alpha_1 - P(\alpha_1))(P(\alpha_1)-\alpha)^3\right] \right|\\
& + 6\gamma \left| \mathds{E}_{\pi_\gamma}\left[(\alpha_1 - \alpha)^2(\alpha-P(\alpha_1))^2\right] \right| + 6\gamma \left| \mathds{E}_{\pi_\gamma}\left[(\alpha_1 - \alpha)^2P(\alpha_1)(1-P(\alpha_1))\right] \right| \\
& + 4\gamma^2 \left| \mathds{E}_{\pi_\gamma}\left[(\alpha_1 - \alpha)(\alpha-B_{P(\alpha_1)})^3\right] \right| + \gamma^3 \left| \mathds{E}_{\pi_\gamma}\left[(\alpha-B_{P(\alpha_1)})^4\right] \right|.  \numberthis \label{eq:mom4_interm}
\end{align*}
We are now going to dominate each term of the right hand side of \cref{eq:mom4_interm} in order of appearance.

\begin{align*}
    \left| \mathds{E}_{\pi_\gamma}\left[(\alpha_1 - P(\alpha_1))^3(\alpha-P(\alpha_1))\right] \right| & \leq \mathds{E}_{\pi_\gamma}\left[\left|\alpha_1 - P(\alpha_1)\right|^3\left|\alpha-P(\alpha_1)\right|\right] \\
     \left| \mathds{E}_{\pi_\gamma}\left[(\alpha_1 - P(\alpha_1))^3(\alpha-P(\alpha_1))\right] \right| & \leq \gamma^3 \numberthis \label{eq:mom4_term1}
\end{align*}
\begin{align*}
    \left| \mathds{E}_{\pi_\gamma}\left[(\alpha_1 - P(\alpha_1))^2(P(\alpha_1)-\alpha)^2\right] \right| & = \mathds{E}_{\pi_\gamma}\left[(\alpha_1 - P(\alpha_1))^2(P(\alpha_1)-\alpha)^2\right] \\
    \left| \mathds{E}_{\pi_\gamma}\left[(\alpha_1 - P(\alpha_1))^2(P(\alpha_1)-\alpha)^2\right] \right| & \leq \gamma^2. \numberthis \label{eq:mom4_term2}
\end{align*}
\begin{align*}
    \left| \mathds{E}_{\pi_\gamma}\left[(\alpha_1 - P(\alpha_1))(P(\alpha_1)-\alpha)^3\right] \right| & \leq \mathds{E}_{\pi_\gamma}\left[\left| \alpha_1 - P(\alpha_1)\right|\left| P(\alpha_1)-\alpha\right|^3\right] \\
    & \leq \gamma \mathds{E}_{\pi_\gamma}\left[\left| P(\alpha_1)-\alpha\right|^3\right] \\
    \left| \mathds{E}_{\pi_\gamma}\left[(\alpha_1 - P(\alpha_1))(P(\alpha_1)-\alpha)^3\right] \right| & \leq O(\gamma^{5/2}). \numberthis \label{eq:mom4_term3}
\end{align*}
where the last inequality holds using \Cref{lem:iid_mom3}.
\begin{align*}
    \gamma \left| \mathds{E}_{\pi_\gamma}\left[(\alpha_1 - \alpha)^2(\alpha-P(\alpha_1))^2\right] \right| & = \gamma \mathds{E}_{\pi_\gamma}\left[(\alpha_1 - \alpha)^2(\alpha-P(\alpha_1))^2\right] \\ 
    & \leq \gamma D_{\gamma,\alpha}^2 (\frac{\gamma}{2} \alpha (1-\alpha) + O(\gamma^2)) \\
    \gamma \left| \mathds{E}_{\pi_\gamma}\left[(\alpha_1 - \alpha)^2(\alpha-P(\alpha_1))^2\right] \right| & \leq D_{\gamma,\alpha}^2 \frac{\gamma^2}{2} \alpha (1-\alpha) + O(\gamma^3). \numberthis \label{eq:mom4_term4}
\end{align*}
again where we've used \Cref{lem:iid_mom3}, and re-used its notation $D_{\gamma,\alpha} = \max(1+\gamma\alpha, \gamma(1-\alpha)) - \alpha = O(1)$.

\begin{align*}
    \gamma \left| \mathds{E}_{\pi_\gamma}\left[(\alpha_1 - \alpha)^2P(\alpha_1)(1-P(\alpha_1))\right] \right| 
     = & \gamma \left| \mathds{E}_{\pi_\gamma}\left[(\alpha_1 - P(\alpha_1))^2P(\alpha_1)(1-P(\alpha_1))\right] \right. \\
     & + 2\mathds{E}_{\pi_\gamma}\left[(\alpha_1 - P(\alpha_1))(P(\alpha_1) - \alpha)P(\alpha_1)(1-P(\alpha_1))\right] \\
     & + \left. \mathds{E}_{\pi_\gamma}\left[(P(\alpha_1) - \alpha)^2P(\alpha_1)(1-P(\alpha_1))\right] \right| \\
    \leq & \frac{\gamma}{4} \mathds{E}_{\pi_\gamma}\left[(\alpha_1 - P(\alpha_1))^2\right] + \frac{\gamma}{2} \mathds{E}_{\pi_\gamma}\left[\left| \alpha_1 - P(\alpha_1) \right| \left| P(\alpha_1) - \alpha \right| \right] \\
     & + \frac{\gamma}{4}\mathds{E}_{\pi_\gamma}\left[(P(\alpha_1) - \alpha)^2\right] \\
    \gamma \left| \mathds{E}_{\pi_\gamma}\left[(\alpha_1 - \alpha)^2P(\alpha_1)(1-P(\alpha_1))\right] \right|  \leq & \frac{\gamma^3}{4} + \frac{\gamma^2}{2} + \frac{\gamma^2}{8}\alpha(1-\alpha) + O(\gamma^3).\numberthis \label{eq:mom4_term5}
\end{align*}
again where we've used \Cref{lem:iid_mom3}.

\begin{align*}
    \gamma^2 \left| \mathds{E}_{\pi_\gamma}\left[(\alpha_1 - \alpha)(\alpha-B_{P(\alpha_1)})^3\right] \right| & \leq \gamma^2 \mathds{E}_{\pi_\gamma}\left[\left| \alpha_1 - \alpha \right| \left| \alpha-B_{P(\alpha_1)} \right| ^3\right] \\
    & \leq \gamma^2 D_{\gamma,\alpha }\mathds{E}_{\pi_\gamma}\left[\left| \alpha-B_{P(\alpha_1)} \right| ^3\right] \\
    \gamma^2 \left| \mathds{E}_{\pi_\gamma}\left[(\alpha_1 - \alpha)(\alpha-B_{P(\alpha_1)})^3\right] \right| & \leq \gamma^2 D_{\gamma,\alpha } \max(\alpha^3,(1-\alpha)^3). \numberthis \label{eq:mom4_term6}
\end{align*}
\begin{equation}
    \gamma^3\left|\mathds{E}_{\pi_\gamma}\left[(\alpha-B_{P(\alpha_1)})^4\right] \right| \leq \gamma^3\max(\alpha^4,(1-\alpha)^4).
    \label{eq:mom4_term7}
\end{equation}

Gathering \crefrange{eq:mom4_term1}{eq:mom4_term7} together with \cref{eq:mom4_interm}, we obtain finally:
\begin{equation}
    \mathds{E}_{\pi_\gamma}\left[(P(\alpha_1)-\alpha)^4\right] = O(\gamma^{3/2}).
    \label{eq:mom4}
\end{equation}
\end{proof}

\subsection{Proof of \Cref{thm:ar}}
In this section, we prove \Cref{thm:ar}. Recall the theorem:

\begin{customTheorem}[\ref{thm:ar}]
Assume that: (i) $\alpha \in \mathds{Q}$; (ii) the residuals follow an AR(1) process (i.e., $\varepsilon_{t+1} = \varphi\varepsilon_{t}+\xi_{t+1}$ with $(\xi_t)_t$ i.i.d. random variables admitting a continuous density with respect to Lebesgue measure, of support $\mathcal{S}$) clipped at a large value $R$, and $[-R,R] \subset \mathcal{S}$;
(iii) the quantile function $Q$ of the stationary distribution of $(\varepsilon_t)_t$ is known; (iv) $Q$ is bounded by $R$. 
Then $(\alpha_t,\varepsilon_{t-1})$ is a homogeneous Markov Chain in $\mathds R^2$ that admits a unique stationary distribution $\pi_{\gamma,\varphi}$. Moreover, $$\frac{1}{T}\sum\limits_{t=1}^T L(\alpha_t) \overset{a.s.}{\underset{T \rightarrow +\infty}{\longrightarrow}} \mathds{E}_{\pi_{\gamma,\varphi}}[L].$$ 
\end{customTheorem}

We consider $Z_t= (\alpha_t, \epsilon_{t-1}) $ defined in the state-space  $\mathcal{Z}= \mathcal A \times [-R,R]$ by 
\begin{align*}
\begin{cases}
\alpha_{t+1}&=\alpha_{t}+\gamma\left(\alpha-\mathds{1}\{ | \epsilon_{t} | > \widehat{Q}_{1-P(\alpha_t)} \}\right), \\
\epsilon_{t} &= -R \vee (\varphi\epsilon_{t-1} +\xi_t) \wedge R
\end{cases}
\end{align*}
That is, $ (\alpha_t )_{t\geq 0} $ is the recurrence defined by \Cref{eq:update_scheme}, and $ (\epsilon_{t})_{t\geq 0} $ is an AR(1) process with parameters $ \varphi $ clipped at some large value $ R $. Finally,  $ (\xi_t)_{t} $ is a sequence of i.i.d. r.v. admitting a continuous density with respect to the Lebesgue measure, of support $\mathcal{S} \supset [-R,R]$.

This chain is defined for parameters $ \alpha, R $ considered as fixed, and we focus on the influence of $ \gamma, \varphi $.
The main difference w.r.t. the previous section is that the state space is not countable anymore. More precisely, the state space is a product of a finite discrete set and an interval of $ \mathds R $.

The state-space $ \mathcal Z $ is $ \mathcal A \times [-R,R] $
, where $ \mathcal A $ is defined in the previous \Cref{subsec:proofiid}.
We equip $ \mathcal Z $ with the $ \sigma $-algebra $ \mathcal F = \mathcal P (\mathcal A) \times \mathcal B(\mathds R)$, where $ \mathcal P (\mathcal A)  $ is the power-set of the finite set $ \mathcal A $ and $ B(\mathds R ) $ is the borel set of $ \mathds R $.

\begin{lemma}
	The sequence $ (Z_t)_{t\geq 0} $ is a Markov chain. Moreover, the chain is Harris-recurrent, and admits a stationary distribution $ \pi_{\gamma, \varphi} $.
\end{lemma}

\begin{proof}
	We observe that 
\begin{equation}\label{eq:defF}
	Z_{t+1} = \left(\begin{array}{c}
\alpha_{t+1} \\
\epsilon_{t}
	\end{array}\right) = \left(\begin{array}{c}
	\alpha_{t}+\gamma \left(\alpha-\mathds{1}\{ | \varphi \epsilon_{t-1} + \xi_t | > \widehat{Q}_{1-P(\alpha_t)} \}\right) \\
	 -R \vee (\varphi\epsilon_{t-1} +\xi_t) \wedge R
	\end{array}\right) = : F_\gamma(Z_t, \xi_t).
\end{equation}
	
For a function $ F_\gamma: \mathds R^2 \times \mathds R $.
Consequently, $ Z_t $ follows a \textit{Non-Linear State Space} model \citep[Section 2.2.2 and Chapter 7]{meyn2012markov}.
We denote $ P_\gamma $ the probability kernel or Markov transition function, that is, for any $ z = (a,e) \in \Z $, and $ F \in \F $:
\begin{align*}
	P_\gamma(z, F) &=\mathds P (Z_1 \in F | Z_0 = z).
\end{align*}	
Remark that relying on \Cref{eq:defF}, we have an explicit formula for $ P_\gamma $. Defining the sequence of  functions $ (F_t)_{t\geq 1} $ such that 
$$
 F_{t+1}\left(z_{0}, \xi_{1}, \ldots \xi_{t+1}\right)=F_\gamma \left(F_{t}\left(z_{0}, \xi_{1}, \ldots \xi_{t}\right), \xi_{t+1}\right) 
$$
where $z_{0}$ and $(\xi_{i})$ are arbitrary real numbers. By induction we have that for any initial condition $Z_{0}=z_{0} \in \mathcal Z$ and any $t \in \mathds{N}$,
$$
Z_{t}=F_{t}\left(z_{0}, \xi_{1}, \ldots, \xi_{t}\right),
$$
which immediately implies that the $t$-step transition function may be expressed as
$$
\begin{aligned}
P_\gamma^{t}(z, F) &=\mathds{P}\left(F_{k}\left(z, \xi_{1}, \ldots, \xi_{t}\right) \in F\right) =\int \cdots \int \mathds{1}\left\{F_{k}\left(z, \xi_{1}, \ldots, \xi_{t}\right) \in F \right\} p \left(d \xi_{1}\right) \ldots p\left(d \xi_{t}\right)
\end{aligned}
$$
where $ p $ is the distribution of $ \xi $.

We first prove that the chain is $ \psi $-irreductible, for $ \psi = \mu \otimes \lambda_{\text{Leb}} $, with $ \mu  $ the uniform probability measure on $ \mathcal A $ and $\lambda_{\text{Leb}}$ the Lebesgue measure.
For any $ z_0= (a_0, e_0)\in \Z $ and $ F= \{a'\} \times O $, with $ O $ open set, we have that
\begin{align*}
	\mathds P(Z_t \in F | Z_0 = z_0) >0
\end{align*}
for some $ t $ large enough. Indeed, 
\begin{enumerate}[noitemsep,topsep=0pt]
	\item There exists a path $ (a_1, \dots, a_t) $ from $ a_0 $ to $ a' $ such that for all $ i\in \{1, \dots , t-1\} $, $ 0<a_i<1 $.
	\item  Let $ F_i $ be the set of values of $ \epsilon_{i} $ such that we obtain $ a_{i} $ from $ a_{i-1} $.
	\item Then if $ 0<a'<1 $, we can directly conclude, as we have that  for all $ i\in \{1, \dots , t\} $,  $\mathds P(Z_i \in \{a_{i}\}\times  F_i) |  Z_{i-1} =({a_{i-1}}, z_{i-1}))) > \delta > 0$,  thus $\mathds P(Z_t \in F | Z_0 = z_0) >\delta^n>0$.
	\item The argument extends to the case where $ a' \notin (0,1) $: one only has to account for the fact that the last step can for $ a $ be in  made in both directions (increasing or decreasing), depending on the set of values in $ O  $.
\end{enumerate}

Moreover, the argument can be extended to show that for any $ a', O $, there exists $ \delta $ such that for all $a_0,e_0$, there exists $ t \le \frac{1}{\alpha \gamma} $ such that 
\begin{align*}
	\mathds P(Z_t \in F | Z_0 = z_0)  >\delta
\end{align*}
Which proves that the chain will visit infinitely many times any borel set $ F $ with probability 1, and is consequently Harris-recurrent~\citep[Chapter 9]{meyn2012markov}. 
Using Theorem 10.0.1 in \citet{meyn2012markov}, we conclude that the chain admits a unique stationary distribution $ \pi_{\gamma, \varphi} $.

Finally, applying \citep[Theorem 17.1.7][]{meyn2012markov} to the later result gives: $$\frac{1}{T}\sum\limits_{t=1}^T L(\alpha_t) \overset{a.s.}{\underset{T \rightarrow +\infty}{\longrightarrow}} \mathds{E}_{\pi_{\gamma,\varphi}}[L].$$ 
\end{proof}

\subsection{Numerical study of ACI efficiency with AR(1) residuals, with respect to the median length}
\label{app:ar_aci_numerical_median}

We here reproduce the same experiment as in \Cref{sec:theory_ar}, but focus on the \textit{efficiency} as the median of the intervals' lengths instead of the average (after imputation). Results are given in \Cref{fig:aci_ar_numerical_median}.

\begin{figure}[!h]
    \centering
    \begin{minipage}[b]{0.49\textwidth}
        \centering
        \includegraphics[scale=0.4]{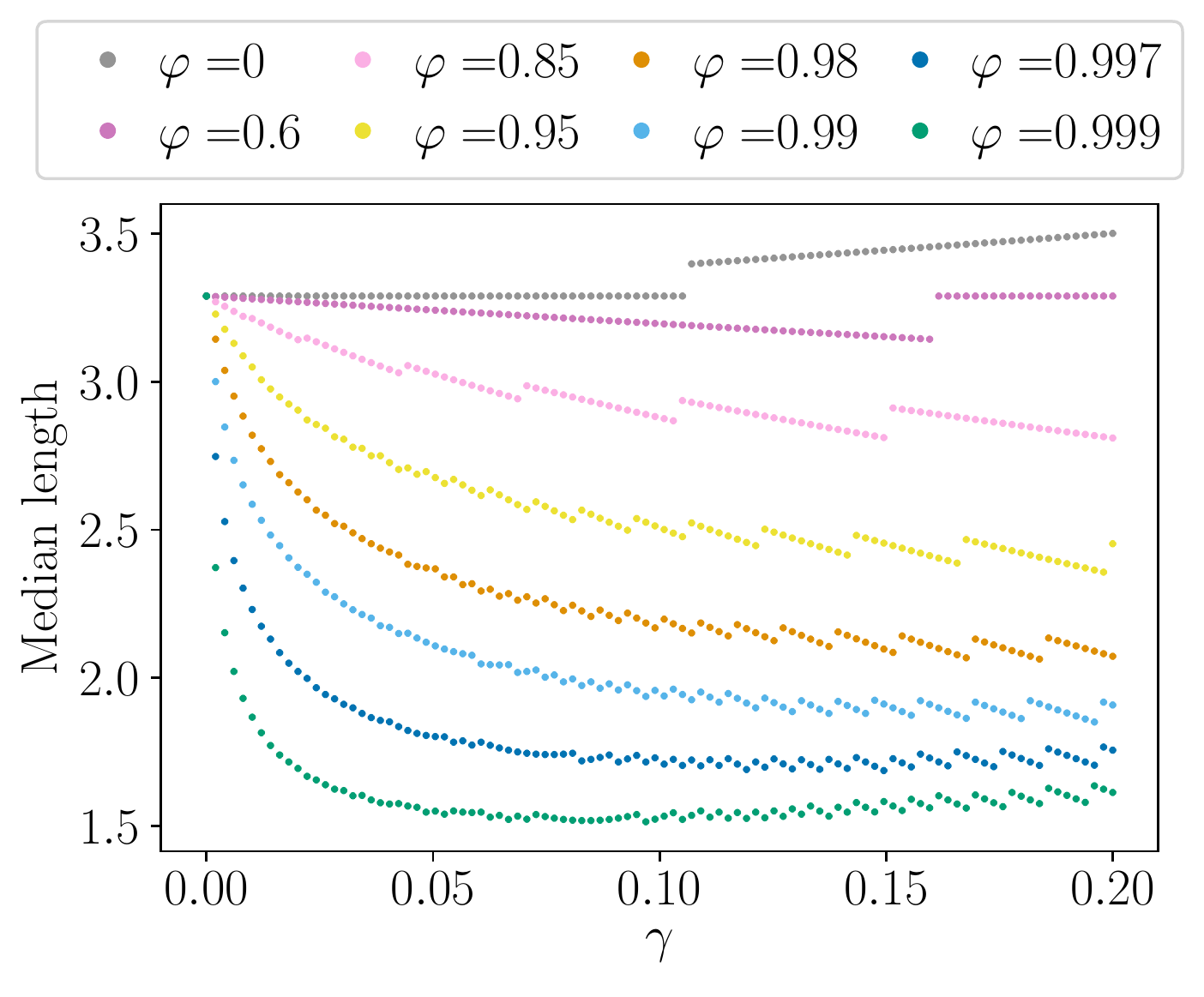}
    \end{minipage}
    \hfill
    \begin{minipage}[b]{0.49\textwidth}
        \centering
        \includegraphics[scale=0.4]{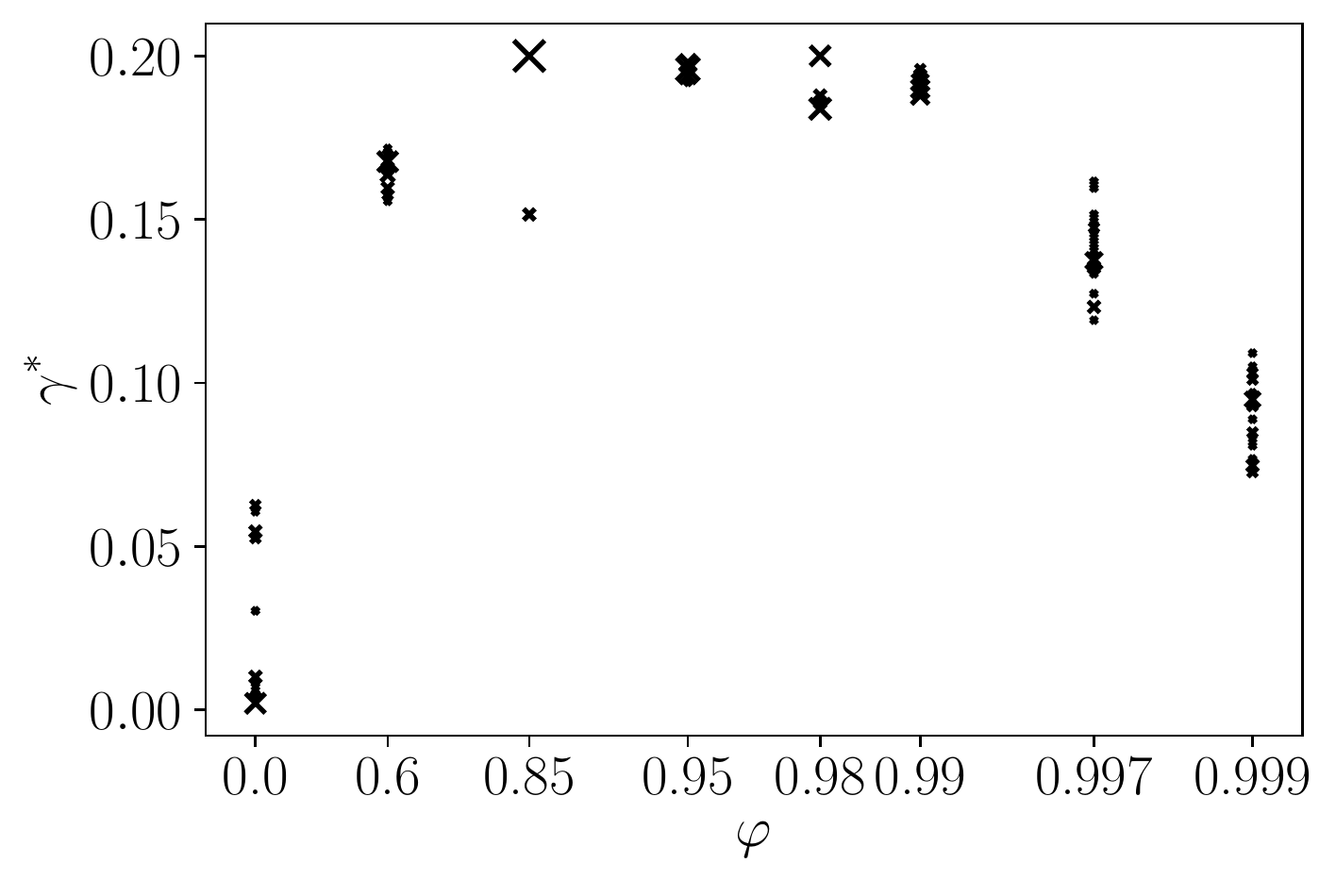}
    \end{minipage}
    \caption{Left: evolution of the median length depending on $\gamma$ for various~$\varphi$. Right: $\gamma^*$ minimizing the average length for each $\varphi$.}
\label{fig:aci_ar_numerical_median}
\end{figure}

Observations are very similar to the average length case, especially regarding (i) the monotonicity of the median interval length w.r.t. $\varphi$,  (ii) the existence of a minimum $\gamma^{+}_\varphi$ to the function $\gamma \mapsto \rm{Med}_{\pi_\gamma}[\tilde\alpha]:= \arg\!\min_m \mathds E_{\pi_{\gamma,\varphi}}[|\tilde \alpha-m|]$ (iii) the non-monotony of $\phi \mapsto \gamma^{+}_\varphi$. 

\section{Experimental details.}
\label{app:experimental_details}
\subsection{Details ARMA(1,1) processes}
\label{app:arma}

\begin{definition}[ARMA(1,1) process]
We say that $\varepsilon_t$ is an ARMA(1,1) process if for any $t$:
$$ \varepsilon_{t+1} =  \varphi \varepsilon_{t} + \xi_{t+1} + \theta \xi_t , $$
with: 
\begin{itemize}
    \item $\theta + \varphi \neq 0$, $|\varphi| < 1$ and $|\theta| < 1$;
    \item $\xi_t$ is a white noise of variance $\sigma^2$, called the \textit{innovation}.
\end{itemize}
\label{def:arma}
\end{definition}

The asymptotic variance of this process is:
\begin{equation}
    \text{Var}(\varepsilon_t) = \sigma^2\frac{1 - 2 \varphi \theta + \theta^2}{1 - \varphi^2}.
\label{eq:var_arma}
\end{equation}

An ARMA(1,1) is thus characterised by three parameters: the coefficients $\varphi$ and $\theta$ and the innovation's variance $\sigma^2$. 
The larger the coefficients, in absolute value, the greater the time dependence and variance.
Note that when $\varphi = 0$, the ARMA(0,1) process corresponds to a MA(1) and when $\theta = 0$,  the ARMA(1,0) process corresponds to an AR(1). 

To fix the asymptotic variance of an ARMA(1,1) of parameters $\varphi$ and $\theta$ to $v$, we fix $\sigma^2 = v\frac{1-\varphi^2}{1-2\varphi\theta+\theta^2}$.

\subsection{Random forest parameters}
\label{app:rf}

All the random forests model have the same parameters, that are the following:
\begin{itemize}
    \item Number of trees: 1000
    \item Minimum sample per leaf: 1 (default)
    \item Maximum number of features: $d$ (default)
\end{itemize}

Furthermore, for EnbPI, as there is already an individual bootstrap in the algorithm, the random forest regressors do not bootstrap them again. 

\subsection{Details about the baselines and comparison}
\label{app:comp_methods}

\subsubsection{EnbPI full algorithm}
\label{app:enbpi}

In order to be self-contained and precise the modifications done in EnbPI V2, the EnbPI algorithm from \citet{pmlr-v139-xu21h} is recalled in the following. In \textcolor{blindpurple}{purple} we precise the difference in EnbPI V2.

\begin{algorithm}
\caption{Sequential Distribution-free Ensemble Batch Prediction Intervals (EnbPI)}
\label{alg:EnbPI}
\begin{algorithmic}[1] 
\REQUIRE Training data $\left\{\left(x_{i}, y_{i}\right)\right\}_{i=1}^{T}$, regression algorithm $\mathcal{A}$, decision threshold $\alpha$, aggregation function $\phi$, number of bootstrap models $B$, the batch size $s$, and test data $\left\{\left(x_{t}, y_{t}\right)\right\}_{t=T+1}^{T+T_{1}}$, with $y_{t}$ revealed only after the batch of $s$ prediction intervals with $t$ in the batch are constructed.
\ENSURE Ensemble prediction intervals $\{C_{\alpha}(x_t) \}_{t=T+1}^{T+T_1}$
\FOR {$b = 1, \dots, B$}
\STATE Sample with replacement an index set $S_{b}=$ $\left(i_{1}, \ldots, i_{T}\right)$ from indices $(1, \ldots, T)$
\STATE Compute $\hat{f}^{b}=\mathcal{A}\left(\left\{\left(x_{i}, y_{i}\right) \mid i \in S_{b}\right\}\right)$
\ENDFOR
\STATE Initialise $\varepsilon = \{\}$
\FOR {$i = 1, \dots, T$}
\STATE $\hat{f}_{-i}^{\phi}\left(x_{i}\right)=\phi\left(\left\{\hat{f}^{b}\left(x_{i}\right) \mid i \notin S_{b}\right\}\right)$
\STATE Compute $\hat{\varepsilon}_{i}^{\phi}=\left|y_{i}-\hat{f}_{-i}^{\phi}\left(x_{i}\right)\right|$
\STATE $\varepsilon=\varepsilon \cup\left\{\hat{\varepsilon}_{i}^{\phi}\right\}$
\ENDFOR
\FOR {$t = T+1, \dots, T+T_1$}
\STATE Let $\hat{f}_{-t}^{\phi}\left(x_{t}\right)=(1-\alpha)$ quantile of $\left\{\hat{f}_{-i}^{\phi}\left(x_{t}\right)\right\}_{i=1}^{T}$ \textcolor{blindpurple}{\textbf{EnbPI V2:} this is replaced by $\hat{f}_{-t}^{\phi}\left(x_{t}\right) = \phi\left(\left\{\hat{f}_{-i}^{\phi}\left(x_{t}\right)\right\}_{i=1}^{T}\right)$.}
\STATE Let $w_{t}^{\phi}=(1-\alpha)$ quantile of $\varepsilon$ 
\STATE Return $C_{T, t}^{\phi, \alpha}\left(x_{t}\right)=\left[\hat{f}_{-t}^{\phi}\left(x_{t}\right) \pm w_{t}^{\phi}\right]$
\IF {$t-T=0 \mod s$}
\FOR {$j=t-1,\dots,t-1$}
\STATE Compute $\hat{\varepsilon}_{j}^{\phi}=\left|y_{j}-\hat{f}_{-j}^{\phi}\left(x_{t}\right)\right|$
\STATE $\boldsymbol{\varepsilon}=\left(\boldsymbol{\varepsilon}-\left\{\hat{\varepsilon}_{1}^{\phi}\right\}\right) \cup\left\{\hat{\varepsilon}_{i}^{\phi}\right\}$ and reset index of $\boldsymbol{\varepsilon}$
\ENDFOR
\ENDIF
\ENDFOR
\end{algorithmic}
\end{algorithm}

\paragraph{Remark on the bootstrap approach.} The bootstrap scheme is not adapted to time series, even if such strategies have been developed \citep{hardle2003bootstrap, kreiss2012hybrid, cai2012simple}, and could be used to improve the adequation of EnbPI with the time series framework. Furthermore, recent works have proposed modifications of RF in the dependent setting \citep{goehry2020random, goehry2021random, rf_spatial_depend}. Generalizing these improvements to any ensemble method and use it for EnbPI could also enhance its performance, but is out of the scope of this paper.

\subsubsection{Details on the implementation}
\label{app:baselines_comp}

We conclude this section by summarizing computational aspects of the methods.
One of the contributions is to provide a unified experimental framework. Therefore, in \Cref{tab:code_comp}, we display the current available code for these methods, and what is available in the proposed repository.

\begin{table}[!h]
\caption{Summary of available code online for each method and the proposed code in the repository. The programming language is specified, and, when relevant, the nature of the code.}
\label{tab:code_comp}
\vskip 0.15in
\begin{center}
\begin{small}
\begin{tabular}{lllll}
\toprule
& \multicolumn{2}{c}{Currently available} & \multicolumn{2}{c}{Contribution}\\ \cmidrule(l){2-3}\cmidrule(l){4-5}
Methods & Language & Details & Language & Options \\ \midrule
CP & R &  & Python &  \\
OSCP & not available &  & Python & randomised split \\
EnbPI & Python &  & Python & same aggregation function \\
ACI & R script & no general function & Python & randomised split \\
\bottomrule
\end{tabular}
\end{small}
\end{center}
\vskip -0.1in
\end{table}

\section{Additional experiments on synthetic data sets}
\label{app:syn_expe}

In this section, we provide supplemental results on the synthetic data sets presented in \Cref{sec:data_syn}. 

First, in \Cref{app:aci_emp_gamma} the sensitivity analysis of ACI $\gamma$ as well as the comparison to the naive strategy and AgACI is extended to AR(1) and MA(1) processes of asymptotic variance 10.

Then, in \Cref{app:syn_expe_baselines}, the comparison of all the CP methods for time series (initiated in \Cref{sec:comp_syn}) is also extended to these noises, that is AR(1) and MA(1) processes of asymptotic variance 10 (\Cref{app:syn_expe_baselines_10}), and to ARMA(1,1), AR(1) and MA(1) processes of asymptotic variance 1 (\Cref{app:syn_expe_baselines_1}). 

Next, we discuss in \Cref{sec:comp_syn} that the improved \textit{validity} for $\gamma = 0.05$ in comparison to $\gamma = 0.01$ comes at the cost of more infinite intervals. This analysis is detailed in \Cref{app:acp_inf}.

Finally, we compare randomized and sequential split in \Cref{app:randomized_sequential}.

\paragraph{Imputation.} The rationale to impute the infinite intervals is the following. We take the maximum of the absolute values of the residuals on the test set, noted $|\varepsilon|_{\max}$. Then, for any $t \in \llbracket T_0+1, T_0+T_1 \rrbracket$, if the predicted upper (resp. lower) bound $\hat{b}^{(u)}_t(x_t)$ is such that $\hat{b}_t(x_t) > \hat{\mu_t}(x_t) + |\varepsilon|_{\max}$ (resp. $\hat{b}^{(\ell)}_t(x_t) < \hat{\mu_t}(x_t) - |\varepsilon|_{\max}$) we impute it by $\hat{\mu_t}(x_t) + |\varepsilon|_{\max}$ (resp. $\hat{\mu_t}(x_t) - |\varepsilon|_{\max}$).
\subsection{Additional experimental results of ACI sensitivity to $\gamma$, presented in \Cref{sec:acp_gamma_emp}}
\label{app:aci_emp_gamma}

In this subsection, we provide similar results to those of \Cref{sec:acp_gamma_emp}, for different models on the noise. Especially, we consider AR(1) and MA(1) processes. 

\textbf{Observations.} The behaviour of the AR(1) process is very similar to the one of ARMA(1,1). On the other hand, for the MA case, the dependence structure is too weak to observe a significant effect of $\gamma$. All ACI methods produce  nearly valid intervals, with coverage above $89.25\%$.

Results are given in \Cref{fig:aci_ar_var10_app,fig:aci_ma_var10_app}. 

\begin{figure}[!h]
\gs\gs
\centering
\includegraphics[scale=0.22]{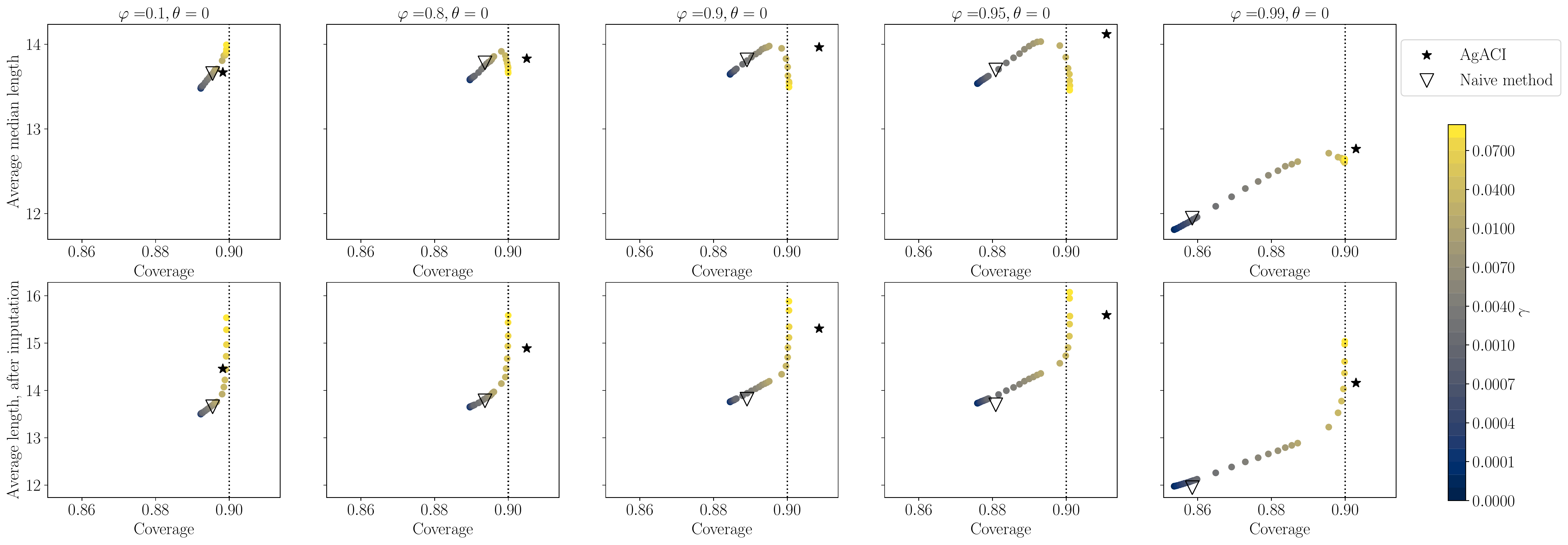}
\gs\gs
\caption{ACI performance with various $\theta$, $\varphi$ and $\gamma$ on data simulated according to \cref{eq:friedman} with a Gaussian AR(1) noise of asymptotic variance 10 (see \Cref{app:arma}). Top row: average median length with respect to the coverage. Bottom row: percentage of infinite intervals. Stars correspond to the proposed online expert aggregation strategy, and empty triangles to the naive choice.}
\label{fig:aci_ar_var10_app}
\end{figure}

\begin{figure}[!h]
\gs\gs
\centering
\includegraphics[scale=0.22]{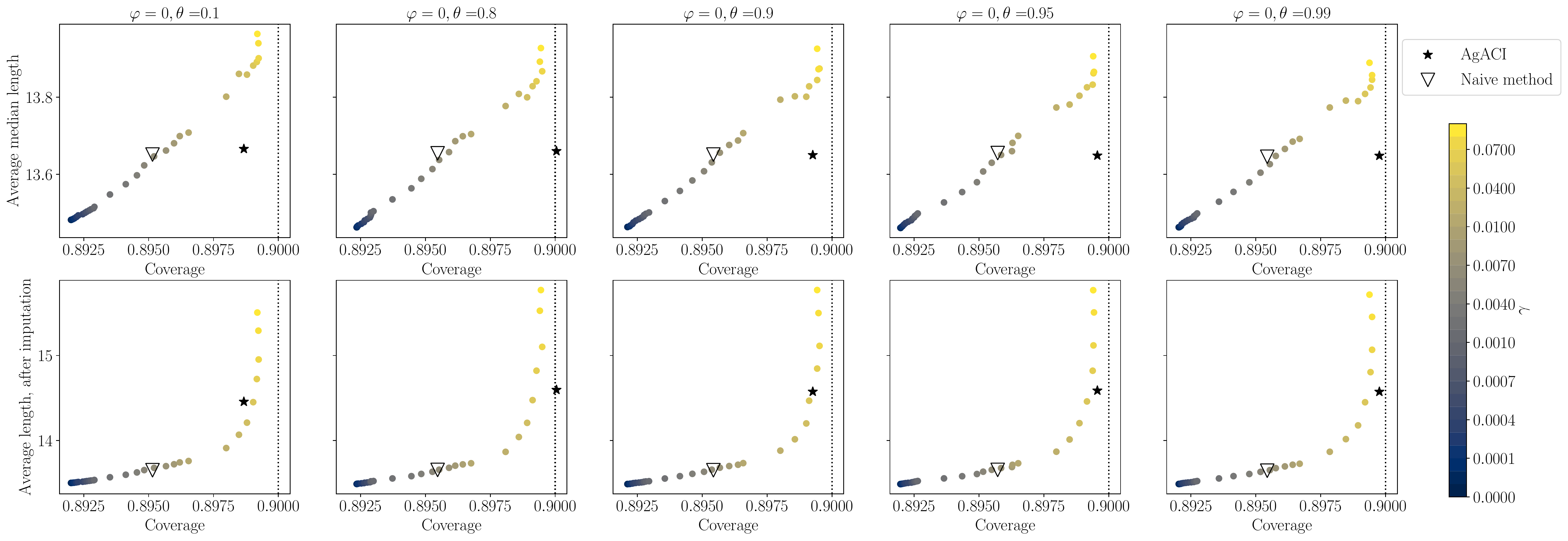}
\gs\gs
\caption{ACI performance with various $\theta$, $\varphi$ and $\gamma$ on data simulated according to \cref{eq:friedman} with a Gaussian MA(1) noise of asymptotic variance 10 (see \Cref{app:arma}). Top row: average median length with respect to the coverage. Bottom row: percentage of infinite intervals. Stars correspond to the proposed online expert aggregation strategy, and empty triangles to the naive choice.}
\label{fig:aci_ma_var10_app}
\end{figure}

\subsection{Comparison to baselines, extension of \Cref{sec:comp_syn}}
\label{app:syn_expe_baselines}

\subsubsection{Asymptotic variance fixed to 10.}
\label{app:syn_expe_baselines_10}

\Cref{fig:baseline_ar_ma_10} displays the results on data generated according to \Cref{sec:data_syn}, for an asymptotic variance of the noise of 10 (as in \Cref{fig:arma_fixed_10}), when this noise is an AR(1) or MA(1) process.

\begin{figure}[!h]
\centerline{\includegraphics[scale=0.28]{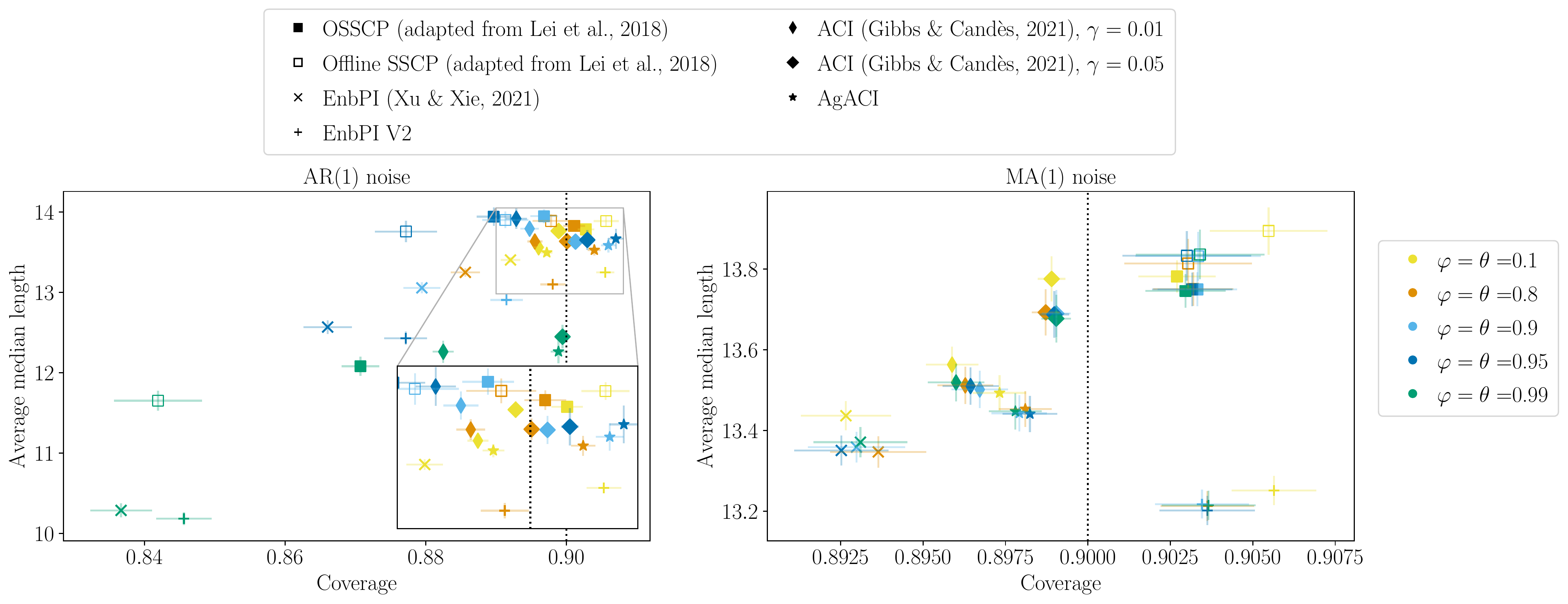}}
\caption{Performance of various interval prediction methods on data simulated according to \cref{eq:friedman} with a Gaussian AR(1) (left) and MA(1) (right) noise of asymptotic variance 10 (see \Cref{app:arma}). Results aggregated from 500 independent runs. Empirical standard error are displayed.}
\label{fig:baseline_ar_ma_10}
\end{figure}

\textbf{Observations.} As in the previous section, the methods' performances are greatly impacted by the type and strength of dependence structure. \Cref{fig:baseline_ar_ma_10} shows that while ARMA(1,1) and AR(1) noises lead to similar patterns, it is not the case for an MA(1) noise. In the latter, $\theta$ has little influence: the five performances (one for each $\theta$) are similar within each method. In addition, offline sequential SCP is very close to OSSCP. This is expected as a MA(1) process has very short memory, and the temporal dependence is thus small even for $\theta=0.99$. 

\subsubsection{Asymptotic variance fixed to 1.}
\label{app:syn_expe_baselines_1}

We now fix the asymptotic variance of the noise to 1. The results are plotted in \Cref{fig:baseline_arma_ar_ma_var1}. Note that this is an easier setting than previously, as the signal to noise ratio is higher for this asymptotic variance.

\textbf{Observations.} Similarly to \Cref{fig:baseline_ar_ma_10}, $\theta$ has little influence when the noise is a MA(1). On AR(1) and ARMA(1,1) noises (left and middle subplots), the patterns are similar. First, we observe again the improvement thanks to the online mode (empty squares versus solid ones), which increases when the dependence increases. Second, all the methods achieve \textit{validity} or are significantly closer to achieving it than when the asymptotic variance is set to 10 (this is related to the high signal to noise ratio mentioned at the beginning of this section). Third, EnbPI V2 is \textit{valid} for $\varphi = \theta \leq 0.95$ and provides the most \textit{efficient} intervals for theses values. Nevertheless, its performances, as well as those of EnbPI, follow a clear trend (similar to that of \Cref{fig:arma_fixed_10}): when the dependence increases, the coverage decreases, as well as the length. EnbPI does not seem to be robust to the increasing temporal dependence in these experiments. 

\begin{figure}[!h]
\centerline{\includegraphics[scale=0.28]{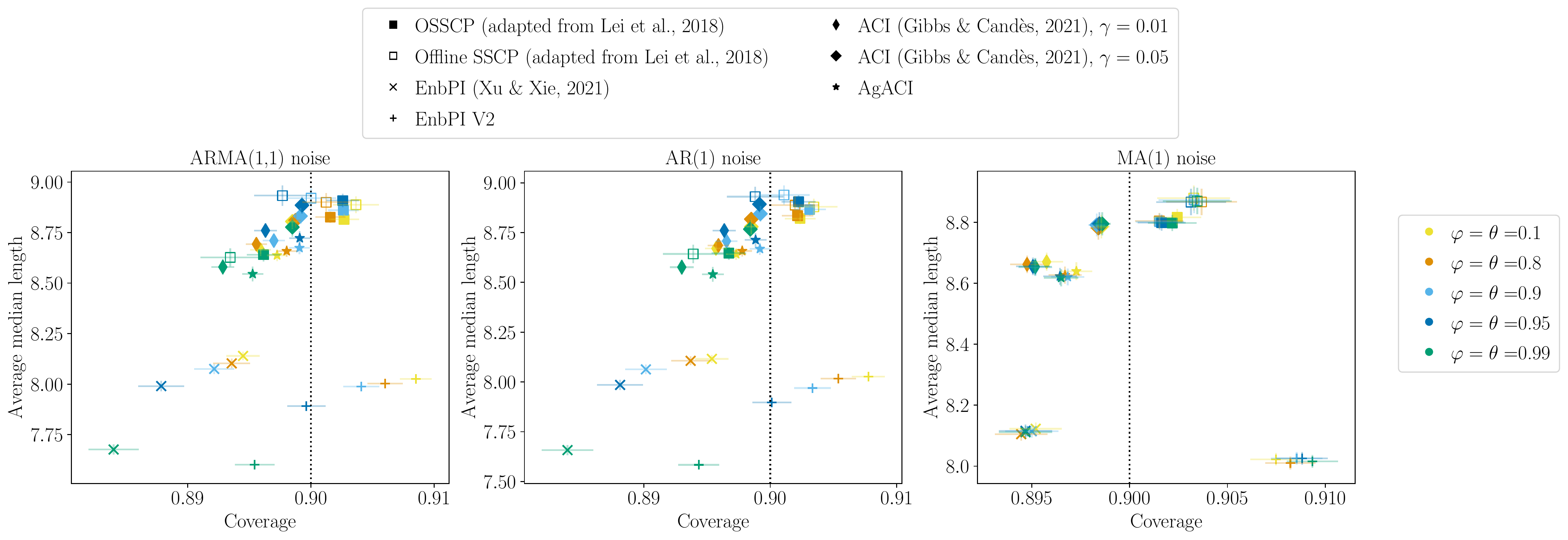}}
\caption{Performance of interval prediction methods on data simulated according to \cref{eq:friedman} with an ARMA(1,1) (left), AR(1) (center) and MA(1) (right) noise with a $\mathcal{N}(0,1\frac{1-\varphi^2}{1-2\varphi\theta+\theta^2})$ \textit{innovation}. Results aggregated from 500 independent runs. Empirical standard error are displayed.}
\label{fig:baseline_arma_ar_ma_var1}
\end{figure}

\subsection{Closer look at infinite intervals}
\label{app:acp_inf}

\begin{wraptable}[20]{r}{0.47\linewidth}
 \vspace{-1cm}
   \caption{Percentage of infinite intervals for ACI, on an ARMA(1,1) noise (first five rows), on an AR(1) noise ($\theta = 0$, next five rows) and a MA(1) noise ($\varphi = 0$, last five rows). The central two columns present the percentage of infinite intervals, for $\gamma = 0.01$ and $\gamma = 0.05$. The last column represents the proportion of points for which $\gamma = 0.05$ predicts $\mathds{R}$ and that are \textit{not} covered for $\gamma = 0.01$.}
\label{tab:acp_inf_ar_ma}
\begin{center}
\begin{small}
\vspace{-0.2cm}
\resizebox{\linewidth}{!}{\begin{tabular}{l|c|c|lr}
\toprule
Noise parameters & $\gamma = 0.01$ & $\gamma = 0.05$ & Intersection \\ 
\midrule
$\varphi = \theta = 0.1$ & 0 & 1.12 & 53\ \  out of 562 &(9.43\%)  \\
$\varphi = \theta = 0.8$ & 0 & 2.76  & 263 out of 1381& (19.04\%) \\
$\varphi = \theta = 0.9$ & 0 & 3.72 & 425 out of 1862& (22.83\%) \\
$\varphi = \theta = 0.95$ & 0.03 & 4.45 & 514 out of 2224& (23.11\%) \\
$\varphi = \theta = 0.99$ & 0.04 & 6.22  & 554 out of 3109 &(17.82\%) \\
\hline
$\varphi = 0.1$ & 0 & 1 & 37\ \  out of 500 &(7.40\%) \\
$\varphi = 0.8$ & 0 & 2.75 & 212 out of 1373 &(15.44\%) \\
$\varphi = 0.9$ & 0 & 3.24 & 359 out of 1622 &(22.13\%) \\
$\varphi = 0.95$ & 0.03 & 4.32 & 488 out of 2160& (22.59\%) \\
$\varphi = 0.99$ & 0.06 & 6.15 & 560 out of 3073 &(18.22\%) \\
\hline
$\theta = 0.1$ & 0 & 1.03 & 38 out of 516 &(7.36\%) \\
$\theta = 0.8$ & 0 & 1.42 & 49 out of 710 &(6.90\%) \\
$\theta = 0.9$ & 0 & 1.54 & 47 out of 772 &(6.09\%) \\
$\theta = 0.95$ & 0 & 1.54 & 45 out of 770 &(5.84\%) \\
$\theta = 0.99$ & 0 & 1.56 & 53 out of 781 &(6.79\%) \\
\bottomrule
\end{tabular}}
\end{small}
\end{center}
\vskip -0.1in
\end{wraptable}
In this subsection, we investigate further the infite intervals generated by ACI for ARMA(1,1), AR(1) and MA(1) noise models. We report the results in \Cref{tab:acp_inf_ar_ma}.  The central two columns present the percentage of infinite intervals, for $\gamma = 0.01$ and $\gamma = 0.05$. 
A first obvious observation is that the number of infinite intervals is orders of magnitude smaller for $\gamma=0.01$ than for $\gamma=0.05$. 
The last column represents the proportion of points for which $\gamma = 0.05$ predicts $\mathds{R}$ and that are \textit{not} covered for $\gamma = 0.01$. This suggests that for those intervals, predicting an infinite interval was somehow justified in the sense that the point was seemingly challenging to cover (as $\gamma = 0.01$ failed to cover).
For example, in the first line ($\varphi = \theta = 0.1$) we read that there are 562 points that result in infinite intervals for $\gamma = 0.05$, among which 53 lead to finite predictions for $\gamma = 0.01$ failing to cover on that point. This means only 9.43 \% of 562 infinite intervals that can be considered as ``somehow justified''. This analysis highlights that $\gamma = 0.05$ seem to predict  more infinite intervals than necessary, to compensate for easy errors as explained in \Cref{sec:acp}.

\subsection{Randomised, sequential and other splits.}
\label{app:randomized_sequential}

In \Cref{fig:arma_fixed_10_randomized_errtype}, we compare the sequential split strategy (dark markers) used in our experiments to the randomised version (clear markers), on online SCP. We observe that the intervals produced by the randomised version are significantly smaller than the sequential one, while covering slightly less. 

Another splitting strategy would consist in calibrating on the first points and training on the last ones. Up to our knowledge, this has not been used in practice. This way, we could hope to obtain a better model for the point prediction task. Nevertheless, we would be calibrating on really different data than the test ones. Thereby, the impact of this scheme regarding the interval prediction task performance is not straightforward. This is why we focus here on the sequential split, which is the most intuitive approach. Analysing further all of these effects theoretically or with extensive numerical experiments would be beneficial to the time series conformal prediction domain.

\begin{figure}[!h]
\centerline{\includegraphics[scale=0.28]{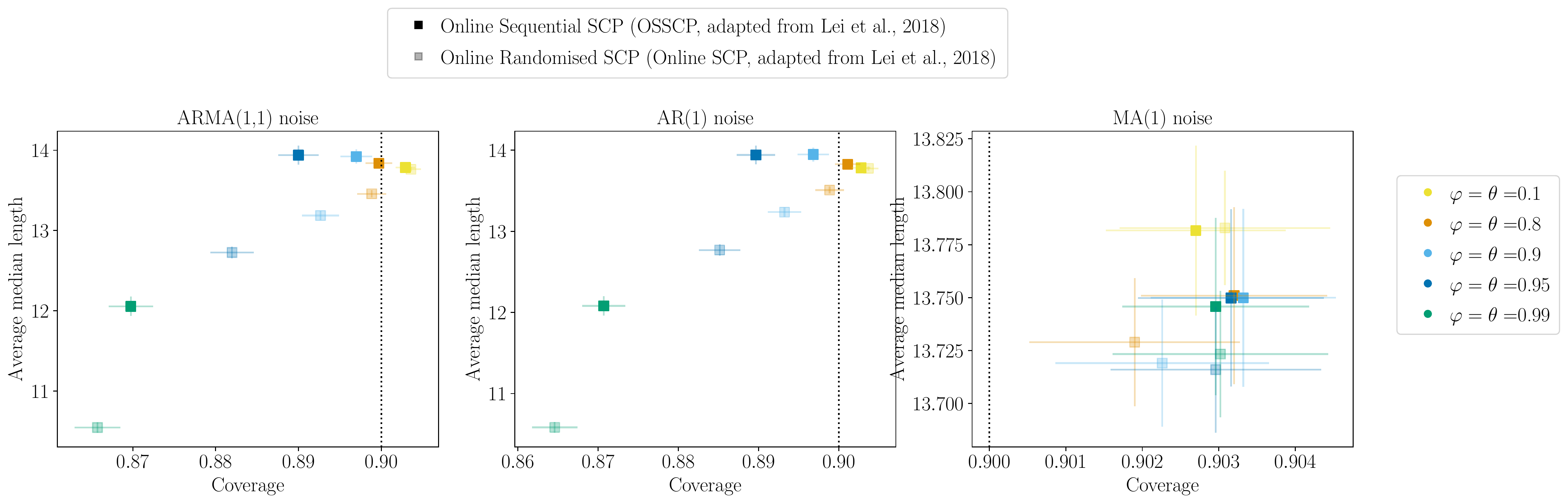}}
\caption{Performance of  interval prediction methods on data simulated according to \cref{eq:friedman} with a Gaussian ARMA(1,1) (left), AR(1) (middle) and MA(1) (right) noise of asymptotic variance 10 (see \Cref{app:arma}). Randomised methods are displayed. Results aggregated from 500 independent runs. Empirical standard error are displayed.}
\label{fig:arma_fixed_10_randomized_errtype}
\end{figure}

\section{Forecasting French electricity spot prices}
\label{app:price}
\subsection{Details about the data set}
\label{app:price_data}

\Cref{tab:spot_dataset} presents an extract of the French electricity spot prices data set used in \Cref{sec:comp_real}.
In this table, $2 \times 23$ columns are hidden for clarity and space: the 24 prices of $D-7$ and the 24 prices of $D-7$ are used as variables.

\begin{table}[!h]
\caption{Extract of the built data set, for French electricity spot price forecasting.}
\label{tab:spot_dataset}
\vskip 0.15in
\begin{center}
\begin{small}
\begin{tabular}{c|ccccc}
\toprule
Date and time & Price & Price D-1 & Price D-7 & For. cons. & DOW \\
\midrule
11/01/16 0PM & 21.95 &  15.58 & 13.78 & 58800 & Monday \\
11/01/16 1PM & 20.04 & 19.05 & 13.44 & 57600 & Monday \\
\vdots & \vdots & \vdots & \vdots & \vdots & \vdots  \\
12/01/16 0PM & 21.51 &  21.95 & 25.03 & 61600 & Tuesday \\
12/01/16 1PM & 19.81 & 20.04 & 24.42 & 59800 & Tuesday \\
\vdots & \vdots & \vdots & \vdots & \vdots & \vdots  \\
18/01/16 0PM & 38.14 & 37.86 & 21.95 & 70400 & Monday \\
18/01/16 1PM & 35.66 & 34.60 & 20.04 & 69500 & Monday \\
\vdots & \vdots & \vdots & \vdots & \vdots & \vdots  \\
\bottomrule
\end{tabular}
\end{small}
\end{center}
\vskip -0.1in
\end{table}

\subsection{Forecasting year 2019}
\label{app:price_visu_int}

In \Cref{fig:ex_int_all} we observe that on January 25, 2019, the forecasts are very different from the actual values. Nevertheless, the prediction intervals manage to include these observations for almost all hours (except after 5 pm) and almost all methods (EnbPI does not include points earlier, starting at 11 am).

\begin{figure}[!h]
     \centering
     \begin{subfigure}[OSSCP]{
         \includegraphics[width=0.36\textwidth]{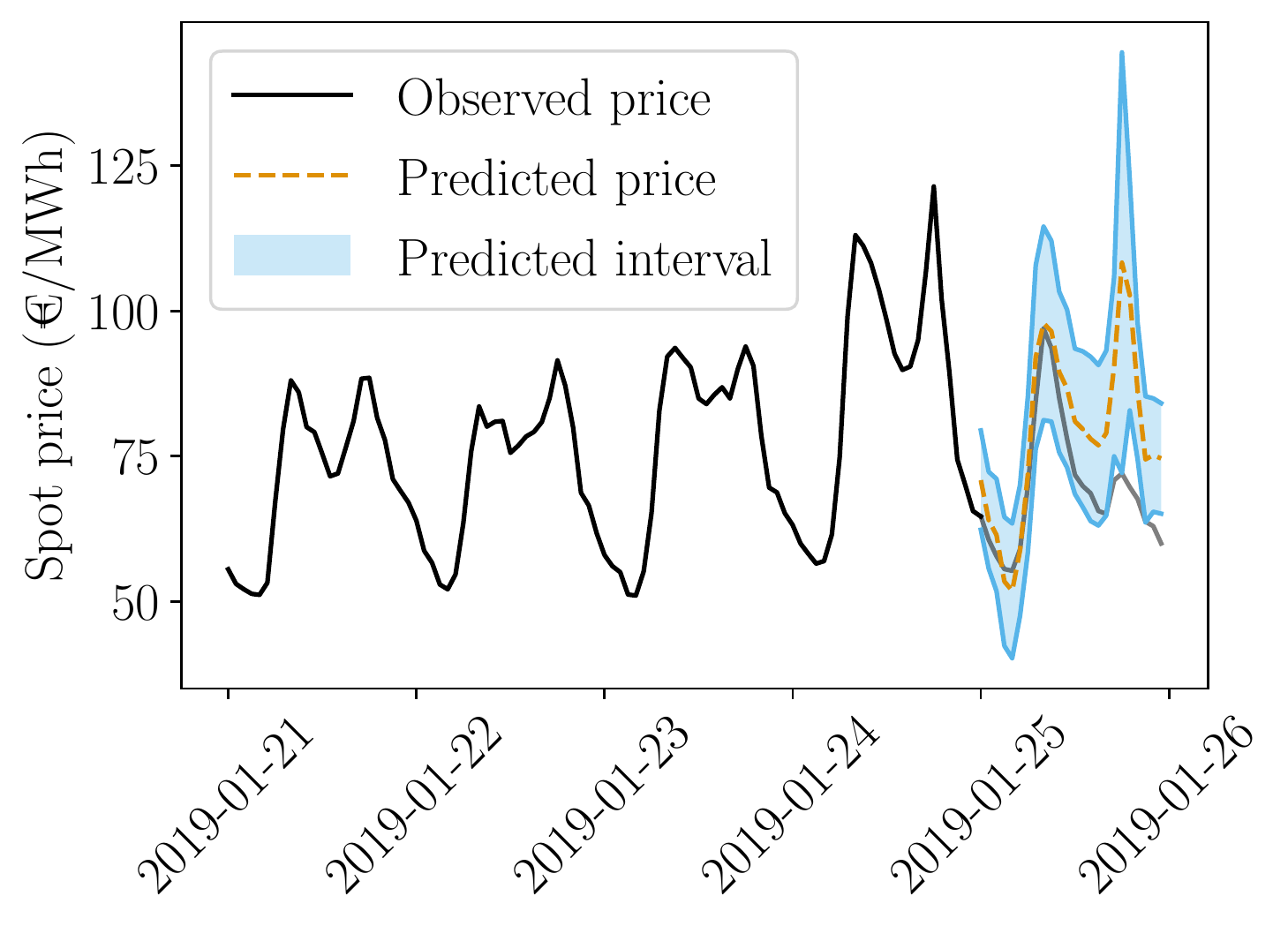}
         \label{fig:ex_int_cp}
         }
     \end{subfigure}
     \begin{subfigure}[ACI with $\gamma = 0.01$]{
         \includegraphics[width=0.36\textwidth]{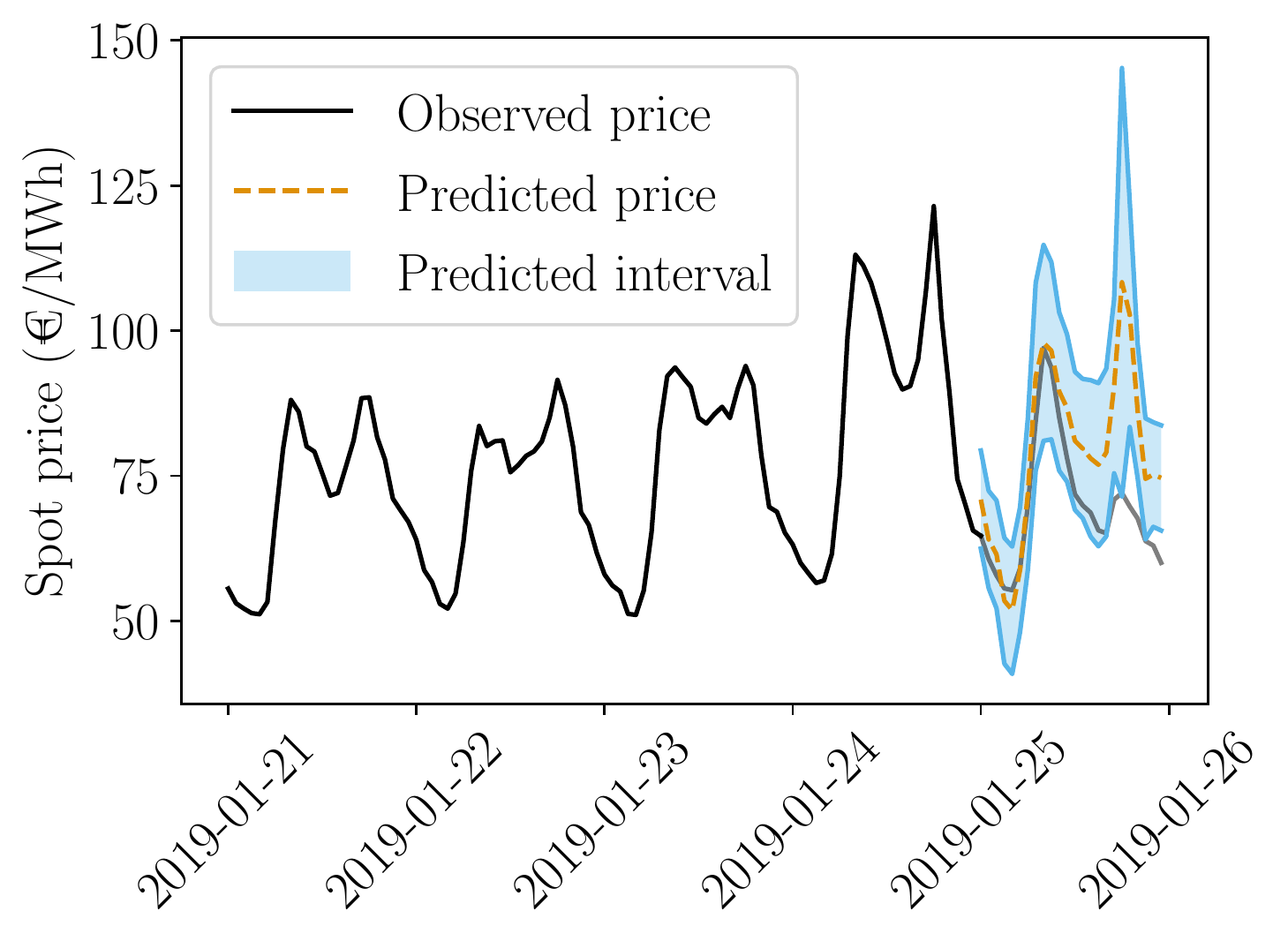}
         \label{fig:ex_int_acp_0.01}
         }
     \end{subfigure}
     \begin{subfigure}[EnbPI V2]{
         \includegraphics[width=0.36\textwidth]{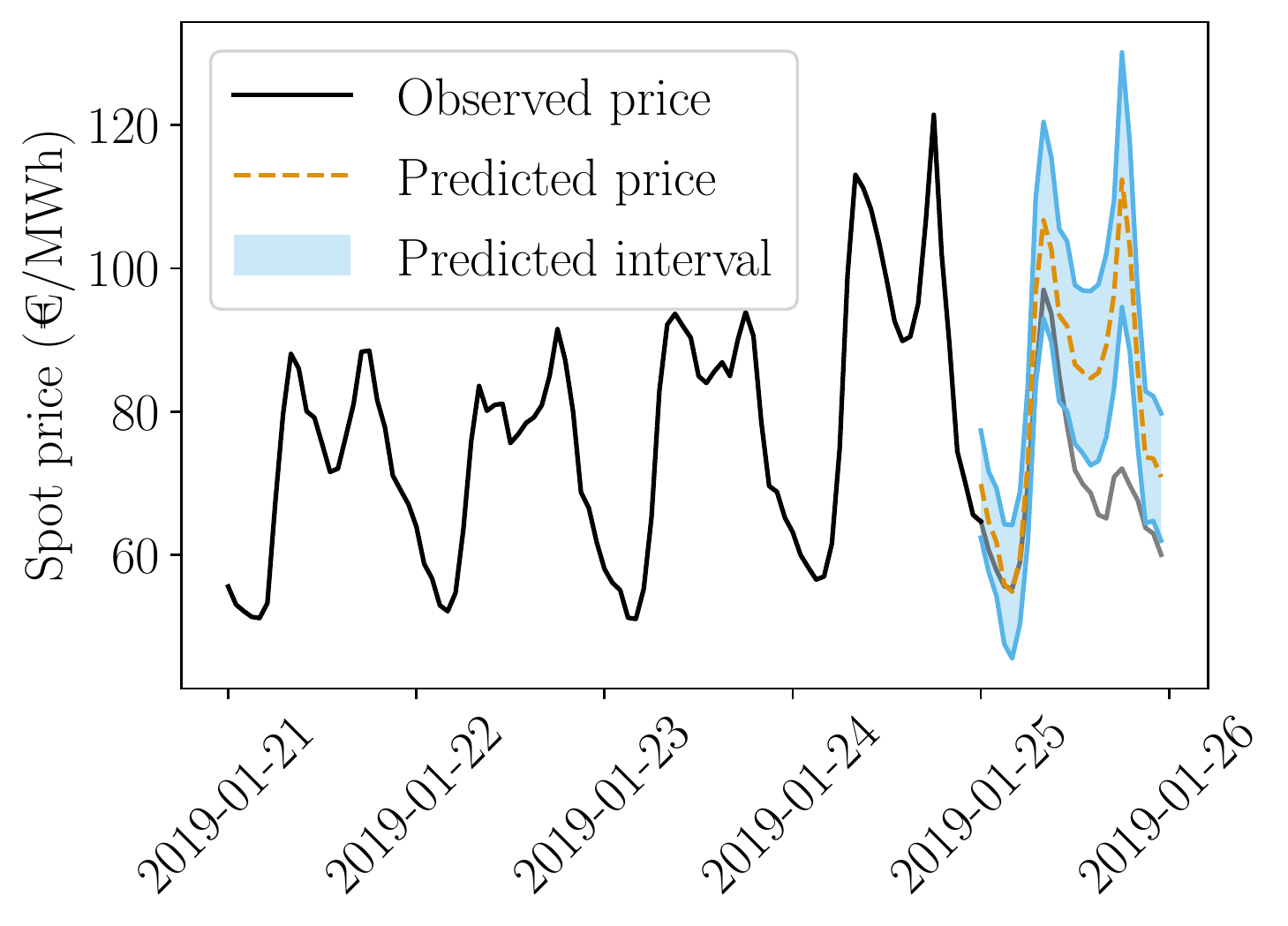}
         \label{fig:ex_int_enbpi}
         }
     \end{subfigure}
     \begin{subfigure}[ACI with $\gamma = 0.05$]{
         \includegraphics[width=0.36\textwidth]{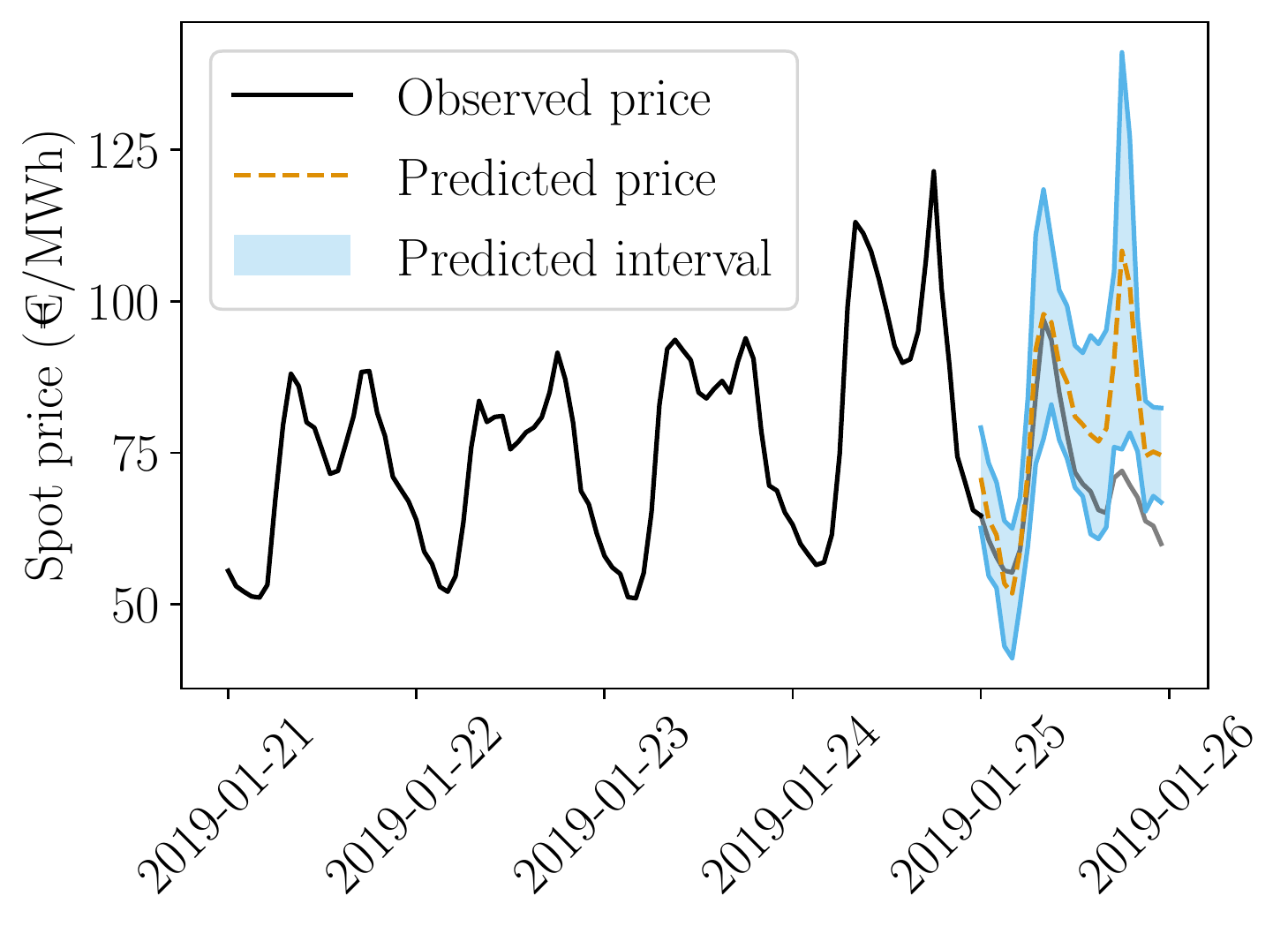}
         \label{fig:ex_int_acp_0.05}
     }
     \end{subfigure}
     \caption{Representation of predicted intervals around point forecasts on the 25th of January of 2019.}
     \label{fig:ex_int_all}
\end{figure}

\end{document}